\documentclass[11pt]{article}
\RequirePackage[OT1]{fontenc}
\RequirePackage{amsmath,amssymb,subcaption,graphicx,epstopdf,enumerate,lmodern}
\RequirePackage[round,colon,authoryear]{natbib}
\RequirePackage{bigints}
\RequirePackage[colorlinks,citecolor=blue,urlcolor=blue]{hyperref}
\usepackage{bm,color,fancyvrb,xcolor,mathtools}
\usepackage{amscd}
\usepackage{mathrsfs}
\usepackage{bbm}

\def\eps{\varepsilon}
\font\tencmmib=cmmib10 \skewchar\tencmmib '60
\newfam\cmmibfam
\textfont\cmmibfam=\tencmmib

\def\lessim{\ \lower4pt\hbox{$
		\buildrel{\displaystyle <}\over\sim$}\ }
\def\gessim{\ \lower4pt\hbox{$\buildrel{\displaystyle >}
		\over\sim$}\ }

\usepackage{fancyhdr}
\usepackage{amsmath}
\usepackage{amsthm}
\usepackage{amssymb,amsfonts}
\usepackage{algorithm,algorithmicx,algpseudocode}
\usepackage{tensor}

\newtheorem{theorem}{Theorem}[section]
\newtheorem{proposition}[theorem]{Proposition}
\newtheorem{lemma}{Lemma}

\DeclarePairedDelimiter{\norm}{\lVert}{\rVert}
\DeclarePairedDelimiter{\abs}{\lvert}{\rvert}
\newcommand{\COMMENT}[2][.5\linewidth]{%
	\leavevmode\hfill\makebox[#1][l]{//~#2}}

\providecommand{\abs}[1]{\left\lvert#1\right\rvert}
\providecommand{\norm}[1]{\left\lVert#1\right\rVert}
\renewcommand{\hat}{\widehat}

\renewcommand{\hat}{\widehat}

\newcommand{\bfm}[1]{\ensuremath{\mathbf{#1}}}
\newcommand\numberthis{\addtocounter{equation}{1}\tag{\theequation}}

\def\ba{\bfm a}   \def\bA{\bfm A}  
   \def\bB{\bfm B}  

   \def\bC{\bfm C}  
   \def\bD{\bfm D}  
\def\be{\bfm e}   \def\bE{\bfm E}  \def\EE{\mathbb{E}}
    
   \def\bG{\bfm G}  
     
   \def\bI{\bfm I}  \def\II{\mathbb{I}}

   \def\bL{\bfm L}  
   \def\bM{\bfm M}

   \def\bP{\bfm P}  \def\PP{\mathbb{P}}
     
     \def\RR{\mathbb{R}}
     \def\SS{\mathbb{S}}
     
\def\bu{\bfm u}   \def\bU{\bfm U}  
\def\bv{\bfm v}   \def\bV{\bfm V}  
\def\bw{\bfm w}   \def\bW{\bfm W}  
\def\bx{\bfm x}   \def\bX{\bfm X}  
\def\by{\bfm y}   \def\bY{\bfm Y}  
\def\bz{\bfm z}   \def\bZ{\bfm Z}

\def\calR{{\cal  R}} 
\def\calS{{\cal  S}}

\def\calW{{\cal  W}}

\newcommand{\bfsym}[1]{\ensuremath{\boldsymbol{#1}}}

           \def\bDelta {\bfsym {\Delta}}

\def\btheta{\bfsym {\theta}}           
\def\beps{\bfsym \varepsilon}          
             
        \def\bLambda {\bfsym {\Lambda}}

\DeclareMathOperator{\argmin}{argmin}

\DeclareMathOperator{\Var}{Var}

\DeclareMathOperator{\tr}{tr}

\def\eps{\varepsilon}
\def\beps{\mbox{\boldmath$\eps$}}

\def\disp{\displaystyle}

\newcommand{\vertiii}[1]{{\left\vert\kern-0.25ex\left\vert\kern-0.25ex\left\vert #1 
		\right\vert\kern-0.25ex\right\vert\kern-0.25ex\right\vert}}

\usepackage{tikz}
\usetikzlibrary{decorations.pathreplacing,calc, patterns}

\def\scrE{\mathscr{E}}
\def\scrX{\mathscr{X}}
\def\scrT{\mathscr{T}}
\def\scrA{\mathscr{A}}
\def\scrG{\mathscr{G}}

\begin{document}

	\title{\Large On Spectral Learning for Odeco Tensors: Perturbation, Initialization, and Algorithms$^\ast$}
	\author{Arnab Auddy\thanks{Department of Statistics, The Ohio State University, Columbus, OH ({auddy.1@osu.edu})}
    \and Ming Yuan\thanks{Department of Statistics, Columbia University, 
			New York, NY
            ({ming.yuan@columbia.edu})}
            }
	\date{(\today)}
	
	\maketitle
	
	\begin{abstract}
    We study spectral learning for orthogonally decomposable (odeco) tensors, emphasizing the interplay between statistical limits, optimization geometry, and initialization. Unlike matrices, recovery for odeco tensors does not hinge on eigengaps, yielding improved robustness under noise. While iterative methods such as tensor power iterations can be statistically efficient, initialization emerges as the main computational bottleneck. We investigate perturbation bounds, non-convex optimization analysis, and initialization strategies, clarifying when efficient algorithms attain statistical limits and when fundamental barriers remain.
	\end{abstract}
	
	\footnotetext[1]{
		This research was supported by NSF Grant DMS-2052955.}
	
	\section{Introduction.}
	Tensors, as higher-order generalizations of matrices, have emerged as powerful tools for representing and analyzing multi-dimensional data. They naturally arise in diverse applications such as multi-relational networks, spatiotemporal measurements, neuroimaging, and latent variable models. Unlike matrices, which capture only pairwise relationships, tensors encode multi-way interactions, offering richer structural insights. Among the various tensor models, \emph{orthogonally decomposable (odeco) tensors} play a special role. Their decomposition structure parallels the eigendecomposition of matrices, but with important advantages in both statistical robustness and computational tractability.
	
	In particular, odeco tensors arise in the \emph{method of moments} for latent variable models. This framework has been especially influential in statistical machine learning, where moment-based tensor decompositions have enabled consistent parameter recovery in models such as topic models, hidden Markov models, and Gaussian mixtures (see, e.g., \cite{janzamin2019spectral}). This approach avoids likelihood maximization, which is often computationally intractable, while still ensuring consistency. Tensor-based methods have since become widely adopted in unsupervised learning, underscoring their importance in bridging statistical modeling and algorithm design. In these problems, the parameter of interest is typically modeled through the spectral decomposition of an odeco tensor. At an abstract level, they can be formulated as recovering the spectral parameters, i.e., singular vectors, of an odeco tensor $\scrT$ from an observed tensor $\scrX=\scrT+\scrE$ where $\scrE$ represents an error or noise tensor. The main goal of this paper is to provide a unified treatment to this type of problems, highlighting the statistical and algorithmic forces that shape its landscape. We argue that three ingredients are central: tensor perturbation bounds, iterative methods for non-convex optimization, and initialization strategies.
	
	\paragraph{Perturbation bounds.} These establish the statistical limits of recovery under noise, revealing a surprising phenomenon: unlike matrices, tensor recovery is not constrained by eigengaps, making tensors statistically more robust in certain regimes. Perturbation analysis has traditionally focused on matrices, with seminal results such as the Davis--Kahan theorem (\cite{davis1970rotation}) and Wedin's theorem (\cite{wedin1972perturbation}) providing eigenvalue and eigenspace stability under noise. See \cite{stewart1990matrix} for a comprehensive review. Extending these results to tensors highlights tensors' inherent statistical advantage over their matrix counterparts.
	
	\paragraph{Non-convex optimization.} Iterative methods such as tensor power iteration and deflation succeed in recovering spectral components of odeco tensors, provided they start from a sufficiently accurate initialization. From an algorithmic perspective, tensor recovery often reduces to solving non-convex optimization problems. Algorithms such as tensor power iteration and deflation provide practical procedures for decomposition. While the non-convex landscape is complex, recent analyses (see, e.g., \cite{ge2015}) show that saddle points dominate but can be escaped efficiently with appropriate initialization. Once a ``warm start'' is available, iterative algorithms converge to the global solution, often achieving the same recovery guarantees as those implied by perturbation bounds. This interplay between optimization geometry and initialization is a defining theme in tensor analysis.
	
	\paragraph{Initialization strategies.} Obtaining a nontrivial initial estimate is the principal computational bottleneck. Randomized approaches such as random slicing and sketching offer partial solutions. But these randomized methods fall short of bridging the statistical--computational gap. The gap resembles hardness results in other high-dimensional estimation problems. These analogies suggest that the limitations are not merely algorithmic deficiencies but reflect fundamental computational barriers.
	
	\medskip
	The emerging consensus is that statistical guarantees alone are insufficient to characterize spectral learning for tensors. While perturbation bounds describe what is possible, actual performance hinges on initialization strategies and the geometry of non-convex objectives. Our contribution is to integrate these strands into a unified framework that explains when efficient algorithms can match statistical limits and when intrinsic barriers remain. By isolating the role of initialization as the critical step, we shed light on both the opportunities and the obstacles in high-dimensional tensor analysis.
	
	The rest of the paper is organized as follows. In the next section, we first review some basic properties of odeco tensors and show how usual iterative methods such as fixed point iteration can be applied for odeco tensor decomposition. Section 3 considers the special case when the ``noisy observation'' $\scrX$ is also odeco. In this case, it is natural to estimate the spectral parameters of $\scrT$ by those of $\scrX$. We provide refined perturbation bounds for these estimates. Section 4 considers more general cases where a good estimating procedure for the spectral parameters of $\scrT$ requires consideration from both statistical and computational perspective. We show that the usual iterative methods such as power iteration combined with deflation can yield estimates that are efficient both statistically and computationally, provided that we can find a nontrivial initialization. Section 5 is devoted to the study of initialization. We argue that a higher signal-to-noise ratio is necessary to ensure a good initialization, and this is even more so if we aim to do away with the eigengap conditions. We conclude with a few remarks in Section 6.
	
	\section{Odeco Decomposition.}\label{sec:odeco-dec}
	
	We denote vectors and matrices by boldfaced lower and upper case letters respectively: e.g., $\bv\in\RR^{d}$ and $\bM\in\RR^{d_1\times d_2}$; and tensors by script-style letters, e.g., $\scrT$. Entries of a matrix or tensor is represented by upper-case letters with a set of indices, e.g., $T_{i_1i_2\cdots i_p}$ is the $(i_1,\ldots, i_p)$ entry of an $p$th order tensor $\scrT$. We first list a set of basic notations and tensor operations that we shall use throughout the paper.
	
	\subsection{Preliminaries.}
	A $p$-th order tensor $\scrT\in \RR^{d_1\times\cdots\times d_p}$ has the mode-$k$ matricization 
	$
	{\sf Mat}_{k}(\scrT)\in \RR^{d_k\times \prod_{j\neq k, j\in [p]}d_j}
	$
	with elements 
	$
	\left({\sf Mat}_k(\scrT)\right)_{i_kj}
	=
	T_{i_1\dots i_{k-1}i_ki_{k+1}\dots i_p}
	$
	where 
	$
	j=1+\sum_{l\in [p],l\neq k}(i_l-1)J_l
	$
	with
	$
	J_l=\prod_{m\in [l-1],m\neq k}d_m$.
	A tensor $\scrT\in \RR^{d_1\times\dots\times d_p}$ can be multiplied along its $k$-th mode by a matrix $\bM\in \RR^{m_k\times d_k}$ to get a new tensor 
	$
	\scrT'
	:=\scrT\times_k\bM\in \RR^{d_1\times d_2\times \dots d_{k-1}\times m_k\times d_{k+1}\dots \times d_p}
	$
	with elements 
	$
	T'_{i_1\dots i_{k-1}j_ki_{k+1}\dots i_p}
	=\sum_{i_k\in [d_k]}T_{i_1\dots i_p}M_{j_ki_k}
	$
	for $i_l\in [d_l]$ for $l=1,\dots,p$, $l\neq k$, and $j_k\in [m_k]$. 
    The spectral norm or operator norm of $\scrT$ is defined as
	$
	\|\scrT\|:=\sup_{\bu_k\in \SS^{d_k-1}}\langle \scrT, \bu_1\circ\dots\circ \bu_p\rangle,
	$
	where $\SS^{d-1}$ denotes the unit sphere in $\RR^d$.

	Singular value decomposition (SVD) plays a central role in analyzing data formatted in matrices. It turns out that there are multiple ways to generalize SVD to higher order tensors. One of most popular approaches to tensor decomposition is the so-called Canonical Polyadic (CP) decomposition \cite{hitchcock1927expression, carroll1970analysis, harshman1970foundations} that expresses a tensor as the sum of rank-one tensors, e.g., tensors that can be expressed as outer product of vectors:
	\begin{equation}\label{eq:cp-decomp}
		\scrT=\sum_{i=1}^r\lambda_i\bu_i^{(1)}\circ\bu_i^{(2)}\circ\cdots\circ\bu_i^{(p)}.
	\end{equation}
	where $\lambda_i$s are scalars, and $\bu_i^{(q)}\in\SS^{d_q-1}$, $q=1,\ldots,p$ and $i=1,\ldots, r$ are unit vectors. The smallest integer $r$ for which such a decomposition holds is called the CP-rank of $\scrT$. It is convenient to use the notation $\scrT=[{\boldsymbol{\lambda}};\bU^{(1)},\bU^{(2)},\ldots,\bU^{(p)}]$ for CP decomposition as described by Equation~\ref{eq:cp-decomp} where ${\boldsymbol{\lambda}}=(\lambda_1,\ldots,\lambda_r)^\top$, and $\bU^{(q)}=[\bu_1^{(q)},\bu_2^{(q)},\ldots, \bu_r^{(q)}]$. In particular if $\bU^{(q)}$ are orthonormal, we say $\scrT$ is orthogonally decomposable (odeco).
	
	\subsection{Odeco decomposition.}	
	While computing the minimum rank CP decomposition of a tensor is computationally hard in general, it can be done very efficiently for odeco tensors. In particular, we may do so by iteratively solving for zeros of the Riemannian gradient (\cite{edelman1998geometry,  kressner2014low}). Note that an odeco tensor $\scrT=[{\boldsymbol{\lambda}};\bU^{(1)},\bU^{(2)},\ldots,\bU^{(p)}]$ inherits a manifold structure from its component matrices:
	\[
	\bU^{(q)}\in St(d,r)=\{\bL\in \RR^{d\times r}:\bL^{\top}\bL=\II_{r}\},
	\]
	i.e., the Stiefel manifold of all $d\times r$ matrices with orthonormal columns.
	
	It is not hard to see that the SVDs of $\scrT$ can be identified with the solution to
	$$
	\text{maximize }_{\bA^{(q)}\in St(d_q,r)}
	\disp\sum_{i=1}^r\mathscr{T}(\ba^{(1)}_i,\dots,\ba^{(p)}_i)^2.
	$$
	See, e.g., \cite{chen2009tensor}. We will consider an alternating update scheme, where each of the mode matrices $\bA^{(q)}$ is updated one by one. Recall that the tangent space of $St(d,r)$ at $\bU$ is given by 
	$$
	T_{\bU}(St(d,r))=
	\{\bU\bB+(\bI_d-\bU\bU^{\top})\bC: \quad \bB\in \RR^{r\times r},\,\bB=-\bB^{\top},\,\bC\in\RR^{d\times r}\}.
	$$ 
	The Riemannian gradient of a function $F(\bA^{(1)})$ is the projection of its Euclidean gradient along the tangent space. In our case, 
	\[
	F_1(\bA^{(1)})
	=
	\disp\sum_{i=1}^r\mathscr{T}(\ba^{(1)}_i,\dots,\ba^{(p)}_i)^2
	=
	\sum_{k=1}^r
	(\langle \ba_k^{(1)},\bv_k^{(1)}\rangle)^2
	\]
	where $\bv_k^{(1)}=\scrT\times_{q\neq 1}\ba^{(q)}_k\in \RR^d$ are held fixed, due to the alternating update scheme. Thus
	\[
	\bG  
	=
	\frac{\partial
		F_1(\bA^{(1)})}{\partial\bA^{(1)}_{i,j}}
	=\bV^{(1)}\bD^{(1)}
	\]
	where $\bD^{(1)}\in \RR^{r\times r}$ is a diagonal matrix with entries 
	\[
	d^{(1)}_{kk}
	=\langle \ba_k^{(1)},\bv_k\rangle
	\quad
	\text{for}
	\quad
	k=1,\dots,r,
	\]
	and $\bV^{(1)}\in \RR^{d\times r}$ is the matrix with $\bv^{(1)}_k$ as its $k$th column. Then the Riemannian gradient at $\bA^{(1)}$ becomes:
	\begin{equation}\label{eq:rie-grad}
		\nabla F_1(\bA^{(1)})
		=\bG-\bA^{(1)}\bG^{\top}\bA^{(1)}.
	\end{equation}
	See, e.g., Lemma 2.1 of \cite{wen2013feasible}. A critical point of the above optimization scheme is arrived at by setting
	\[
	\nabla F_1(\bA^{(1)})=0
	\iff 
	(\bA^{(1)})^{\top}\bG
	=\bG^{\top}\bA^{(1)}
	\]
	or equivalently:
	\begin{align*}
		\langle \ba_j^{(1)},\bv_j\rangle
		\left(\ba^{(1)}_i\right)^{\top}{\bv}^{(1)}_j&=
		\langle \ba_i^{(1)},\bv_i\rangle
		\left(\ba^{(1)}_j\right)^{\top}{\bv}^{(1)}_i
		\text{ for }1\le i,j\le r.
	\end{align*}
    Holding $\bV^{(1)}$ fixed, the above equation is satisfied by setting
    \[
    \ba_k^{(1)}
    =\frac{1}{\|\bv_k^{(1)}\|}\bv_k^{(1)}
    \quad
	\text{for}
	\quad
	k=1,\dots,r,
    \]
    suggesting a  fixed point iteration scheme.
    
    Note however that the output after $T$ fixed point iterations is not necessarily orthogonal. To take care of this, we use a Gram-Schmidt orthogonalization at the last step. This procedure is repeated for each mode. The detailed optimization scheme is in Algorithm~\ref{alg:tpi}.
	\begin{algorithm}[htb]
		\caption{Odeco Decomposition via Alternating Fixed Point Iteration}\label{alg:tpi}
		\begin{algorithmic}
			\State{\textbf{Input}:
				odeco tensor $\scrT$ whose odeco SVD is to be computed;
				initializations $\{ {\bU}^{(q)}_{[0]}\in {\rm St}(d_q,r):1\le q\le p\}$ }
            \State Set $\calS=\{1,\dots,r\}$.
            \COMMENT{Set of column indices to be recovered}
            \While{ $|\calS|>0$}
			\For{{$k\in \calS$}}
        	\For{$t=1,\dots, T$}
            \For{{$q=1,\dots,p$}}
            \State Compute $\bv=\scrT\times_{q_1< q} {\bu}_{k,[t]}^{(q_1)}
            \times_{q_2>q}{\bu}_{k,[t-1]}^{(q_2)}
            $.
            \State Compute $ {\bu}_{k,[t]}^{(q)}=\frac{1}{\|\bv\|}\bv$.
   			\EndFor
            \EndFor
			\EndFor 
            \For{{$q=1,\dots,p$}}
            \State Run Gram-Schmidt orthogonalization on the columns of $ {\bU}^{(q)}_{[T]}$.
            \State Update $ {\bU}^{(q)}_{[0]}=(\II_{d_q}- {\bU}^{(q)}_{[T]}( {\bU}^{(q)}_{[T]})^{\top}) {\bU}^{(q)}_{[0]}$.
            \EndFor
            \State Update $\calS=\{k\in \calS: \|{\bu}_{k,[T]}^{(1)}\|<0.1\}$.
		\EndWhile
        \State{\textbf{Return}: $\{ {\bU}^{(q)}_{[t]}:1\le q\le p\}$.}
		\end{algorithmic}
	\end{algorithm}

	The following result shows that Algorithm~\ref{alg:tpi} is guaranteed to recover the components of the odeco tensor $\scrT$ with a random starting point.
	\begin{proposition}\label{pr:noiseless-tpi}
		Let $\{{\bU}^{(q)}_{[t]}:1\le q\le p\}$ be the iterates at the $t$-th iteration of Algorithm~\ref{alg:tpi} with an odeco tensor $[{\boldsymbol{\lambda}};\bU^{(1)},\bU^{(2)},\ldots,\bU^{(p)}]$ as input and initializations $\{ {\bU}^{(q)}_{[0]}:1\le q\le p\}$ sampled uniformly on ${\rm St}(d,r)$. Then with probability 1, there exists a permutation $\pi:[r]\to[r]$ such that 
        \[
        \max_{1\le q\le p}
		\sin\angle\left(
		  {\bu}^{(q)}_{j,[t]},
		\bu^{(q)}_{\pi(j)}
		\right)
         \le \lambda_{\pi(j)}^{(p-2)}
         \left(
            \sum_{k\neq \pi(j)}
            \frac{1}{\lambda_k^{2(p-2)}}
            \right)^{1/2}
         \left\{\frac{v_{2j}}{v_{1j}}
        \right\}^{(p-1)^{t}}, 
        \]
		where $v_{1j},v_{2j}$ are respectively the largest and second largest elements of $\{|\scrT\times_1\bu_k^{(1)}\times_{2,\ldots,p}\bu_{j,[0]}^{(q)}|:1\le k\le r\}$.
	\end{proposition}

	\section{Perturbation Bounds.}\label{sec:perturb}
	We first consider the case when an observed tensor $\scrX$ is also odeco. A natural approach is to estimate the singular vectors of $\scrT$ by those of $\scrX$, which can be readily computed using Algorithm \ref{alg:tpi} from the previous section.
	Denote by
	$$
	\scrT=[\{\lambda_k: 1\le k\le d_{\min}\}; \bU^{(1)},\ldots,\bU^{(p)}],
	$$
	and
	$$	
	{\scrX}=[\{\hat{\lambda}_k: 1\le k\le d_{\min}\}; \hat{\bU}^{(1)},\ldots,\hat{\bU}^{(p)}]
	$$
	their respective SVDs. We are interested in understanding the differences between the two sets of singular values and vectors in terms of the ``perturbation'' $\scrX-\scrT$. For an odeco tensor $\scrT$, its SVD determines that of ${\sf Mat}_q(\scrT)$. More specifically,
	\begin{equation}\label{eq:odeco-mat}
		{\sf Mat}_q(\scrT)=\bU^{(q)}({\rm diag}(\lambda_1,\ldots,\lambda_{d_{\min}}))(\bV^{(q)})^\top,
	\end{equation}
	where
	$$
	\bV^{(q)}=\bU^{(1)}\odot\cdots\odot\bU^{(q-1)}\odot\bU^{(q+1)}\odot\cdots\odot\bU^{(p)}.
	$$
	Here $\odot$ stands for the Khatri-Rao product. Thanks to Weyl's and Davis-Kahan's perturbation theorem the matricization of odeco tensors in \eqref{eq:odeco-mat} immediately implies that
	$$
	|\lambda_k-\hat{\lambda}_k|\le \min_{1\le q\le p}\|{\sf Mat}_q(\scrE)\|,
	$$	
	and
	$$
	\sin\angle(\bu_k^{(q)},\hat{\bu}_k^{(q)})\le {2\|{\sf Mat}_q(\scrE)\|\over \min\{\lambda_{k-1}-\lambda_k,\lambda_k-\lambda_{k+1}\}}.
	$$
	We note that the dependence of the above perturbation bounds for matricized singular vectors on the singular value gap is intrinsic and cannot be removed. This fundamental feature of matrix based perturbation bounds can be rigorously demonstrated through simple examples. See, e.g., \cite{bhatia2013matrix}.
	
	The above bounds are not optimal and can be substantially tightened in several ways that reveal key distinctions between matrices and higher-order tensors, for a number of reasons. Firstly, the presence of eigengap is necessary for identifying singular vectors for matrices, while odeco tensors are always identifiable up to a sign and permutation of the indices, thanks to the Kruskal's theorem. This suggests tensor perturbation bounds that should similarly be free from any dependence on eigengaps. Secondly, flattening a tensor does not preserve the manifold structure present among odeco tensors. Especially in high dimensional applications where the dimensions along the tensor modes is large this additional structure makes tensors more resilient to noise than matrices. 
	
	More specifically, as shown by \cite{auddy2020perturbation}, there exist a numerical constant $C\ge 1$ and a permutation $\pi: [d_{\min}]\to [d_{\min}]$ such that for all $k=1,\ldots, d_{\min}$,
	\begin{equation}
		\label{eq:odecoweyl0}
		|\lambda_k-\hat{\lambda}_{\pi(k)}|\le C\|\scrE\|,
	\end{equation}
	and
	\begin{equation}
		\label{eq:odecodavis0}
		\max_{1\le q\le p}\sin\angle(\bu_k^{(q)}, \hat{\bu}_{\pi(k)}^{(q)})\le {C\|\scrE\|\over \lambda_k},
	\end{equation}
	with the convention that $1/0=+\infty$.
	
	Note that the above result gives perturbation bounds between $\bu_k$ and $\hat{\bu}_{\pi(k)}$ for some unknown permutation $\pi:[d]\to [d]$. This is because, unlike matrices, where eigengap is necessary for singular vector perturbation analysis, we do not restrict that the singular values $\lambda_k$s are distinct and sufficiently apart from each other. As a result, we may not necessarily match the $k$th singular value/vector  $(\lambda_k, \bu_k^{(1)},\ldots, \bu_k^{(p)})$ of $\scrT$ with that of $\scrX$. 
	
	As noted earlier, the tensor perturbation bounds \eqref{eq:odecodavis0} for singular vectors are fundamentally different from Wedin-Davis-Kahan $\sin\Theta$ theorems due to the absence of the gap between singular values. This means that for higher-order odeco tensors, the perturbation affects the singular vectors separately. The perturbation bound \eqref{eq:odecodavis0} depends only on the amount of perturbation relative to their corresponding singular value.
	
	A careful inspection of the proofs of \cite{auddy2020perturbation} shows that when the perturbation is sufficiently small, or for large enough singular values, we can derive the precise first order remainder term in the asymptotic expansion of $\scrX$ around $\scrT$.
	\begin{theorem}\label{th:ortho-perturb}
		There exists a permutation $\pi: [d_{\min}]\to [d_{\min}]$ such that for any $\eps>0$, 
		\begin{equation}\label{eq:odeco-first-order}
			\max_{1\le q\le p}\sin\angle\left(\hat{\bu}_{\pi(k)}^{(q)},\,\bu_{k}^{(q)}
			+\dfrac{1}{\lambda_{k}}(\scrE)\times_{s\neq q}
			\bu_{k}^{(s)}\right)
			\le 
			\left(2+{\|\scrE\|\over\lambda_k}\right)
			\left(\dfrac{(1+\eps)\|\scrE\|}{\lambda_{k}}\right)^{p-1}
		\end{equation}    
		provided that $\|\scrE\|\le c_{\eps}\lambda_k$ for some constant $c_{\eps}>0$ depending on $\eps$ only.
	\end{theorem}
	
	When considering infinitesimal perturbation in that $\|\scrE\|=o(\lambda_k)$, we can express the bound \eqref{eq:odeco-first-order} for singular vectors as
	\begin{equation}\label{eq:odeco-first-order-approx}
		{\rm sign}(\langle {\bu}_k^{(q)},\hat{\bu}_{\pi(k)}^{(q)}\rangle)
		\hat{\bu}_{\pi(k)}^{(q)}
		=
		\frac{1}{\|\bu_{k,*}^{(q)}\|}\bu_{k,*}^{(q)}
		+
		O\left({\|\scrE\|^{p-1}\over \lambda_k^{p-1}}\right),
	\end{equation}
	where
	\[
	\bu_{k,*}^{(q)}
	=\bu_{k}^{(q)}
	+\dfrac{1}{\lambda_{k}}\cdot \scrE\times_{s\neq q}
	\bu_{k}^{(s)}.
	\]
	The approximation $\bu_{k,*}^{(q)}$ in \eqref{eq:odeco-first-order-approx} for $\hat{\bu}_{\pi(k)}^{(q)}$, the singular vectors of $\scrX$, is accurate to the order of $O\left((\|\scrX-\scrT\|/\lambda_k)^{p-1}\right)$ in the direction of any vector perpendicular to $\hat{\bu}_{\pi(k)}^{(q)}$. 
	

	As an example demonstrating the above bound, consider the two $d\times d\times d$ odeco tensors
	\begin{equation}\label{eq:ortho_ex}
		\scrT=\lambda \displaystyle\sum_{k=1}^{d-1}\be_k\circ \be_k\circ \be_k
		\quad
		{\rm and }
		\quad
		\scrX=\lambda \displaystyle\sum_{k=1}^{d-1}({\be}_i+\bv)\circ \be_k\circ \be_k
	\end{equation}
	where $\{\be_k:1\le i\le d\}$ are the canonical basis vectors in $\RR^d$ and
	$\bv=\dfrac{1}{\sqrt{d-1}}\be_d-\dfrac{1}{d-1}\left(\be_1+\dots+\be_{d-1}\right)$. By the Kruskal theorem, the SVDs for both $\scrT$ and $\scrX$ are uniquely identified upto a sign and permutation. In particular, there exists a permutation $\pi:[d]\to [d]$ such that
	\(
	\hat{\bu}_{\pi(k)}^{(1)}
	=\be_k+\bv 
	\) and 
	\(\bu_k^{(1)}=\be_k\).
	Meanwhile, note that $\scrE=\lambda\sum_{k=1}^d \bv\circ\be_k\circ\be_k$, and thus
	\[
	\|\scrE\|=\lambda\|\bv\|=\lambda\sqrt{\frac{2}{d-1}}.
	\]
	Now, since
	\[
	\scrE\times_{2}\be_k\times_3\be_k=\lambda \bv 
	\,\,
	\text{for }
	k=1,\dots,d-1
	\]
	we can write:
	\begin{align*}
		\hat{\bu}_{\pi(k)}^{(1)}
		=\be_k+\bv 
		=\bu_k^{(1)}+\frac{1}{\lambda}\cdot\scrE\times_{2}\be_k\times_3\be_k
		\,\,
		\text{for }
		k=1,\dots,d-1
		,
	\end{align*}
	meaning that the approximation in \eqref{eq:odeco-first-order-approx} is exact for the first mode. Similarly, with
	$\hat{\bu}_{\pi(k)}^{(2)}
	=\bu_k
	=\be_k$
	for $k=1,\dots,d-1$ it can be checked that
	\[
	\bu_{k*}^{(2)}
	=\bu_k^{(2)}+
	(\scrX-\scrT)\times_{1}\bu_k^{(1)}\times_3\bu_k^{(3)}
	=
	\left(1
	-\frac{\lambda}{d-1} 
	\right)\be_k
	\]
	meaning
	\begin{align*}
		\hat{\bu}_{\pi(k)}^{(2)}
		=
		\be_k
		=\frac{1}{\|\bu_{k,*}^{(2)}\|}\bu_{k,*}^{(2)}
		\,\,
		\text{for }
		k=1,\dots,d-1.
	\end{align*}
	For this particular example, all higher order terms in \eqref{eq:odeco-first-order-approx}
	vanish.
	
	\section{Odeco Approximation.}\label{sec:compute}
	We now consider the more general situations where the observed tensor $\scrX$ is not necessarily odeco. For simplicity, we shall assume that $d_1=d_2=\ldots=d_p=d$ throughout the section. In this case, we need to first define our estimate for $\bu^{(q)}_k$s. One immediate approach is to first seek the best odeco approximation to $\scrX$ as an estimate of $\scrT$:
    $$
    \hat{\scrT}=\argmin_{\scrA {\rm \ is\ odeco}}\|\scrX-\scrA\|,
    $$
    and then estimate the singular vectors of $\scrT$ by those of $\hat{\scrT}$. By triangular inequality,
    $$
   \|\scrT-\hat{\scrT}\|\le 2\|\scrE\|,
    $$
    and thus we can apply the perturbation bounds from the previous section to derive error bounds for the estimated singular vectors. However, computing the best odeco approximation $\hat{\scrT}$ is difficult and this estimating procedure may not be practically feasible. Instead, we shall show here that a simply procedure based on power iteration and deflation can yield estimates with similar performance guarantees and perhaps even better statistical properties in some cases.
	
	\subsection{Power Iteration.}
	We first show that given initial estimators $\hat{\bu}_{[0]}^{(q)}$ which are suitably close to the leading singular vector $\bu_1^{(q)}$ for $1\le q\le p$, we can improve the estimation accuracy of $\bu_1^{(q)}$ by an iterative scheme. More precisely, suppose we have initial estimates $\hat{\bu}_{[0]}^{(q)}$ such that
	\begin{equation}\label{eq:nontriv-init}
		\max_{1\le q\le p}
		\sin\angle\left(
		\hat{\bu}_{[0]}^{(q)},\bu_1^{(q)}
		\right)\le \frac{1}{4}.
	\end{equation}
	The power iteration updates $\hat{\bu}_{[t]}^{(q)}$ by
	\begin{equation}\label{eq:piter}
		\hat{\bu}_{[t+1]}^{(q)}=
		\frac{1}{\|\bv\|}\bv 
		\quad
		\text{where}
		\quad 
		\bv=\scrX\times_{p\neq q}\bu_{[t]}^{(q)}
	\end{equation}
	The following quantities based on the error tensor $\scrE$ will determine the estimation accuracy of our iterative estimator:
	\begin{equation}\label{eq:def-eps1-eps2}
		\eps_1:=\max\limits_{j,k}\|\scrE\times_{q\neq k}\bu^{(q)}_j\|
		\quad
		\text{and}
		\quad
		\eps_2:=\max\limits_{j,k_1,k_2}\|\scrE\times_{q\neq k_1,k_2}\bu^{(q)}_j\|.
	\end{equation}
	In contrast to $\|\scrE\|$ which measures the largest linear combination of $\scrE$ with all possible unit vectors along the $p$ modes, the definitions of $\eps_1$ and $\eps_2$ respectively restrict all but 1 and 2 modes' linear combinations to be among the signal directions $\bu_j^{(q)}$. 
	
	\begin{proposition}
		\label{pr:piter} There exists a numerical constant $C>0$ such that if $C\|\scrE\|\le \lambda_1$, after $T\ge C\log \eps_1$ steps of the power iteration \eqref{eq:piter} initialized at estimators satisfying \eqref{eq:nontriv-init}, the estimators $\hat{\bu}_1^{(q)}:=\bu^{(q)}_{1,[T]}$ satisfy
		$$
		\max_{1\le q\le p}
		\sin\angle(\hat{\bu}_1^{(q)},\bu_1^{(q)})
		\le \dfrac{C\eps_1}{\lambda_1}
		$$
		and
		\begin{equation}\label{eq:fo-err-bd}
			\max_{1\le q\le p}
			\sin\angle\left(
			{\rm sign}(\langle
			\hat{\bu}_1^{(q)},\bu_1^{(q)}
			\rangle)
			\bu^{(q)}_1,
			\hat{\bu}_{1}^{(q)}
			-\dfrac{1}{\lambda_1}\scrE\times_{k\neq q}\bu_1^{(k)}
			\right)
			\le
			\dfrac{C(\eps_1^2+\eps_1\eps_2)}{\lambda_1^2}.
		\end{equation}	    
	\end{proposition}
	Similar to the perturbation bounds in Section~\ref{sec:perturb}, the error in estimating $\bu_1^{(q)}$ with $\hat{\bu}_1^{(q)}$, depends only on $\lambda_1$ and is free of any dependence on the gap between singular values. Moreover, the estimation rates depend on $\scrE$ only through $\eps_1$ and $\eps_2$, which by definition satisfy the inequalities:
	\[
	\eps_1\le \eps_2\le \|\scrE\|\le \lambda_{\min}/C
	\]
	for some constant $C>0$. Thus, while Proposition~\ref{pr:piter} requires $\lambda_1\gtrsim \|\scrE\|$, the rate of estimation is affected by the smaller error norms $\eps_1$ and $\eps_2$. While for iid Gaussian errors, $\eps_1\asymp \eps_2\asymp \|\scrE\|\asymp \sqrt{d}$, these three norms can differ in magnitude for other types of error tensors arising in practice. Thus the rates given in Proposition \ref{pr:piter} is an improvement over existing results in the literature in important ways. An example is for random errors with heavy tails, where say for an error tensor with only finite $4^{\text th}$ moments, $\eps_1\asymp \eps_2\asymp\sqrt{d}$, while $\|\scrE\|\asymp d^{p/4}$. See \cite{auddy2021estimating} for details. This difference can also arise when estimating higher order population moments through their sample counterparts, as is the case for multiview mixture models and ICA (see, e.g., \cite{auddy2023large}).
	
	Finally, equation~\eqref{eq:fo-err-bd} is analogous to the perturbation bound in Theorem~\ref{th:ortho-perturb}, and shows the precise leading order term in the approximation of $\bu_1^{(q)}$ by $\hat{\bu}_1^{(q)}$. This can be readily used for establishing distributional properties of $\hat{\bu}_1^{(q)}$ under suitable conditions.
	
	\subsection{Deflation.}
	In the previous subsection, we showed that the power iteration successively improves estimation of the top singular vectors $\{\bu_1^{(q)}:1\le q\le p\}$. We now move on to the remaining singular vectors $\bu_j^{(q)}$ for $j\ge 2$. To ensure that the previously estimated leading singular vector $\bu^{(q)}_1$ does not obfuscate the signal in the direction of, say, $\bu^{(q)}_2$, we will remove from $\scrX$ the estimated rank one tensor component $\hat{\lambda}_1\hat{\bu}_1^{(1)}\circ\hat{\bu}_1^{(2)}\circ\ldots\circ \hat{\bu}_1^{(p)}$. In general, if we have estimated $(j-1)$ directions, at the $j$th step we define the deflated tensor:
	\begin{align*}\label{eq:deflate}
		\scrX_j
		=:&\sum_{i\ge j}\lambda_i\bu^{(1)}_i\circ\dots\circ \bu^{(p)}_i+\scrE+
		\sum_{i<j}\left(\lambda_i\bu^{(1)}_i\circ\dots\circ \bu^{(p)}_i
		-
        \langle\scrX, \hat{\bu}^{(1)}_i\circ\dots\circ \hat{\bu}^{(p)}_i\rangle
        \,\hat{\bu}^{(1)}_i\circ\dots\circ \hat{\bu}^{(p)}_i
		\right)\\
		=:&\scrT_{j}+\scrE+\hat{\scrT}_{-j}.\numberthis
	\end{align*}
	
	\begin{algorithm}[htbp]
		\caption{Odeco Approximation via Power Iteration and Deflation}\label{alg:alg-odeco-svd}
		\hspace*{\algorithmicindent} \textbf{Input:} tensor $\scrX$ whose rank-$r$ odeco approximation is to be computed
		\begin{algorithmic}[1]
			\For {$j=1$ to $r$}
			\State $\hat{\bu}_{j,[0]}^{(1)}, \ldots, \hat{\bu}_{j,[0]}^{(p)}\leftarrow$ 
			initialize 
			\COMMENT{Initialization}
			\For{$t=1$ to $T$} \COMMENT{Power iteration}
			\State $\hat{\bu}_j^{(q)}\leftarrow \mathscr{X}\times_{q'\neq q}\bu_{j,[t-1]}^{(q')}$
			\State $\hat{\bu}_{j,[t]}^{(q)}\leftarrow {\hat{\bu}_j^{(q)}\over \|\hat{\bu}_j^{(q)}\|}$
			\EndFor
			\State $\scrX\leftarrow \scrX-\langle\scrX, \circ_{q=1}^p \hat{\bu}_{j,[T]}^{(q)}\rangle\cdot \circ_{q=1}^p \hat{\bu}_{j,[T]}^{(q)}$ \COMMENT{Deflation}
			\EndFor\\
			\Return $\{\hat{\bu}_{j,[T]}^{(q)}:1\le j\le r, 1\le q\le p\}$.
		\end{algorithmic}
	\end{algorithm}
	
	To ensure that the errors in the previously estimated directions do not accumulate, we need to make some additional assumptions on $\scrE$.
	
	\paragraph{Assumptions.} There exists a numerical constant $C>0$ such that the following hold.
	\begin{enumerate}
		\item[A1.] $\lambda_{\min}\ge C\max\{\eps_0\eps_1,\eps_1r^{1/4},\|\scrE\|\}$ 
		where $\eps_0=\max_{j}|\scrE\times_q\bu_j^{(q)}|$ and $\eps_1,\eps_2$ are as defined in \eqref{eq:def-eps1-eps2}
		.
		\item[A2.] Let us define the matrices $\bE_j^{(q_1,q_2)}=\scrE\times_{k\neq q_1,q_2}\bu_j^{(k)}\in \RR^{d\times d}$ for $1\le q_1,q_2\le p$. Then for every $q$ there exists $q'\neq q$ such that:
		\[
		\norm*{
			\left[
			\left(\bE_1^{(q,q')}\right)\bu^{(q')}_1\,
			\left(\bE_2^{(q,q')}\right)\bu^{(q')}_2\,
			\ldots\,
			\left(\bE_r^{(q,q')}\right)\bu^{(q')}_r
			\right]
		}\le C\eps_1
		\]
		and    
		\[
		\norm*{
			\left[
			\left(\bE_1^{(q,q')}\right)\left(\bE_1^{(q,q')}\right)^{\top}\bu^{(q)}_1,\,
			\left(\bE_2^{(q,q')}\right)\left(\bE_2^{(q,q')}\right)^{\top}\bu^{(q)}_2,\,
			\ldots\,,
			\left(\bE_r^{(q,q')}\right)\left(\bE_r^{(q,q')}\right)^{\top}\bu^{(q)}_r
			\right]
		}\le C\eps_1^2.
		\]
		
		%
		%
	\end{enumerate}	
	Moreover, suppose that we have an initial estimator $\hat{\bu}_{[0]}$ for the $j$-th singular vectors:
	\begin{equation}\label{eq:nontriv-init-def}
		\max_{1\le q\le p}
		\sin\angle\left(
		\hat{\bu}_{[0]},\bu_j^{(q)}
		\right)\le\frac{1}{4}.
	\end{equation}
	With the above initialization, we then run power iteration on the deflated version:
	\begin{equation}\label{eq:piter-def}
		\hat{\bu}_{[t+1]}^{(q)}=
		\frac{1}{\|\bv\|}\bv 
		\quad
		\text{where}
		\quad 
		\bv=\scrX_j\times_{p\neq q}\bu_{[t]}^{(q)}
	\end{equation}
	where $\scrX_j$ is as defined in \eqref{eq:deflate}. The complete algorithm is specified in Algorithm~\ref{alg:alg-odeco-svd}. We will show that the power iteration estimator satisfies
	\begin{equation}\label{eq:ind-hyp}
		\max_{1\le q\le p}\sin\angle\left(\hat{\bu}^{(q)}_i,\bu^{(q)}_i\right)\le
		\dfrac{\max_{i,k}|\scrE\times_{k\neq q}\bu^{(q)}_i|}{\lambda_i}:=\dfrac{C\eps_1}{\lambda_i}.
	\end{equation}	
	The proof is by induction over $1\le i\le r$. Note that the base case of the induction for $i=1$ is satisfied by Proposition~\ref{pr:piter}.
	
	
	\begin{theorem}\label{th:piter} Under assumptions A1-A2, after $T\ge C\log \eps_1$ steps of the power iteration \eqref{eq:piter-def} initialized at estimators satisfying \eqref{eq:nontriv-init-def}, the estimators $\hat{\bu}_j^{(q)}:=\bu^{(q)}_{j,[T]}$ satisfy
		$$
		\max_{1\le q\le p}
		\sin\angle(\hat{\bu}_j^{(q)},\bu_j^{(q)})
		\le \dfrac{C\eps_1}{\lambda_j}
		$$
		and
		\begin{equation}\label{eq:fo-err-bd-def}
			\max_{1\le q\le p}
			\sin\angle\left(
			{\rm sign}(\langle
			\hat{\bu}_1^{(q)},\bu_1^{(q)}
			\rangle)
			\bu^{(q)}_j,
			\hat{\bu}_{j}^{(q)}
			-\dfrac{1}{\lambda_j}\scrE\times_{k\neq q}\bu_j^{(k)}
			\right)
			\le
			\dfrac{C(\eps_1^2+\eps_1\eps_2)}{\lambda_j^2}.
		\end{equation}	
	\end{theorem}

	
	%
	%

	\section{Initialization.}\label{sec:init}
	The non-convex nature of the odeco approximation dictates that the power iteration described in the previous section depends crucially on the initialization of the singular vectors of $\scrT$.

    Two obvious but sub-optimal choices initialization are through i) random starts, and ii) matricization of $\scrT$. Firstly, random start from the unit sphere requires a strong signal requirement of $\lambda_{\min}\ge Cd^{(p-2)/2}\|\scrE\|$ for nontrivial initializations that obey \eqref{eq:nontriv-init} and \eqref{eq:nontriv-init-def}. See, e.g., \cite{arous2019landscape,richard2014statistical} for the special case of rank one tensor SVD where $\scrE$ is a noise tensor of iid Gaussian errors. As we will show below, this signal requirement is sub-optimal since the random start does not leverage sufficient information about the odeco structure, and hence can be weakened. Secondly, the naive matricization of $\scrT$ once again faces the eigengap problem. Following our discussion in Section~\ref{sec:perturb}, since the matrix SVD is uniquely identified only in the presence of eigengap, an odeco tensor $\scrT$ with repeated singular values, even without noise, cannot be recovered from its matricization.

    In practice, especially in spectral learning problems where the odeco estimation is most relevant, some auxiliary information about $\scrE$ is often available, and the signal strength requirement might be much lower. We refer the interested reader to \cite{anandkumar2014guaranteed,anand2014sample,auddy2020perturbation,belkin2018eigenvectors} for a more detailed discussion.

	\subsection{General Strategies.}
	A general strategy for initialization proceeds in two steps. In the first step, one applies HOSVD to the observed tensor $\scrX$ to identify the linear span of the singular vectors along one or more modes of the odeco tensor $\scrT$. For example, the projection matrices, say
    \[
    \bP^{(1)}=\sum_{i=1}^r\bu_i^{(1)}(\bu_i^{(1)})^{\top}
    \quad
    \text{and}
    \quad 
    \bP^{(1,2)}=\sum_{i=1}^r(\bu_i^{(1)}\otimes \bu_i^{(2)})
    (\bu_i^{(1)}\otimes \bu_i^{(2)})^{\top}
    \]
    can be estimated by 
    \[
    \hat{\bP}^{(1)}=\sum_{i=1}^r\hat{\bu}_i^{(1)}(\hat{\bu}_i^{(1)})^{\top}
    \quad
    \text{and}
    \quad 
    \hat{\bP}^{(1,2)}=\sum_{i=1}^r(\hat{\bu}_i^{(1,2)})
    (\hat{\bu}_i^{(1,2)})^{\top}
    \]
    where $\{\hat{\bu}_i^{(1)}:1\le i\le r\}$ and $\{\hat{\bu}_i^{(1,2)}:1\le i\le r\}$ are the left singular vectors of ${\sf Mat}_{1}(\scrX)$ and ${\sf Mat}_{(1,2)}(\scrX)$ respectively. Note that $\bP^{(1)}\in \RR^{d\times d}$ and $\bP^{(1,2)}\in \RR^{d^2\times d^2}$, but both have rank $r$, and since $\scrT$ is odeco, we must have $r\le d$. Thus estimating the projection matrices will allow us to reduce the dimensionality of the problem from $d$ or $d^2$ to $r$.

    The second step slices one or more modes of $\scrX$ with random vectors. For example, one may slice $\scrX$ along the first mode with $\hat{\bP}^{(1)}\btheta$, or along the two modes with $\hat{\bP}^{(1,2)}\btheta'$ yielding matrices
    \[
        \bX_{\btheta}
        =\scrX\times_1(\hat{\bP}^{(1)}\btheta)
        \quad
        \text{and}
        \quad 
        \bX'_{\btheta'}
        =\scrX\times_{1,2}(\hat{\bP}^{(1,2)}\btheta')
    \]
    respectively where $\btheta\sim N(\mathbf{0},\II_d)$ and $\btheta'\sim N(\mathbf{0},\II_{d^2})$ are independent of $\scrX$. The randomness in $\btheta$ forces the singular values of $\bX_{\btheta}$ and $\bX'_{\btheta'}$ to be distinct. Moreover, repeating this initialization procedure for $L$ vectors $\btheta_1,\dots,\btheta_L$ ensures that the singular value gap is large enough, so that we can recover $\bu_i^{(1)}$ and $\bu_i^{(2)}$ from the left and right singular vectors of $\bX_{\theta_{i_*}}$ for some $1\le i_*\le L$. The basic initialization procedure is summarized in Algorithm~\ref{alg:alg-init}. 

    \begin{algorithm}[htb]
		\caption{Initialization for Odeco Approximation}\label{alg:alg-init}
		\hspace*{\algorithmicindent} \textbf{Input:} tensor $\scrX$ whose rank-$r$ odeco approximation is to be computed
		\begin{algorithmic}[1]
         \State Compute $\hat{\bU}^{(1)}\in \RR^{d\times r}$ as the matrix of left singular vectors of ${\sf Mat}_1(\scrX)$.
        \State Compute $\hat{\bP}^{(1)}=\hat{\bU}^{(1)}(\hat{\bU}^{(1)})^{\top}$.
        \For {$l=1$ to $L$}
        \State Generate a $d$ dimensional standard Gaussian vector $\btheta_l$.
        \State Compute the leading singular value of ${\sf Mat}_1(\scrX\times_{1}(\hat{\bP}^{(1)}\btheta_l))$ denoted by $\sigma_l$.
        \For {$q=2$ to $p$}
        \State Compute the leading left singular vector of ${\sf Mat}_{q-1}(\scrX\times_{1}(\hat{\bP}^{(p)}\btheta_l))$, denoted by 
        $\hat{\bu}_l^{(q)}$.
        \EndFor
        \State Compute $\scrX\times_{2\le q\le p}\hat{\bu}^{(q)}_l/\|\scrX\times_{1\le q\le p-1}\hat{\bu}^{(q)}_l\|$, denoted by 
        $\hat{\bu}_l^{(1)}$.
        \EndFor
        \State Compute $L^*={\rm argmax}_{1\le l\le L}\sigma_l$.\\
		\Return $\{\hat{\bu}_{L_*}^{(q)}:1\le q\le p\}$.
		\end{algorithmic}
	\end{algorithm}

    Theoretical analysis of the initialization algorithm depends on the nature of the error tensor $\scrE$. For illustration, we shall give an example where $\scrE$ is an iid ensemble of random variables, i.e., $E_{i_1i_2\dots i_p}\stackrel{iid}{\sim } E$, $1\le i_1,\dots,i_p\le d$.

	\begin{theorem}\label{th:init} If $\scrE$ is an iid ensemble whose entries satisfy $\EE(E^8)<C$, $L>Cr^2\log d$, and $\lambda_j\ge Cd^{p/4}\sqrt{\log d}$, there exists a numerical constant $C>0$, a permutation $\pi:[r]\to [r]$ such that 
		$$
        \max_{1\le q\le p}
        \sin\angle
        \left(\hat{\bu}_j^{(q)},\bu_{\pi(j)}^{(q)}\right)\le 
        \dfrac{Cd^{(p-1)/2}r\log d+C\lambda_j\sqrt{dr}(\log d)^{3/2}}{\lambda_j^2}$$
		with probability at least $1-d^{-1}$.
	\end{theorem}	
	In other words, we have a nontrivial initialization, i.e., \eqref{eq:nontriv-init} and \eqref{eq:nontriv-init-def} hold, whenever
	\begin{equation}\label{eq:signal-reqd-gen}
		\lambda_j\ge C\max\left\{
		d^{p/4}\sqrt{(\log d)},\,\,
        d^{(p-1)/4}
		\sqrt{r(\log d)}
		\right\}.
	\end{equation}

    \subsection{Refinement for Higher Order Incoherent Tensors.}

    The signal strength requirement above can be further improved for higher order tensors (i.e., $p\ge 4$) and $\scrT$ is \emph{incoherent}. The key idea is that incoherence allows us split $\scrX$ into two halves along some mode. We will then use one half of $\scrX$ for estimating the projection matrix $\bP^{(1,2)}$, and the other half for random slicing. The independence between the two halves allows us to sharpen the analysis and reduce the signal strength requirement. The detailed procedure can be found in Algorithm~\ref{alg:alg-init-incoherent}. A similar strategy was used for rank one tensor SVD in \cite{auddy2021estimating}. 

    \begin{algorithm}[htb]
		\caption{Initialization for Incoherent Odeco Approximation }\label{alg:alg-init-incoherent}
		\hspace*{\algorithmicindent} \textbf{Input:} tensor $\scrX$ whose rank-$r$ odeco approximation is to be computed
		\begin{algorithmic}[1]
        \State Randomly split $\scrX$ into two halves $\scrX_{[1]}$ and $\scrX_{[2]}$ along the $p$-th mode.
        \State Compute $\hat{\bU}^{(1,2)}\in \RR^{d^2\times r}$ as the matrix of left singular vectors of ${\sf Mat}_{(1,2)}(\scrX_{[1]})$.
        \State Compute $\hat{\bP}^{(1,2)}=\hat{\bU}^{(1,2)}(\hat{\bU}^{(1,2)})^{\top}$.
        \For {$l=1$ to $L$}
        \State Generate a $d^2$ dimensional standard Gaussian vector $\btheta_l$.
        \State Compute the leading singular value of ${\sf Mat}_1(\scrX_{[2]}\times_{1,2}(\hat{\bP}^{(1,2)}\btheta_l))$ denoted by $\sigma_l$.
        \For {$q=3$ to $p$}
        \State Compute the leading left singular vector of ${\sf Mat}_{q-2}(\scrX_{[2]}\times_{1,2}(\hat{\bP}^{(1,2)}\btheta_l))$, denoted by 
        $\hat{\bu}_l^{(q)}$.
        \EndFor
        \State Compute the leading left and right singular vectors of $
        \scrX_{[2]}\times_{3\le q\le p}\hat{\bu}^{(q)}_l$, denoted by 
        $\hat{\bu}_l^{(1)}$ and $\hat{\bu}_l^{(2)}$.
        \State Compute $\bv=\scrX\times_{1\le q\le (p-1)}\hat{\bu}^{(q)}_l$, update $\hat{\bu}^{(p)}_l$ to be $\bv/\|\bv\|$.
        \EndFor
        \State Compute $L^*={\rm argmax}_{1\le l\le L}\sigma_l$.\\
		\Return $\{\hat{\bu}_{L_*}^{(q)}:1\le q\le p\}$.
		\end{algorithmic}
	\end{algorithm}

    We now show how random slicing enables nontrivial estimation of each singular vector of $\scrT$. Without loss of generality, let us focus on fourth order tensors and singular vectors along the first mode, i.e., $\bu^{(1)}_i$. Suppose that $\hat{\bP}^{(1,2)}$ is the projection estimator computed by HOSVD on an independent copy of $\scrX$. For $\btheta\in \RR^{d^2}$ we obtain
	\begin{equation}\label{eq:slicing-order4}
    {\sf Mat}_1(\scrX\times_{1,2}\hat{\bP}^{(1,2)}\btheta)
	=\sum_{i=1}^r
    \lambda_i
    \langle \bu^{(1)}_i\otimes\bu^{(2)}_i, \hat{\bP}^{(1,2)}\btheta\rangle 
	\bu^{(3)}_i(\bu^{(4)})^{\top}
	+{\sf Mat}_1(\scrE\times_{1,2}\hat{\bP}^{(1,2)}\btheta).
	\end{equation}
    Since $\hat{\bP}^{(1,2)}$ is assumed to be independent of $\scrX$ and hence $\scrE$, the entries of the matrix ${\sf Mat}_3(\scrE\times_{1,2}\hat{\bP}^{(1,2)}\btheta)$ are independent copies of a random variable with mean zero, and variance $\|\hat{\bP}^{(1,2)}\btheta\|^2$.

    On the other hand, by sampling random Gaussian vectors $\btheta\sim N(\mathbf{0},\II_{d^2})$, it is possible to create a large enough gap between the top two singular values. Since $\hat{\bP}^{(1,2)}$ is close to $\bP^{(1,2)}$,
    the singular values of the above matrix are approximately 
    \[
\lambda_i\langle\bu_i^{(1)}\otimes\bu_i^{(2)},\bP^{(1,2)}\btheta\rangle
    =\lambda_i\langle\bu_i^{(1)}\otimes \bu_i^{(2)},\btheta\rangle.
    \]
    Since $\btheta\sim N(\mathbf{0},\II_d^2)$, $\langle\bu_i^{(1)}\otimes \bu_i^{(2)},\btheta\rangle \sim N(0,1)$. Now suppose this random slicing is repeated $L$ times. As $\btheta_l$ varies for $1\le l\le L$, $\{\langle\bu_i^{(1)}\otimes \bu_i^{(2)},\btheta_l\rangle:1\le i\le r,1\le l\le L\}$ gives a collection of $Lr$ independent Gaussian variables. 
	This allows us to ensure that the gap between the top two singular values of ${\sf Mat}_3(\scrX\times_{1,2}\hat{\bP}^{(1,2)}\btheta_{l_*})$ is sufficiently large for some $1\le l_*\le L$ as long as $L>2r^2\log d$. Consequently matrix perturbation bounds such as the Davis-Kahan theorem guarantees that $\bu_1^{(3)}$ can be accurately estimated by the top singular vectors of ${\sf Mat}_3(\scrX\times_{1,2}\hat{\bP}^{(1,2)}\btheta_{l_*})$. 

    The steps outlined above depend on the existence of $\hat{\bP}^{(1,2)}$ computed from an independent copy of $\scrX$. Of course, this independent copy is not available in practice. However, we can randomly split $\scrX$ into two halves $\scrX_{[1]}$ and $\scrX_{[2]}$, which would be independent since $\scrE$ is an iid ensemble. Next, we compute $\hat{\bP}^{(1,2)}$ from $\scrX_{[1]}$, while using the projection matrix for random slicing on $\scrX_{[2]}$. To ensure that the signal strength is evenly spread across $\scrX_{[1]}$ and $\scrX_{[2]}$, we require the singular vector matrix $\bU^{(q)}$ from at least one mode of $\scrT$ to satisfy an incoherence condition. Without loss of generality, we assume this is satisfied for the $p$-th mode:
    \[
        \max_{1\le i\le d,1\le j\le r} |u^{(p)}_{ij}|
        \le \frac{1}{C\log(d)}
    \]
    for a constant $C>0$. 

    With the above construction, the following theorem states the accuracy of the initialization method in Algorithm~\ref{alg:alg-init-incoherent}, proceeding sequentially through the initialization-power iteration-deflation routine outlined in Section~\ref{sec:compute}. For $1\le j\le r$, Algorithm~\ref{alg:alg-init-incoherent} is run with the deflated tensor $\scrX_j$ defined in \eqref{eq:deflate}. We denote the estimators at the $j$-th step by $\{\hat{\bu}_j^{(q)}:1\le q\le p\}$.
	
	\begin{theorem}\label{th:init-incoh} If $\scrE$ is an iid ensemble whose entries satisfy $\EE(E^8)<C_0$, $L>C_0r^2\log d$, $\max_{1\le i\le d,1\le j\le r} |u^{(p)}_{ij}| \le (C_0\log(d))^{-1}$ and $\lambda_j\ge C_0d^{p/4}\log(d)$ for a constant $C_0>0$, there exists another numerical constant $C>0$, a permutation $\pi:[r]\to [r]$ such that 
		$$
        \max_{1\le q\le p}
        \sin\angle
        \left(\hat{\bu}_j^{(q)},\bu_{\pi(j)}^{(q)}\right)\le 
        \dfrac{Cd^{(p-2)/2}r(\log d)+C\lambda_j\sqrt{dr}(\log d)}{\lambda_j^2}
        +
        \frac{\mathbbm{1}(p=4)}{\sqrt{\log d}}    
        $$
	for $1\le j\le r$, with probability at least $1-d^{-1}$.
	\end{theorem}	
	Since $r\le d$, we have a nontrivial initialization, in that \eqref{eq:nontriv-init} and \eqref{eq:nontriv-init-def} are satisfied, whenever
	\begin{equation}\label{eq:signal-reqd-incoh}
		\lambda_j\ge C d^{p/4}(\log d)
	\end{equation}
	for a constant $C>0$. Note that the above signal strength requirement matches the computational lower bound for tensor SVD from \cite{zhang2018tensor} up to the multiplied logarithmic term.
    
    With this initialization we can run the power iteration and deflation procedure in Algorithm~\ref{alg:alg-odeco-svd} to find estimates $\hat{\bu}_{j}^{(q)}$ that satisfy:
    \[
    \max_{1\le q\le p}
    \sin\angle\left(
    \hat{\bu}_{j}^{(q)},
    {\bu}_{\pi(j)}^{(q)}
    \right)\le 
    \frac{C\sqrt{d\log(d)}}{\lambda_j}
    \quad
    \text{for}
    \quad
    1\le j\le r
    \]
    with probability at least $1-d^{-1}$, where we use Theorem~\ref{th:piter} and the bound the error related quantities from Lemma~\ref{lem:eps012-bds} in the last step. In fact the fine-grained error analysis from Theorem~\ref{th:piter} enables us to infer linear combinations of $\bu_j^{(q)}$ accurately, as the following theorem shows.

    \begin{theorem}\label{th:asy-dist-rand}
        Suppose $\scrE$ is an iid ensemble whose entries satisfy $\EE(E^8)<C_0$, $L>C_0r^2\log d$, $\max_{1\le i\le d,1\le j\le r} |u^{(p)}_{ij}| \le 1/C_0\log(d)$ and $\lambda_j\ge C_0d^{p/4}(\log d)^2$ for a constant $C_0>0$. Then there exists a constant $C>0$ and a permutation $\pi:[r]\to [r]$ such that if  $\lambda_j>Cd^{p/4}(\log d)$, 
			$$
			\abs*{
					\langle\bu^{(q)}_{\pi(j)},\hat{\bu}^{(q)}_j\rangle^2-\left(1-\dfrac{d-1}{\lambda_j^2}\right)}
			\le \dfrac{Cd^{3/2}}{\lambda_j^3}
			$$
			with probability at least $1-d^{-1}$. Moreover, under the same assumptions
			$$
			\calW_1\left(
			\dfrac{\lambda_j^2}{\sqrt{d}}
			\left\{\langle\bu^{(q)}_{\pi(j)},\hat{\bu}^{(q)}\rangle^2-\left(1-\dfrac{d-1}{\lambda_j^2}\right)\right\},\,
			\sigma_jZ 
			\right)\le \dfrac{Cd^{3/4}}{\lambda_j}.$$
			Here $Z\sim N(0,1)$ and $\sigma_j^2={\rm Var}(E_i^2)$, where $E_i$ are the iid random variables defined as $E_i=\scrE\times_q\be_i\times_{k\neq q}\bu^{(k)}_{\pi(j)}$ for $1\le i\le d$. 
		
			On the other hand, for any $\ba\perp\bu^{(1)}_{\pi(j)}$, we have for $\gamma^{(q)}_j={\rm sign}(\langle\hat{\bu}^{(q)}_j,\,\bu^{(q)}_j\rangle)$ that
			$$
			\abs*{
					\langle\ba,\hat{\bu}_j^{(q)}\rangle
					-\dfrac{\gamma_j^{(q)}\langle\ba,\bu_{\pi(j)}^{(q)}\rangle}{1-d\lambda_j^{-2}}
					-\dfrac{\scrE\times_q\ba\times_{k\neq q}\bu^{(k)}_{\pi(j)}}{\lambda_j}
				}\le \dfrac{Cd^{3/4}}{\lambda_j^2}
			$$
			with probability at least $1-d^{-1}$.
		\end{theorem}

	
	

	\section{Concluding Remarks}
    In this work, we provided a general framework for spectral learning with orthogonally decomposable tensors. Our results highlight three central themes: (i) perturbation bounds for odeco tensors demonstrate robustness to noise without requiring eigengaps, (ii) iterative algorithms such as power iteration and deflation achieve statistically optimal recovery once equipped with suitable initialization, and (iii) initialization itself remains the principal computational bottleneck, revealing an intrinsic statistical–computational gap.

    Taken together, these findings clarify both the advantages of tensor methods over matrix analogues and the limitations imposed by non-convex landscapes. Beyond their theoretical significance, our results inform practical algorithm design for latent variable models, independent component analysis, and related applications.
    
    Several directions remain open. Tightening lower bounds for initialization under weaker signal conditions would sharpen the boundary between statistical and computational feasibility. Extending our analysis to non-odeco settings or structured noise models also presents important challenges. We hope that this work helps to bridge the gap between statistical theory and algorithmic practice in high-dimensional tensor analysis.
	
	\bibliographystyle{plainnat}
	\bibliography{review-refs}

	\section{Proofs}
	


    \subsection{Proof of Result in Section~\ref{sec:odeco-dec}}

    \begin{proof}[Proof of Proposition~\ref{pr:noiseless-tpi}]
        Let us fix $q=1$, $k=1$ and consider $j={\rm argmax}\lambda_k\prod_{q'=3}^p
        |\langle\bu_{1,[0]}^{(q')},\bu_k^{(q')} \rangle|$. Then by definition of the fixed point iteration:
        \begin{align*}
            \frac{\lambda_j
        |\langle\bu_{1,[1]}^{(q')},\bu_j^{(q')} \rangle|    
        }{
        \max_{k\neq j}
        \lambda_k
        |\langle\bu_{1,[1]}^{(q')},\bu_k^{(q')} \rangle|    
        }
        =\frac{
        (\lambda_j
            \prod_{q'=2}^p
        |\langle\bu_{1,[0]}^{(q')},\bu_j^{(q')} \rangle|)
        }{
        \max_{k\neq j}
        (\lambda_k\prod_{q'=2}^p
        |\langle\bu_{1,[0]}^{(q')},\bu_k^{(q')} \rangle|)
        }
        >1
        \end{align*}
        where the last inequality is ensured due to random initialization. Repeating this argument over $q=1,\dots,p$ and a recursion over $t=1,2,\ldots$ gives
        \begin{align*}\label{eq:element-ratio-rec}
            \frac{\lambda_j
        |\langle\bu_{1,[t]}^{(q')},\bu_j^{(q')} \rangle|    
        }{
        \max_{k\neq j}
        \lambda_k
        |\langle\bu_{1,[t]}^{(q')},\bu_k^{(q')} \rangle|    
        }
        =
        \left\{\frac{
        (\lambda_j
            \prod_{q'=2}^p
        |\langle\bu_{1,[0]}^{(q')},\bu_j^{(q')} \rangle|)
        }{
        \max_{k\neq j}
        (\lambda_k\prod_{q'=2}^p
        |\langle\bu_{1,[0]}^{(q')},\bu_k^{(q')} \rangle|)
        }
        \right\}^{(p-1)^{t-1}}.\numberthis
        \end{align*}
        Next,
        \begin{align*}
            \sup_{\bw\in \SS^{d_1-1},\bw\perp\bu_j^{(1)}}
            \langle\bw, \scrT\times_{2\le q\le p}\bu_{j,[t+1]}^{(q)}\rangle
            =&~\sum_{k\neq j}
            \lambda_k\langle\bu_k^{(1)},\bw\rangle
            \prod_{q=2}^p\langle\bu_k^{(q)},\bu_{k,[t]}^{(q)}\rangle\\
            \le&~
            \left(
            \sum_{k\neq j}
            \frac{1}{\lambda_k^{2(p-2)}}
            \right)^{1/2}
            \max_{k\neq j}
            (\lambda_k^{p-1}\prod_{q'=2}^p
        |\langle\bu_{1,[t]}^{(q')},\bu_k^{(q')} \rangle|)
            \end{align*}
         On the other hand
         \begin{align*}
             \|\scrT\times_{2\le q\le p}\bu_{j,[t]}^{(q)}\|
             \ge &~
             |\scrT\times_1\bu_j^{(1)}\times_{1\le q\le p}\bu_{1,[t]}^{(q)}\|
             =\lambda_j
             (\prod_{q'=2}^p
        |\langle\bu_{1,[t]}^{(q')},\bu_j^{(q')} \rangle|),
         \end{align*}
         which implies:
         \[
         L_{T,j}
         =\max_{1\le q\le p}
         \sin\angle\left(
         \bu_{1,[t]}^{(q)},\bu_j^{(q)}
         \right)
         \le \lambda_j^{(p-2)}
         \left(
            \sum_{k\neq j}
            \frac{1}{\lambda_k^{2(p-2)}}
            \right)^{1/2}
         \left\{\frac{
        (\lambda_j
            \prod_{q'=2}^p
        |\langle\bu_{1,[0]}^{(q')},\bu_j^{(q')} \rangle|)
        }{
        \max_{k\neq j}
        (\lambda_k\prod_{q'=2}^p
        |\langle\bu_{1,[0]}^{(q')},\bu_k^{(q')} \rangle|)
        }
        \right\}^{(p-1)^{t}}.  
         \]
        Since $p\ge 3$, the above implies that $\bu_{1,[t]}^{(1)}$ converges to $\bu_{2}^{(1)}$ at least quadratically fast.
         
         Thus continuing the power iteration for sufficiently large $T$ ensures $L_{T,j}\le \min\{\eps,1/d^2\}$ for any pre-specified tolerance $\eps$. Since $L_{T,j}$ can be made arbitrarily small for any $j\in \calS$, the error rates remain small even after Gram-Schmidt orthogonalization. To ensure that all components for $1\le j\le r$ are initialized at least once, we use the projection onto the orthogonal complement of the already components. Then repeating the above argument finishes the proof.
    \end{proof}
	
	\subsection{Proofs of Results in Section~\ref{sec:compute}}

    \begin{proof}[Proof of Proposition~\ref{pr:piter}]
        The proof follows from the proof of Theorem~\ref{th:piter} in the special case $j=1$ noting that $\scrX_1=\scrX$.
    \end{proof}

    \medskip
    
	\begin{proof}[Proof of Theorem \ref{th:piter}]
		We will prove the stronger result described below.
		$$L_{t+1}\le \dfrac{\eps_1}{\lambda_j}+\dfrac{C(1+\eps_1)\eps_1^2}{\lambda_j^2}
		+\dfrac{CL_t(1+\eps_1)(\eps_1+\eps_2)}{\lambda_j}
		+CL_t^2(1+\eps_1)\left(1+\dfrac{\|\scrE\|}{\lambda_j}\right).$$
		Moreover, defining $\gamma_j^{(q)}:={\rm sign}(\langle\hat{\bu}^{(q)}_{j,[t]},\bu^{(q)}_j\rangle)$, we have
		\begin{equation}\label{eq:secopert}
			\sin\angle\left(
			\gamma_j^{(q)}\bu^{(q)}_j,
			\hat{\bu}_{j,[t+1]}^{(q)}
			-\dfrac{1}{\lambda_j}\scrE\times_{k\neq q}\bu_j^{(k)}
			\right)
			\le \dfrac{C(1+\eps_1)(L_t(\eps_1+\eps_2)+L_t^2(\lambda_j+\|\scrE\|)
				+\eps_1^2/\lambda_{j})}{\lambda_j}.
		\end{equation}
		By the initialization procedure, and also since we choose the largest singular value at the end of each iteration, we have nontrivial initialization for the component satisfying
		$\lambda_j= \max_{i\ge j}\lambda_i$. By the induction hypothesis, we have estimates $\hat{\bu}^{(q)}_{[t]}$ such that 
		$$
		\norm*{\bu^{(q)}_j-{\rm sign}((\bu^{(q)}_j)^\top\hat{\bu}^{(q)}_{[t]})
			\hat{\bu}^{(q)}_{[t]}}
		\le \sqrt{2}L_t.
		$$
		To simplify the proof, we assume that ${\rm sign}((\bu^{(q)}_j)^\top\hat{\bu}^{(q)}_{[t]})=1$. The other cases can be handled by reversing the sign of $\hat{\bu}^{(q)}_{j,[t]}$ accordingly. We write the proof for $q=1$. We have, for any $\bv\perp \bu^{(1)}_j$ that
		\begin{equation}\label{eq:piter-num}
			\langle\bv,\scrX_j\times_{q\neq 1}\hat{\bu}^{(q)}_{j,[t]}\rangle
			=\sum_{i> j}\lambda_i
			\langle\bv,\bu^{(1)}_i\rangle
			\prod_{q>1}\langle\hat{\bu}^{(q)}_{[t]},\bu^{(q)}_i\rangle
			+\scrE\times_1\bv\times_{q>1}\hat{\bu}^{(q)}_{[t]}
			+\hat{\scrT}_{-j}\times_1\bv\times_{q>1}\hat{\bu}^{(q)}_{[t]}
		\end{equation}
		By Cauchy-Schwarz inequality, the first term is
		\begin{equation}\label{eq:piter-sigbd}
			\abs*{\sum_{i> j}\lambda_i
				\langle\bv,\bu^{(1)}_i\rangle
				\prod_{q>1}\langle\hat{\bu}^{(q)}_{[t]},\bu^{(q)}_i\rangle}
			\le \max_{i>j}\{\lambda_i |\langle\bv,\bu^{(1)}_i\rangle|\}
			\prod_{q=2}^3
			\sin\angle\left(\hat{\bu}^{(q)}_{[t]},\bu^{(q)}_j\right)
			\le \lambda_jL_t^2
		\end{equation}
		For the second term, note that
		\begin{align*}\label{eq:piter-errbd}
			\scrE\times_1\bv\times_{q>1}\hat{\bu}^{(q)}_{[t]}
			=&\scrE\times_1\bv\times_{q>1}(
			\bu^{(q)}_j+
			\hat{\bu}^{(q)}_{j,[t]}-\bu^{(q)}_j)\\
			\le& \scrE\times_1\bv\times_{q>1}\bu^{(q)}_j
			+\sqrt{2}L_t\cdot(p-1)\max_{2\le q\le p}\norm*{\scrE\times_{q'>1,\,q'\neq q}\bu^{(q')}_j}\\
			&+\sum_{A\subset [p]/\{1\},\, |A|\ge 2}
			\scrE\times_1\bv\times_{q\in A}(\hat{\bu}^{(q)}_{j,[t]}-\bu^{(q)}_j)
			\times_{q'\notin A}\bu^{(q')}_j\\
			\le&
			\scrE\times_1\bv\times_{q>1}\bu^{(q)}_j
			+CL_t\eps_2
			+CL_t^2\|\scrE\|.\numberthis
		\end{align*}	
		The third term arising from the deflation errors can be bounded by Lemma \ref{lem:defl-er}. Plugging in the upper bounds from \eqref{eq:piter-sigbd}, \eqref{eq:piter-errbd} and Lemma~\ref{lem:defl-er} into \eqref{eq:piter-num}, yields
		\begin{equation}\label{eq:sinth-num}
			\sup_{\bv\perp\bu_j^{(1)}}
			\langle\bv,\scrX_j\times_{q\neq 1}\hat{\bu}^{(q)}_{j,[t]}
			-\scrE\times_{q>1}\bu_j^{(q)}
			\rangle
			\le 
			CL_t(\eps_1+\eps_2)+CL_t^2(\lambda_j+\|\scrE\|)
			+C\eps_1^2/\lambda_{j-1}.
		\end{equation}
		It remains to bound the norm. We have
		\begin{align*}\label{eq:sinth-den}
			\|\scrX_j\times_{q\neq 1}\hat{\bu}^{(q)}_{j,[t]}\|
			\ge & \scrX_j\times_1\bu_{j}^{(1)}\times_{q\neq 1}\hat{\bu}^{(q)}_{j,[t]}\\
			\ge & \lambda_j\prod_{q>1}\langle \hat{\bu}^{(q)}_{j,[t]},\bu^{(q)}_j\rangle
			-\scrE\times_1\bu^{(1)}_j\times_{q> 1}\hat{\bu}^{(q)}_{j,[t]}
			-\hat{\scrT}_{-j}\times_1\bu^{(1)}_j\times_{q\neq 1}\hat{\bu}^{(q)}_{j,[t]}\\
			\ge& \lambda_j(1-L_t^2)-\scrE\times_{q\ge 1}\bu^{(q)}_j-CL_t\eps_2-CL_t^2\|\scrE\|
			-C\eps_1^2/\lambda_{j-1}\numberthis
		\end{align*}
		following the steps of \eqref{eq:piter-sigbd}, \eqref{eq:piter-errbd} and Lemma~\ref{lem:defl-er} with $\bu_j^{(1)}$ instead of $\bv\perp\bu_j^{(1)}$.
		Dividing \eqref{eq:sinth-num} by \eqref{eq:sinth-den}, one has
		\begin{align*}
			\sup_{\bv\perp\bu_j^{(1)}}
			\dfrac{
				\langle\bv,\scrX_j\times_{q\neq 1}\hat{\bu}^{(q)}_{j,[t]}
				-\scrE\times_{q>1}\bu_j^{(q)}
				\rangle}{\|\scrX_j\times_{q\neq 1}\hat{\bu}^{(q)}_{j,[t]}\|}
			\le &
			\dfrac{CL_t(\eps_1+\eps_2)+CL_t^2(\lambda_j+\|\scrE\|)
				+C\eps_1^2/\lambda_{j}}{\lambda_j(1-L_t^2)-\scrE\times_{q\ge 1}\bu^{(q)}_j-CL_t\eps_2-CL_t^2\|\scrE\|
				-C\eps_1^2/\lambda_j}\\
			\le & \dfrac{CL_t(\eps_1+\eps_2)+CL_t^2(\lambda_j+\|\scrE\|)
				+C\eps_1^2/\lambda_{j}}{\lambda_j}
		\end{align*}
		since by assumption, $L_t^2\le 0.1$ and $\lambda_j\ge 3\scrE\times_{q\ge 1}\bu^{(q)}_j+CL_t\eps_2-CL_t^2\|\scrE\|
		+C\eps_1^2/\lambda_j$. Moreover, by similar calculations, we also have that 
		\begin{align*}
			\sup_{\bv\perp\bu_j^{(1)}} \scrE\times_1\bv\times_{q>1}\bu_j^{(q)}
			\left(\dfrac{1}{\lambda_j}-\dfrac{1}{\|\scrX_j\times_{q\neq 1}\hat{\bu}^{(q)}_{j,[t]}\|}\right)
			\le &
			C\eps_1
			\cdot
			\dfrac{L_t(\eps_1+\eps_2)+L_t^2(\lambda_j+\|\scrE\|)
				+\eps_1^2/\lambda_{j}}{\lambda_j}.
		\end{align*}
		Combining all the bounds we have
		\begin{align*}\label{eq:sinth-fo-bd}
			\sin\angle\left(\bu^{(1)}_j,
			\hat{\bu}_{j,[t+1]}^{(1)}
			-\dfrac{1}{\lambda_j}\scrE\times_{q>1}\bu_j^{(q)}
			\right)
			\le \dfrac{C(1+\eps_1)(L_t(\eps_1+\eps_2)+L_t^2(\lambda_j+\|\scrE\|)
				+\eps_1^2/\lambda_{j})}{\lambda_j}.\numberthis
		\end{align*}
		Thus,
		$$
		L_{t+1}\le \dfrac{\eps_1}{\lambda_j}+\dfrac{C(1+\eps_1)\eps_1^2}{\lambda_j^2}
		+\dfrac{CL_t(1+\eps_1)(\eps_1+\eps_2)}{\lambda_j}
		+CL_t^2(1+\eps_1)\left(1+\dfrac{\|\scrE\|}{\lambda_j}\right).
		$$
	\end{proof}
	
	\medskip
	
	\begin{lemma}\label{lem:esti-op-norm}
		We define $\eps_1:=\max\limits_{j,k}\|\scrE\times_{q\neq k}\bu^{(q)}_j\|$ and $\eps_2:=\max\limits_{j,k_1,k_2}\|\scrE\times_{q\neq k_1,k_2}\bu^{(q)}_j\|$. Assume that $\scrE$ satisfies assumptions A1-A3.	Then there exist a constant $C>0$ and a permutation $\pi:[r]\to[r]$ such that the matrices 
		$
		\check{\bU}^{(q)}_{(i)}=[\hat{\bu}^{(q)}_1\,\dots\,
		\hat{\bu}^{(q)}_i]
		$ and 
		$
		\bU^{(q)}_{(i)}=[\gamma_1^{(q)}\bu^{(q)}_{\pi(1)}\,\dots\,\gamma_i^{(q)}\bu^{(q)}_{\pi(i)}]$.
		Then there exists a constant $C>0$ such that
		$$
		\max_{1\le q\le p}\,
		\norm*{\check{\bU}^{(q)}_{(i)}-\bU^{(q)}_{(i)}}\le \dfrac{C\eps_1}{\lambda_i}.
		$$
		Here $\gamma_i^{(q)}={\rm sign}\bigg(\left(\hat{\bu}_i^{(q)}\right)^{\top}\bu_{\pi(i)}^{(q)}\bigg)$.
	\end{lemma}
	
	\medskip
	\begin{lemma}\label{lem:defl-er}
		The deflation errors arising from the previously estimated vectors satisfy $$	\sup_{\bv}\hat{\scrT}_{-j}\times_1\bv\times_{q>1}\hat{\bu}^{(q)}_{[t]}
		\le C(\eps_0+\eps_1^2/\lambda_{j-1})\eps_1/\lambda_{j-1}
		+C(\eps_0+\eps_1^2/\lambda_{j-1})L_t
		+C\eps_1L_t^2
		$$
		where $\hat{\scrT}_{-j}=\displaystyle\sum_{i<j}
		(\lambda_i\bu^{(1)}_1\circ \dots\circ\bu^{(p)}_i-
		\hat{\lambda}_i\hat{\bu}^{(1)}_1\circ \dots\circ \hat{\bu}^{(p)}_i)$.
	\end{lemma}

	\subsection{Proofs of Results in Section~\ref{sec:init}.}
    The following results will be useful in the proof of nontrivial initialization.
    \begin{lemma}\label{lem:hosvd} If $\scrE$ is an iid ensemble whose entries satisfy $\EE(E)=0$, $\Var(E)=1$, $\EE(E^8)<C_0$ for some constant $C>0$, then there exists a constant $C>0$ such that 
	\[
    \norm*{\hat{\bP}^{(1)}-\bP^{(1)}}\le\dfrac{C(d^{p/2}+\lambda_r\sqrt{d})(\log d)}{\lambda_r^2}
    \quad
    \text{and}
    \quad 
    \norm*{\hat{\bP}^{(1,2)}-\bP^{(1,2)}}
    \le\dfrac{C(d^{p/2}+\lambda_rd)(\log d)}{\lambda_r^2}
    \]
	with probability at least $1-d^{-1}$.
	\end{lemma}

    \medskip
    
	\begin{lemma}\label{lem:sing-gap}
		If $L>2r^2\log d$, there is at least one $l\in[L]$ such that $Z_{1l},\dots,Z_{kl}\stackrel{iid}{\sim}N(0,1)$ satisfies
		$$
		\tilde{\lambda}_1^2Z_{1l}^2\ge 1.2\max_{k\neq 1}\tilde{\lambda}_k^2Z_{kl}^2
		$$
		for any set of numbers $\tilde{\lambda}_1\ge\dots\ge \tilde{\lambda}_k>0$, $1\le k\le r$, with probability at least $1-d^{-2}$.
	\end{lemma}

    \medskip
  	
	\begin{proof}[Proof of Theorem \ref{th:init-incoh}]
		To simplify notation, we assume without loss of generality that $\pi=Id$. Fixing $q=3$, we also drop the subscript indicating the tensor mode. We will define:
        \[
        \bX={\sf Mat}_1(\scrX_{[2]}\times_{1,2}\hat{\bP}^{(1,2)}\btheta)
        =
        {\sf Mat}_1(\scrT_2\times_{1,2}\hat{\bP}^{(1,2)}\btheta)
        +
        {\sf Mat}_1(\scrT_2\times_{1,2}\hat{\bP}^{(1,2)}\btheta)
        \]
        We define $\bv_j:=(\otimes_{4\le q\le (p-1)}\bu_j^{(q)}\otimes \bu_{j,2}^{(p)})$ for $j=1,\dots,r$. Then let us denote the singular values and vectors of $\scrT_2\times_{1,2}(\bP^{(1,2)}\btheta)$ by 
		$$
		\omega_j=\lambda_j|(\bu_j^{(1,2)})^\top\bP^{(1,2)}\btheta|,
		\,\,
		\Omega_j={\rm diag}(\omega_j,\dots,\omega_r)
		$$
        and
        $$
		\tilde{\bU}_j=[\bu_j,\dots,\bu_r],
		\,\,
		\tilde{\bV}_j=[\bv_j,\dots,\bv_r].
		$$
        Since $L>Cr^2\log d$, it is clear that there exists $\btheta_*=\btheta_l$ such that the conclusion of Lemma \ref{lem:sing-gap} is satisfied for $\lambda_{\pi(j)}\ge \dots\ge \lambda_{\pi(r)}$. 
        Moreover, we define the matricization of the deflation errors in the previous steps by:
        \begin{equation}\label{eq:def-delta-j}
        \bDelta_j
        ={\sf Mat}_1\left(
        \sum_{i<j}\left(\lambda_i\bu^{(1)}_i\circ\dots\circ \bu^{(p)}_i
		-
        \langle\scrX, \hat{\bu}^{(1)}_i\circ\dots\circ \hat{\bu}^{(p)}_{i,2}\rangle
        \,\hat{\bu}^{(1)}_i\circ\dots\circ \hat{\bu}^{(p)}_{i,2}
		\right)
        \times_{1,2}\hat{\bP}^{(1,2)}\btheta
        \right).
        \end{equation}
        Also define $m_j:=(\omega_j^2+d^{p-3}\|\hat{\bP}^{(1,2)}\btheta\|^2)$,  $\bZ:=\hat{\bP}^{(1,2)}\btheta$ and $\bE_\theta:={\sf Mat}_1(\scrE_2\times_{1,2}\bZ)$. Note that $\bE_\theta\in \RR^{d\times d^{p-3}}$. 
		
		\noindent Since $\btheta$ is independent of $\hat{\bP}^{(1,2)}$ which is a projection matrix of rank $r$ and hence satisfies $\tr(\hat{\bP}^{(1,2)})=r$, we have that $\btheta^\top\hat{\bP}^{(1,2)}\btheta\sim \chi^2({\rm trace}(\hat{\bP}^{(1,2)}))\equiv \chi^2_r$. By concentration inequalities for $\chi^2$ random variables (see e.g., \cite{vershynin2018high}), we have
		\begin{align*}
			\|\bZ\|^2=\btheta^\top\hat{\bP}^{(1,2)}\btheta\le r+2\sqrt{r\times r}+2r=4r
		\end{align*}
		with probability at least $d^2\exp(-r)$. Note that $\bE_\theta$ has iid elements with mean zero, variance $\|\bZ\|^2$ and finite eighth moment. Thus 
        \[
        \EE(\bE_\theta\bE_\theta^{\top})=d^{p-3}\|\bZ\|^2\II_d
        \]
        and standard matrix concentration inequalities proved in, e.g., \cite{bai1988note} show that
        \begin{equation}\label{eq:Etheta-conc}
        \|\bE_{\theta}\bE_{\theta}-d^{p-3}\|\bZ\|^2\II_d\|
        \le 
        d^{p-3}
        \sqrt{\dfrac{Cd(\log d)^2}{d^{p-3}}}\cdot\|\bZ\|^2
		\le Cd^{(p-2)/2}r\log d
        \end{equation}
        for $p\ge 5$. For $p=4$ we have $\|\bE_{\theta}\|\le C\sqrt{d\|\bZ\|^2}=C\sqrt{dr}$ by \cite{latala2005some} which again implies
        \[
        \|\bE_{\theta}\bE_{\theta}-d^{p-3}\|\bZ\|^2\II_d\|
        \le Cd^{(p-2)/2}r\log d        
        \]
        thus ensuring \eqref{eq:Etheta-conc} is satisfied for all $p\ge 4$. Finally, we bound the gap between top two singular values. Note that
		\begin{align*}
			\omega_1-\omega_2
			\ge &~\lambda_j\|\bu^{(p)}_{j,2}\|\times|(\bu_j^{(1,2)})^\top\hat{\bP}^{(1,2)}\btheta|
			-\lambda_{j+1}\|\bu^{(p)}_{(j+1),2}\|
            \max_{i\ge j}|(\bu_i^{(1,2)})^\top\hat{\bP}^{(1,2)}\btheta|\\
            \ge &~\lambda_j\|\bu^{(p)}_{j,2}\|\times|(\bu_j^{(1,2)})^\top\hat{\bP}^{(1,2)}\btheta|
			-\lambda_{j+1}\|\bu^{(p)}_{(j+1),2}\|
            \max_{i\ge j}|(\bu_i^{(1,2)})^\top\hat{\bP}^{(1,2)}\btheta|\\
            &~-2
            \max_{i\ge j}
            \lambda_j\|\bu^{(p)}_{j,2}\|\times|(\bu_j^{(1,2)})^\top
            (\hat{\bP}^{(1,2)}-{\bP}^{(1,2)})
            \btheta| \\
			\ge &
            0.2 \lambda_j\|\bu^{(p)}_{j,2}\|\times|(\bu_j^{(1,2)})^\top\hat{\bP}^{(1,2)}\btheta|
            -2\|\hat{\bP}^{(1,2)}-{\bP}^{(1,2)}\|
            \max_{i\ge j}
            \lambda_j\|\bu_{j,2}^{(p)}\|\\
			\ge &
			\lambda_j\|\bu^{(p)}_{j,2}\|
            \left(
            0.2\sqrt{\log d}
            -
            -\dfrac{Cd^{p/2}\sqrt{\log d}}{\lambda_r^2}\right)
			\ge  0.1\lambda_j\|\bu_{j,2}^{(p)}\|\sqrt{\log d}
			\ge \lambda_j\|\bu_{j,2}^{(p)}\|/5
		\end{align*}
		whenever $d\ge \exp(4)$. In the last line we have used the Gaussian gap result from Lemma~\ref{lem:sing-gap} and the projection estimation upper bound from Lemma~\ref{lem:hosvd}. Thus,
		\begin{equation}\label{eq:sing-gap}
			\omega_1^2-\omega_2^2\ge \lambda_j^2\|\bu_{j,2}^{(p)}\|^2/25
            \ge \lambda_j^2/400
            .
		\end{equation}
        where in the last step we used the bound 
        \[
        \min_{1\le j\le r}
        \|\bu_{j,2}^{(p)}\|\ge 1/4
        \]
        which holds with probability at least $1-d^{-2}$ due to an application of the scaled Chernoff bound on the random sampling mechanism. See e.g., Theorems 1 and 2 and related discussion in \cite{raghavan1988probabilistic}, see also an application of the above inequality used in the proof of Theorem 4.3 in \cite{auddy2021estimating}.
        
        We split the rest of the proof into two parts.
        \paragraph{Case 1: $p\ge 5$.} In this case $\tilde{\bV}_j$ is an orthogonal matrix and thus $\bU_j\Omega_j\tilde{\bV}_j$ forms an SVD of $\scrT_2\times_{1,2}\hat{\bP}^{(1,2)}\btheta$. Then by Davis-Kahan theorem,
		\begin{align*}
			\sin\angle\left(\hat{\bu}_j^{(2)},\bu_j^{(2)}\right)
			\le&
			\dfrac{
            \|\bE_{\theta}\bE_{\theta}^\top-d^{p-2}\|\bZ\|^2\II_d 
				+\Delta_j\Delta_j^\top		
				+\Delta_j\tilde{\bV}_j\Omega_j\tilde{\bU}_j^\top 
				+\tilde{\bU}_j\Omega_j\tilde{\bV}_j^\top\Delta_j^\top
				+\bE_\theta \tilde{\bV}_j\Omega_j\tilde{\bU}_j^\top 
				+\tilde{\bU}_j\Omega_j\tilde{\bV}_j^\top\bE_\theta^\top\|
			}{\omega_1^2-\omega_2^2}\\
			\le& \dfrac{Cd^{(p-2)/2}r\log d+\|\Delta_j\|^2+2\omega_j\|\Delta_j\|+C\omega_j\|\bE_\theta\tilde{\bV}_j\|}{\lambda_j^2/C}.
		\end{align*}	
        By its definition from \eqref{eq:def-delta-j}, $\|\bDelta_j\|\le C\sqrt{d}/\lambda_r\times \|\bZ\|\le C\sqrt{dr}/\lambda_r$, by successive triangle inequality on each mode estimate, along with Theorem~\ref{th:piter}. We will now use the estimates from Lemma \ref{lem:EVthetanorm-2}. Since $\|\bE_\theta\tilde{\bV}_j\|\le C\sqrt{dr}(\log d)$, we have
		\begin{align*}
			\sin\angle\left(\hat{\bu}_j^{(3)},\bu_j^{(3)}\right)
			\le & 
			\dfrac{C(d^{(p-2)/2}r(\log d)
				+\lambda_j\sqrt{dr}(\log d))
				}{\lambda_j^2}
		\end{align*}
        where, for the numerator we have used the estimates from \eqref{eq:Etheta-conc} and Lemma \ref{lem:EVthetanorm}, the assumption that $\|\Delta_j\|\le C\sqrt{d}$, and the fact that $\|\tilde{\bU}_j\|=1$. For the denominator we use \eqref{eq:sing-gap}.

        \paragraph{Case 2: $p=4$.} In this case $\tilde{\bV}_j$ has columns $\bu^{(4)}_{j,2}$ which are no longer orthogonal. Nonetheless, using matrix perturbation bounds such as Theorem 1 from \cite{cai2018rate} allows us to write:
        \begin{align*}
        \sin\angle\left(\hat{\bu}_j^{(3)},\bu_j^{(3)}\right)
			\le &~
			\dfrac{Cdr\log d+\|\Delta_j\|^2+2\omega_j\|\Delta_j\|+C\omega_j\|\bE_\theta\tilde{\bV}_j\|}{\lambda_j^2/C}
            +
            \frac{C
            \max_{i> j}
            \lambda_i
            |\langle
            \bu_{j,2}^{(4)},\bu_{i,2}^{(4)}
            \rangle|
            }{\lambda_j^2}\\
            \le&~
            \dfrac{C(dr\log d+\lambda_j\sqrt{dr}(\log d))}{\lambda_j^2/C}
            +
            \frac{1}{\sqrt{\log d}}    
        \end{align*}
        where in the last step we use the random sampling mechanism to note that $i\neq j$, 
        \[
        \langle
            \bu_{j,2}^{(4)},\bu_{i,2}^{(4)}
            \rangle
        =
        \sum_{k=1}^d
        (u^{(4)}_{jk})(u^{(4)}_{ik})B_k
        \]
        for iid $B_k\sim {\rm Bernoulli}(1/2)$, and 
        \(
        \EE(\langle
            \bu_{j,2}^{(4)},\bu_{i,2}^{(4)}
            \rangle)=1/2\cdot \langle
            \bu_{j}^{(4)},\bu_{i}^{(4)}
            \rangle=0
        \)
        so that an application of Hoeffding bound along with a union bound over $1\le j\le r$ yields
        \[
        \PP(\max_{1\le i,j\le r, i\neq j}
        |\langle
            \bu_{j,2}^{(4)},\bu_{i,2}^{(4)}
            \rangle|
        >\frac{1}{\sqrt{\log d}}    
        )
        \le \frac{C}{d^2}.
        \]
	\end{proof}

    \medskip

	\begin{proof}[Proof of Theorem \ref{th:init}]
		To simplify notation, we assume without loss of generality that $\pi=Id$. Fixing $q=2$, we also drop the subscript indicating the tensor mode. We will define:
        \[
        \bX={\sf Mat}_1(\scrX\times_{1}\hat{\bP}^{(1)}\btheta)
        =
        {\sf Mat}_1(\scrT\times_{1}\hat{\bP}^{(1)}\btheta)
        +
        {\sf Mat}_1(\scrT\times_{1}\hat{\bP}^{(1)}\btheta)
        \]
        We define $\bv_j:=(\otimes_{3\le q\le p}\bu_j^{(q)}$ for $j=1,\dots,r$. Then let us denote the singular values and vectors of $\scrT\times_1(\bP^{(1)}\btheta)$ by 
		$$
		\omega_j=\lambda_j|(\bu_j^{(1,2)})^\top\bP^{(1,2)}\btheta|,
		\,\,
		\Omega_j={\rm diag}(\omega_j,\dots,\omega_r)
		\,\,
		\tilde{\bU}_j=[\bu_j,\dots,\bu_r],
		\,\,
		\tilde{\bV}_j=[\bv_j,\dots,\bv_r].
		$$
        Since $L>Cr^2\log d$, it is clear that there exists $\btheta_*=\btheta_l$ such that the conclusion of Lemma \ref{lem:sing-gap} is satisfied for $\lambda_{\pi(j)}\ge \dots\ge \lambda_{\pi(r)}$. 
        
        Moreover, we define the matricization of the deflation errors in the previous steps by:
        \begin{equation}\label{eq:def-delta-j-1}
        \bDelta_j
        ={\sf Mat}_3(
        \sum_{i<j}\left(\lambda_i\bu^{(1)}_i\circ\dots\circ \bu^{(p)}_i
		-
        \langle\scrX, \hat{\bu}^{(1)}_i\circ\dots\circ \hat{\bu}^{(p)}_i\rangle
        \,\hat{\bu}^{(1)}_i\circ\dots\circ \hat{\bu}^{(p)}_i
		\right)
        )
        \end{equation}
        
		Also define $m_j:=(\omega_j^2+d^{p-2}\|\hat{\bP}^{(1)}\theta\|^2)$,  $\bZ:=\hat{\bP}^{(1)}\theta$ and $\bE_\theta:={\sf Mat}_1(\scrE\times_1\bZ)$. Note that $\bE_\theta\in \RR^{d\times d^{p-2}}$. 
		
		\noindent Since $\btheta$ is independent of $\hat{\bP}^{(1)}$, we have that $\btheta^\top\hat{\bP}^{(1)}\btheta\sim \chi^2({\rm trace}(\hat{\bP}^{(1)}))\equiv \chi^2_r$. By concentration inequalities for $\chi^2$ random variables (see e.g., \cite{vershynin2018high}), we have
		\begin{align*}
			\|\bZ\|^2=\btheta^\top\hat{\bP}^{(1)}\btheta\le r+2\sqrt{r\times r}+2r=4r
		\end{align*}
		with probability at least $d^2\exp(-r)$. Note that by Lemma \ref{lem:errnormbd2}, one has
		\begin{align*}\label{eq:uEEu}
			\|\bE_{\theta}\bE_{\theta}-d^{p-2}\|\bZ\|^2\II_d\|
			=&\sup_{\bw}\bw^\top\bE_\theta\bE_\theta^\top \bw-d^{p-2}\|\bZ\|^2
			=\sup_{\bw}\sum_{i_2\dots i_{p-1}}(\langle\bw\otimes \bZ, \bE_{(i_2\dots i_{p-1})}\rangle)^2-d^{p-2}\|\bZ\|^2\\
			=&\sup_{\bw}\sum_{s=1}^{d^{p-2}}
			\left(\bE_{(s)}\circ\bE_{(s)}\right)\times_1\bw\times_2\bZ\times_3\bw\times_4\bZ-d^{p-2}\|\bZ\|^2\\
			\le &d^{p-2}\left\{\norm*{\dfrac{1}{d^{p-2}}\sum_{s}\bE_{(s)}\circ\bE_{(s)}-\scrG}\bZ\|^2+\sup_{\bw}\scrG\times_1\bw\times_2\bZ\times_3\bw\times_4\bZ
			\right\}-d^{p-2}\|\bZ\|^2\\
			\le &d^{p-2}\sqrt{\dfrac{Cd(\log d)^2}{d^{p-2}}}\cdot\|\bZ\|^2
			\le Cd^{(p-1)/2}r\log d\numberthis
		\end{align*}
		Finally, we bound the gap between top two singular values. Note that
		\begin{align*}
			\omega_1-\omega_2
			\ge &\lambda_j|(\bu_j^{(1)})^\top\hat{\bP}^{(1)}\theta|
			-\lambda_{j+1}\max_{i\ge j}|(\bu_i^{(1)})^\top\hat{\bP}^{(1)}\theta|\\
			\ge &
			\lambda_j\left(|(\bu_j^{(1)})^\top\hat{\bP}^{(1)}\theta|
			-\max_{i> j}|(\bu_i^{(1)})^\top\hat{\bP}^{(1)}\theta|
			\right)\\
			\ge &
			\lambda_j\left(|(\bu_j^{(1)})^\top\bP^{(1)}\theta|
			-\max_{i> j}|(\bu_i^{(1)})^\top\bP^{(1)}\theta|
			-2\max_{i\ge j}|(\bu_i^{(1)})^\top(\bP^{(1)}-\bP^{(1)})\theta|
			\right)\\
			\ge &
			\lambda_j\left(|Z_{jl}|-\max_{k>j}|Z_{kl}|-C\|\hat{\bP}^{(1)}-\bP^{(1)}\|\sqrt{\log d}\right)\\
			\ge &\lambda_j\left(0.2\sqrt{\log d}-\dfrac{Cd^{p/2}\sqrt{\log d}}{\lambda_r^2}\right)
			\ge  0.1\lambda_j\sqrt{\log d}
			\ge \lambda_j/5
		\end{align*}
		whenever $d\ge \exp(4)$. In the last line we have used the Gaussian gap result from \ref{lem:sing-gap} and the projection estimation upper bound from Lemma~\ref{lem:hosvd}. Thus,
		\begin{equation}\label{eq:sing-gap-1}
			\omega_1^2-\omega_2^2\ge \lambda_j^2/25.
		\end{equation}
		By Davis-Kahan theorem,
		\begin{align*}
			\sin\angle\left(\hat{\bu}_j^{(2)},\bu_j^{(2)}\right)
			\le&
			\dfrac{
            \|\bE_{\theta}\bE_{\theta}^\top-d^{p-2}\|\bZ\|^2\II_d 
				+\Delta_j\Delta_j^\top		
				+\Delta_j\tilde{\bV}_j\Omega_j\tilde{\bU}_j^\top 
				+\tilde{\bU}_j\Omega_j\tilde{\bV}_j^\top\Delta_j^\top
				+\bE_\theta \tilde{\bV}_j\Omega_j\tilde{\bU}_j^\top 
				+\tilde{\bU}_j\Omega_j\tilde{\bV}_j^\top\bE_\theta^\top\|
			}{\omega_1^2-\omega_2^2}\\
			\le& \dfrac{Cd^{(p-1)/2}r\log d+\|\Delta_j\|^2+2\omega_j\|\Delta_j\|+C\omega_j\|\bE_\theta\tilde{\bV}_j\|}{\lambda_j^2/C}.
		\end{align*}	
        By its definition from \eqref{eq:def-delta-j-1}, $\|\bDelta_j\|\le C\sqrt{d}/\lambda_r$, by successive trinagle inequality on each mode estimate, along with Theorem~\ref{th:piter}. We will now use the estimates from Lemma \ref{lem:EVthetanorm} and divide the rest of the proof into two cases.
		
		\paragraph{Case 1: $p=3$.} Since $\|\bE_\theta\tilde{\bV}_j\|\le C\sqrt{dr}(\log d)$, we have
		$$
		\sin\angle\left(\hat{\bu}_j^{(2)},\bu_j^{(2)}\right)
		\le \dfrac{C\left(
			dr(\log d)+\lambda_j\sqrt{dr}(\log d)^{3/2}
			\right)}{\lambda_j^2}
		$$
		where, for the numerator we have used the estimates from \eqref{eq:uEEu} and Lemma \ref{lem:EVthetanorm}, the assumption that $\|\Delta_j\|\le C\sqrt{d}$, and the fact that $\|\tilde{\bU}_j\|=1$. For the denominator we use \eqref{eq:sing-gap-1}.
		
		\paragraph{Case 2: $p\ge 4$.} Similarly, in this case, we get			
		\begin{align*}
			\sin\angle\left(\hat{\bu}_j^{(2)},\bu_j^{(2)}\right)
			\le &
			\dfrac{C\left(d^{(p-1)/2}r\log d+\lambda_j\sqrt{\log d}
				\cdot\dfrac{d^{(p+1)/2}r(\log d)}{\lambda_j^2}
				\right)}{\lambda_j^2}\\
			\le &
			C\cdot
			\left(
			\dfrac{d^{(p-1)/2}r(\log d)+\lambda_j\sqrt{dr}(\log d)^{3/2}}{\lambda_j^2}
			+\dfrac{d^{(p+1)/2}r(\log d)^{3/2}}{\lambda_j^3}
			\right) \\
			\le & 
			\dfrac{C(d^{(p-1)/2}r(\log d)
				+\lambda_j\sqrt{dr}(\log d)^{3/2})
				}{\lambda_j^2}
		\end{align*}
		using the same estimates as above. In the last line we have used the assumption that
		$
		\lambda_j\ge Cd^{p/4}\sqrt{\log d}\ge Cd\sqrt{\log d}
		$
		since $p\ge 4$.
	\end{proof}

    \medskip
    
\begin{proof}[Proof of Theorem \ref{th:asy-dist-rand} Part 1]
	We will use the expansion in Lemma \ref{lem:piter-so-err} with the randomly distributed errors $\scrE_{i_1\dots i_p}$ are independent copies of a random variable $E$ satisfying $\EE(E)=0$, $\EE(E)^2=1$ and $\EE(E)^8\le C$ for a finite $C>0$.
	
	We first bound the norm of the remainder term $\bE^{\top}\bU_{\perp}(\tilde{\bLambda}+\bU_{\perp}^{\top}\Delta_j)$. Notice that
\begin{align*}
	\bE^{\top}\bU_{\perp}(\tilde{\bLambda}+\bU_{\perp}^{\top}\Delta_j)
	\le& \|\bE\|\|\tilde{\bLambda}\|+\|\bE\|\|\Delta_j\|
	\le C\left(\eps_1+\dfrac{\eps_1^2}{\lambda_j}\right)
	\cdot \left(\dfrac{\eps_0\eps_1}{\lambda_j}+\dfrac{\eps_1^3}{\lambda_j^2}\right) 
\end{align*}
plugging in the bounds from equations \eqref{eq:til-lambd-norm},  \eqref{eq:asy-del-bd}, and \eqref{eq:asy-err-bd}. This implies that
\begin{align*}\label{eq:asy-uusq-rand}
	\abs*{
	\langle\bu^{(1)}_j,\hat{\bu}^{(1)}\rangle^2-1
	+\dfrac{\bE^{\top}\bU_{\perp}\bU_{\perp}^{\top}\bE
	}{\lambda_j^2}}
	\le 
	C\left(\dfrac{d^{5/4}}{\lambda_j^3}+\dfrac{d^2}{\lambda_j^4}\right).
	\numberthis
\end{align*}	
Recall that $\bE=\scrE\times_{k\neq q}\check{\bu}^{(k)}_j$. Let us define $\bE_*=\scrE\times_{k\neq q}\bu^{(k)}_j$. We first bound $\bE^{\top}\bE-\bE_*^{\top}\bE_*$.

By the steps preceding \eqref{eq:piter-errbd}, we then have, for a sign $\gamma\in\{+1,-1\}$, that
\begin{align*}\label{eq:bd-err0}
	\bE=&
	\bE_* 
	+\scrE\times_2(\hat{\bu}^{(2)}_j-\bu^{(2)}_j)\times_3\bu_j^{(3)}\times_{4\le q\le p}\bu^{(q)}_j
	+\scrE\times_2\bu^{(2)}_j\times_3(\hat{\bu}_j^{(3)}-\bu_j^{(3)})\times_{4\le q\le p}\bu^{(q)}_j\\
	&+\scrE\times_2(\hat{\bu}^{(2)}_j-\bu^{(2)}_j)\times_3(\hat{\bu}_j^{(3)}-\bu_j^{(3)})\times_{4\le q\le p}\bu^{(q)}_j\\
	&+\scrE\times2(\hat{\bu}^{(2)}_j-\bu^{(2)}_j)\times_3(\hat{\bu}_j^{(3)}-\bu_j^{(3)})\times_4(\hat{\bu}_j^{(4)}-\bu^{(4)}_j)\times_{5\le q\le p}\hat{\bu}^{(q)}_j\numberthis
\end{align*}
Define $\bM=\scrE\times_{3\le q\le p }\bu_j^{(q)}\in \RR^{d\times d}$. Note that the entries of $\bM$ are independent copies of a random variable $E$ with mean zero, variance one and $\EE(E)^8\le C$. Then, the error due to the first term above becomes
\begin{align*}\label{eq:bd-cterm1}
\bE^{\top}_*\scrE\times_2(\hat{\bu}^{(2)}_j-\bu^{(2)}_j)\times_3\bu_j^{(3)}\times_{4\le q\le p}\bu^{(q)}_j	
=&\bE^{\top}_*\bM (\hat{\bu}^{(2)}_j-\bu^{(2)}_j)\\
=&(\bu^{(2)}_j)^{\top}\bM^{\top}\bM 
\left(
\dfrac{1}{\lambda_j}\cdot \bM^{\top}\bu^{(1)}_j
+O\left(\dfrac{\eps_1\eps_2}{\lambda_j^2}\right)
\right).\numberthis
\end{align*}
Since $\{\bu^{(2)}_k:1\le k\le d\}$ form an orthogonal basis of $\RR^d$, the first term in the rightmost hand side of \eqref{eq:bd-cterm1} becomes
\begin{align*}
	(\bu^{(2)}_j)^{\top}\bM^{\top}\bM\bM^{\top}\bu_j^{(1)}
	=&\bu_j^{(2)}\bM^{\top}\bM
	\sum_{k=1}^d\bu^{(2)}_k\left(\bu^{(2)}_k\right)^{\top}
	\bM^{\top}\bu_j^{(1)}\\
	=&\|\bM\bu_j^{(2)}\|^2\left(\bu_j^{(2)}\right)^{\top}\bM\bu_j^{(1)}
	+\sum_{k\neq j}(\bu^{(2)}_j)^{\top}\bM\bM^{\top}\bu_k^{(2)}(\bu_k^{(2)})^{\top}\bM\bu_j^{(1)}\\
	\le & \,
	\eps_1^2\eps_0+\sum_{k\neq j}(\bu^{(2)}_j)^{\top}\bM\bM^{\top}\bu_k^{(2)}(\bu_k^{(2)})^{\top}\bM\bu_j^{(1)}.
\end{align*}
For the second term on the right hand side, note that the rows of $\bM$ are independent copies of a random $d$-dimensional vector with mean zero and variance $\II_d$. Then it can be checked through some straightforward but tedious calculation that for $k_1\neq k_2\neq j$,
\begin{align*}
	{\rm Cov}(
	(\bu^{(2)}_j)^{\top}\bM\bM^{\top}\bu_{k_1}^{(2)}(\bu_{k_1}^{(2)})^{\top}\bM\bu_j^{(1)},
	(\bu^{(2)}_j)^{\top}\bM\bM^{\top}\bu_{k_2}^{(2)}(\bu_{k_2}^{(2)})^{\top}\bM\bu_j^{(1)}
	=\sum_{l=1}^d(\bu^{(1)}_{jl})^2=1
\end{align*}
and 
\begin{align*}
	{\rm Var}((\bu^{(2)}_j)^{\top}\bM\bM^{\top}\bu_{k}^{(2)}(\bu_{k}^{(2)})^{\top}\bM\bu_j^{(1)})
	\le Cd.
\end{align*}
Thus, by Chebychev's inequality, we have that
\begin{align*}
	\EE \abs*{\sum_{k\neq j}(\bu^{(2)}_j)^{\top}\bM\bM^{\top}\bu_k^{(2)}(\bu_k^{(2)})^{\top}\bM\bu_j^{(1)}}
	\le Cd+\int\limits_{Cd}^{\infty}\dfrac{Cd^2}{t^2}dt\le Cd.
\end{align*}
Plugging in the bounds from Lemma \ref{lem:eps012-bds}, we then have, from \eqref{eq:bd-err0} and \eqref{eq:bd-cterm1} that
\begin{align*}
	\EE\abs*{\bE^{\top}\bE-\bE_{\ast}^{\top}\bE_{\ast}}
	\le& \dfrac{\EE\eps_1^2\eps_0+Cd^{5/4}}{\lambda_j}
	+\dfrac{1}{\lambda_j^2}\cdot\EE\eps_1^2\|\scrE\times_{4\le q\le p}\hat{\bu}_j^{(q)}\|
	\\
	\le& \dfrac{Cd^{5/4}}{\lambda_j}+\dfrac{Cd}{\lambda_j^2}
	\cdot 
	\sum_{q=3}^pCd^{\tfrac{q+1}{8}}(\log d)^{3/2}\cdot \left(\dfrac{\sqrt{d}}{\lambda_j}\right)^{q-3}	\\
	\le& \dfrac{Cd^{5/4}}{\lambda_j}
\end{align*}
where in the last lines we use the fact that $\EE\|\scrE\times_{k\le q\le p}\bu_j^{(q)}\|\le Cd^{(k+1)/8}$ since $\alpha\ge 8$ (see Theorem 2.1 of \cite{auddy2021estimating}), and the  assumption that $\lambda_j\ge Cd^{3/4}$. It can be checked similarly that
\begin{equation}\label{eq:asy-uusq-rand-2}
	\EE\abs*{
	(\bE^{\top}\bu^{(1)}_j)^2-(\bE^{\top}_*\bu^{(1)}_j)^2	
	}\le \dfrac{Cd^{3/4}}{\lambda_j}.
\end{equation}
Thus we have from \eqref{eq:asy-uusq-rand} that
\begin{align*}
	\EE\abs*{
		\langle\bu^{(1)}_j,\hat{\bu}^{(1)}\rangle^2-1
		+\dfrac{\bE_*^{\top}\bU_{\perp}\bU_{\perp}^{\top}\bE_*
		}{\lambda_j^2}}
	\le 
	\dfrac{Cd^{5/4}}{\lambda_j^3}.
\end{align*}

Notice finally that $\EE\bE^{\top}_*\bU_{\perp}\bU_{\perp}^{\top}\bE_*={\rm tr}(\bU_{\perp}\bU^{\top}_{\perp})=d-1$. Moreover,
\begin{align*}
	\dfrac{1}{\sqrt{d}}\left(\bE_*^{\top}\bU_{\perp}\bU_{\perp}^{\top}\bE_*-(d-1)\right)
	=&\dfrac{1}{\sqrt{d}}\sum_{i=1}^{d}(E_i^2-1)
	-\dfrac{1}{\sqrt{d}}\{((\bu_j^{(1)})^{\top}\bE_*)^2-1\}.
\end{align*}
Since $\scrE$ has independent elements, the entries of $\bE_*=\scrE\times_{q>1}\bu^{(q)}_1\in \RR^d$  are independent random variables with mean zero and variance one. Moreover, $(\bu^{(1)}_j)^{\top}\bE$ also has mean zero and variance one. Thus,
\begin{align*}
\calW_1\left(\dfrac{1}{\sqrt{d}}\left(\bE_*^{\top}\bU_{\perp}\bU_{\perp}^{\top}\bE_*-(d-1)\right),
	\sigma_j Z\right)
\le 
\dfrac{C}{\sqrt{d}}
\end{align*}
by Berry Esseen theorem, since $E_i^2-1$ are iid random variables with mean zero and variance $\sigma_j^2$. Here $\bZ\sim N(0,1)$ and $\sigma_j^2$ is as defined in the statement of Theorem 3.1. Finally using \eqref{eq:asy-uusq-rand-2}, we also have that 
\begin{align*}
	\calW_1\left(
	\dfrac{\lambda_j^2}{\sqrt{d}}
	\left\{\langle\bu^{(1)}_j,\hat{\bu}^{(1)}\rangle^2-\left(1-\dfrac{d-1}{\lambda_j^2}\right)\right\},\,
	\sigma_jZ 
	\right)\le \dfrac{Cd^{3/4}}{\lambda_j}.
\end{align*}
\end{proof}

\begin{proof}[Proof of Theorem \ref{th:asy-dist-rand} part 2.]
	Next, we consider a unit vector $\ba\perp \bu^{(1)}_j$. Let $\gamma_j^{(1)}={\rm sign}(\langle\bu^{(1)}_j,\bu^{(1)}_j\rangle)$. Once again, using the resolvent based series expansion from \cite{xia2021normal}, we can write
	\begin{equation}\label{eq:a-linfrm}
		\langle\ba,\,\hat{\bu}_j^{(1)}\rangle \langle\hat{\bu}^{(1)}_j,\bu^{(1)}_j\rangle
		=\gamma_j^{(1)}\langle\ba,\bu^{(1)}_j\rangle
		+\dfrac{1}{\hat{\lambda}^2_j}\ba^{\top}\bG\bu_j^{(1)}
		+\dfrac{\ba^{\top}\bG\bU_{\perp}\bU_{\perp}^{\top}\bG\bu_j^{(1)}
		-\ba^{\top}\bG\bu_{j}^{(1)}(\bu_j^{(1)})^{\top}\bG\bu_j^{(1)}	
	}{\hat{\lambda}_j^4}
		+\calR_3.
	\end{equation}
	We have, from \eqref{eq:defGrand} and \eqref{eq:defYrand} that
	\begin{align*}\label{eq:aG-exp}
		\ba^{\top}\bG\bu_j^{(1)}
		=&(\bY^{\top}\bu_j^{(1)}+\hat{\lambda}_j)\ba^{\top}\bY
		=(\bDelta_j^{\top}\bu_j^{(1)}+\bE^{\top}\bu_j^{(1)}
		+\hat{\lambda}_j)\ba^{\top}\bY\\
		=&\lambda_j\ba^{\top}\bE
		+\lambda_j\ba^{\top}(\bU_{\perp}\tilde{\bLambda}+\bDelta_j)
		+(\bDelta_j^{\top}\bu_j^{(1)}+\bE^{\top}\bu_j^{(1)}
		+\hat{\lambda}_j-\lambda_j)\ba^{\top}\bY\\
		=&\lambda_j\ba^{\top}\bE_*
		+\lambda_j\ba^{\top}(\bE_*-\bE)
		+\lambda_j\ba^{\top}(\bU_{\perp}\tilde{\bLambda}+\bDelta_j)
		+(\bDelta_j^{\top}\bu_j^{(1)}+\bE^{\top}\bu_j^{(1)}
		+\hat{\lambda}_j-\lambda_j)\ba^{\top}\bY.\numberthis
	\end{align*}
	We now expand each of these terms using \eqref{eq:defYrand}. First note that, 
	\begin{align*}
		X=\scrE\times_{1}\ba\times_{2}(\hat{\bu}^{(2)}_j-\bu^{(2)}_j)\times_{k\ge 3}\bu^{(k)}_j
		=&\dfrac{1}{\lambda_j}\cdot \ba^{\top}\bM\bM^{\top}\bu^{(1)}_j+\calR_*
	\end{align*}	
	where $|\calR_*|\le \dfrac{C\eps_1^3}{\lambda_j^2}\le \dfrac{Cd^{3/2}}{\lambda_j^2}$. Here $\bM=\scrE\times_{k\ge 3}\bu^{(k)}_j$ is a $d\times d$ matrix whose entries are independent copies of a random variable $E$ with $\EE E=0$, ${\rm Var}(E)=1$ and $\EE E^8\le C$. Then it is not hard to check that
	$$
	\EE(\ba^{\top}\bM\bM^{\top}\bu^{(1)}_j)=d\langle\ba,\bu_j^{(1)}\rangle=0
	$$
	and
	$$
	{\rm EE}(\ba^{\top}\bM\bM^{\top}\bu^{(1)}_j)^4\le Cd.
	$$
	Thus, $|\ba^{\top}\bM\bM^{\top}\bu^{(1)}_j|\le Cd^{3/4}$ with probability at least $1-d^{-2}$. Thus by \eqref{eq:bd-err0}, we have
	\begin{equation}\label{eq:alinf-bd1}
		\abs*{\ba^{\top}\bE-\ba^{\top}\bE_*}
		\le \dfrac{Cd^{3/4}}{\lambda_j}
		+\dfrac{C\eps_1^2}{\lambda_j^2}\sum_{q=3}^pCd^{\tfrac{q+1}{8}}(\log d)^{3/2}\cdot\left(\dfrac{\sqrt{d}}{\lambda_j}\right)^{q-3}
		\le \dfrac{Cd^{3/4}}{\lambda_j}.
	\end{equation}
In the above we have used tensor norm bounds from Theorem 2.1 of \cite{auddy2021estimating}, noting that $\scrE$ has independent entries with finite $8^\text{th}$ moments.
	
	Next, using \eqref{eq:lamb-hat}, \eqref{eq:til-lambd-norm}, \eqref{eq:asy-del-bd}, \eqref{eq:ujEhat}, and \eqref{eq:asy-Y-bd}, we can bound the third term above using
	\begin{align*}\label{eq:alinf-bd2}
		\lambda_j\ba^{\top}(\bU_{\perp}\tilde{\bLambda}+\bDelta_j)
		\le \lambda_j(\|\tilde{\bLambda}\|+\|\bDelta_j\|)
		\le C\eps_0\eps_1+\dfrac{C\eps_1^3}{\lambda_j}
		\le Cd^{3/4}\numberthis
	\end{align*}
	since $\eps_0\le Cd^{1/4}$, $\eps_1\le C\sqrt{d}$ with probability at least $1-d^{-1}$, by Lemma \ref{lem:eps012-bds}, and by assumption $\lambda_j\ge Cd^{3/4}$. 
	Similarly the third term can be written as
	\begin{align*}\label{eq:alinf-bd3}
		&|(\bDelta_j^{\top}\bu_j^{(1)}+\bE^{\top}\bu_j^{(1)}
		+\hat{\lambda}_j-\lambda_j)\ba^{\top}\bY|\\
		\le&
		\left(
		\|\bDelta_j\|+\|\bE-\bE_*\|+\bE_*^{\top}\bu_j^{(1)}
		+|\hat{\lambda}_j-\lambda_j|
		\right)
		(\|\tilde{\Lambda}\|+\|\bDelta\|+\|\bE-\bE_*\|+\bE_*^{\top}\ba)
		\\
		\le&
		\left(
		\dfrac{C\eps_0\eps_1}{\lambda_j}+\dfrac{\eps_1^2}{\lambda_j^2}+\eps_0
		\right)\cdot 
		\left(
		\dfrac{C\eps_0\eps_1}{\lambda_j}+\dfrac{C\eps_1^3}{\lambda_j^2}
		+\dfrac{C\eps_1^2}{\lambda_j}
		+\dfrac{C\eps_1^2}{\lambda_j^2}\sum_{q=3}^pCd^{\tfrac{q+1}{8}}(\log d)^{3/2}\cdot\left(\dfrac{\sqrt{d}}{\lambda_j}\right)^{q-3}+\eps_0
		\right)
		\\
		 \le& 
		 Cd^{1/4}\cdot\left(\dfrac{Cd}{\lambda_j}+Cd^{1/4}\right)
		 \\
		\le& C\sqrt{d}.\numberthis
	\end{align*}
	Once again we have used the error bounds from \eqref{eq:asy-del-bd}, \eqref{eq:til-lambd-norm}, \eqref{eq:bd-err0} along with the high probability bounds on $\eps_0$, $\eps_1$, $\eps_2$ from Lemma \ref{lem:eps012-bds}. We also use the upper bounds for tensor norms provided by Theorem 2.1 of \cite{auddy2021estimating}, noting that $\scrE$ has independent elements with finite 8th moments. Plugging in the bounds from equations \eqref{eq:alinf-bd2}-\eqref{eq:alinf-bd3} into \eqref{eq:alinf-bd1}, we have
	\begin{align*}\label{eq:alinf-bd4}
		\abs*{
	\ba^{\top}\bG\bu_j^{(1)}-\lambda_j\ba^{\top}\bE_*	
		}\le Cd^{3/4}
\numberthis	
	\end{align*}
with probability at least $1-d^{-1}$. It remains to bound the third and fourth terms in \eqref{eq:a-linfrm}. To that end, note that $\ba\perp\bu_j^{(1)}$ and hence by equations \eqref{eq:defGrand} and \eqref{eq:defYrand}, we have 
\begin{align*}\label{eq:alinf-bd5}
	\ba^{\top}\bG\bU_{\perp}=(\ba^{\top}\bY)^2
	\le & 2(\ba^{\top}\bE_*)^2	
	+2(\|\tilde{\Lambda}\|+\|\bDelta\|+\ba^{\top}(\bE-\bE_*))^2\\
	\le & Cd^{1/4}+\dfrac{C\eps_0^2\eps_1^2}{\lambda_j^2}+\dfrac{C\eps_1^6}{\lambda_j^4}+\dfrac{Cd^{3/2}}{\lambda_j^2}\\
	\le & Cd^{1/4}+ \dfrac{Cd^{3/2}}{\lambda_j^2}+\dfrac{Cd^3}{\lambda_j^4}\numberthis
\end{align*}
where we use the bounds from \eqref{eq:til-lambd-norm}, \eqref{eq:asy-del-bd}, and \eqref{eq:alinf-bd2}. Moreover, $\ba^{\top}\bE$ has mean zero, variance one and $\EE(\ba^{\top}\bE_*)^8\le C$, so that
$
\PP(|\ba^{\top}\bE_*|\le Cd^{1/4})\le d^{-2}.
$
Similarly, by \eqref{eq:alinf-bd1}-\eqref{eq:alinf-bd3},
\begin{align*}
	|\ba^{\top}\bG\bu_j^{(1)}|\le \lambda_j\ba^{\top}\bE_*
	+Cd^{3/4}\le C\lambda_1d^{1/4}+Cd^{3/4}\le C\lambda_jd^{1/4}.
\end{align*}
Thus,
\begin{align*}\label{eq:alinf-bd6}
	|\ba^{\top}\bG\bU_{\perp}\bU_{\perp}^{\top}\bG\bu_j^{(1)}
	-\ba^{\top}\bG\bu_{j}^{(1)}(\bu_j^{(1)})^{\top}\bG\bu_j^{(1)}|
	\le& |\ba^{\top}\bG\bU_{\perp}||\bU_{\perp}^{\top}\bG\bu_j^{(1)}|
	+|\ba^{\top}\bG\bu_j^{(1)}||(\bu_j^{(1)})^{\top}\bG\bu_j^{(1)}|\\
	\le & Cd^{1/4}\cdot (\lambda_j\|\bE\|+Cd)+Cd^{1/4}\cdot (Cd+C\lambda_jd^{1/4})\\
	\le & C\lambda_jd^{3/4}\numberthis
\end{align*}
by the bounds from \eqref{eq:asy-fo-term} and \eqref{eq:asy-fo-term-2}. Finally note that we have, by \eqref{eq:lamb-hat}, that
$
|\hat{\lambda}_j-\lambda_j|\le \dfrac{C\eps_1^2}{\lambda_j}\le \dfrac{Cd}{\lambda_j}
$
and by part 1 of from Theorem \ref{th:asy-dist-rand}, that 
$
|\hat{\bu}_j^{(1)},\bu_j^{(1)}|\ge
|\hat{\bu}_j^{(1)},\bu_j^{(1)}|^2
\ge  1-\dfrac{d-1}{\lambda_j^2}-\dfrac{Cd^{3/2}}{\lambda_j^3}
$. Moreover, the remainder term $\calR_3$ satisfies
\begin{equation}\label{eq:alinf-bd7}
|\calR_3|\le \dfrac{\|\bG\|^3}{\hat{\lambda}_j^6}\le\dfrac{Cd^{3/2}}{\lambda_j^3}. 
\end{equation}
Plugging in the bounds from equations \eqref{eq:alinf-bd1}, \eqref{eq:alinf-bd5}, \eqref{eq:alinf-bd6} and \eqref{eq:alinf-bd7} into \eqref{eq:a-linfrm}, we then have that
\begin{align*}
	\abs*{
	\langle\ba,\hat{\bu}_j^{(1)}\rangle
	-\dfrac{\gamma_j^{(1)}\langle\ba,\bu_j^{(1)}\rangle}{1-(d-1)\lambda_j^{-2}}
	-\dfrac{\ba^{\top}\bE_*}{\lambda_j}
	}\le \dfrac{Cd^{3/4}}{\lambda_j^2}+\dfrac{Cd}{\lambda_j^3}.
\end{align*}
Since $\bE_*=\scrE\times_{k\ge 2}\bu^{(k)}_j$, this finishes the proof.
\end{proof}

\subsection{Proofs of Lemmas}

	\begin{proof}[Proof of Lemma~\ref{lem:esti-op-norm}]
		We fix $q=1$. Then the columns of $(\check{\bU}^{(1)}_{(i)}-\bU^{(1)}_{(i)})$ are
		\begin{align*}
			(\check{\bU}^{(1)}_{(i)}-\bU^{(1)}_{(i)})\be_j
			=\dfrac{1}{\lambda_j}\scrE\times_{k\neq 1}\bu_j^{(k)}
			+R_j.
		\end{align*}
		for $j\le i$. Following the proof steps of Theorem \ref{th:piter}, we can write the following expansion for $R_j$. We have
		\begin{align*}
			R_j=&
			\dfrac{
				\scrE\times_{2}(\hat{\bu}_j^{(2)}-\gamma_j^{(2)}\bu_j^{(2)})\times_{k\ge 3}\bu_j^{(k)}
				+\displaystyle\sum_{l>j}\lambda_l\prod_{k=2}^p\langle\hat{\bu}^{(k)}_j,\bu_l^{(k)}\rangle	
				+\hat{\scrT}_{-j}\times_{k\ge 2}\hat{\bu}^{(k)}_j
			}{\|\scrX\times_{k\neq 1}\hat{\bu}_j^{(k)}\|}.
		\end{align*}
		We first bound the denominator and the norms of the second and third terms of the numerator.	Note that
		\begin{align*}\label{eq:denom}
			\|\scrX\times_{k\neq 1}\hat{\bu}_j^{(k)}\|
			\ge& \lambda_j(1-L_T^2)-\scrE\times_{q\ge 1}\bu_j^{(q)}-CL_T\eps_2 
			-CL_T^2\|\scrE\|-C\eps_1^2/\lambda_{j-1}\\
			\ge& \lambda_j\left(1
			-\dfrac{C\eps_1}{\lambda_j}
			-\dfrac{C\eps_1\eps_2}{\lambda_j^2}\right)\numberthis
		\end{align*}
		assuming $\lambda_j\ge C\|\scrE\|$. Next,
		\begin{align*}\label{eq:num2}
			\norm*{\displaystyle\sum_{l>j}\lambda_l\prod_{k=2}^p\langle\hat{\bu}^{(k)}_j,\bu_l^{(k)}\rangle}
			\le &
			\lambda_j|\langle\hat{\bu}^{(2)}_j,\bu_i^{(2)}\rangle|
			\sin\angle\left(\hat{\bu}^{(3)}_j,\bu^{(3)}_j\right)\\
			\le & 
			\lambda_j \abs*{
				\dfrac{\scrE\times_{2}\bu^{(2)}_i\times_{k\neq 2}\bu^{(k)}_j}{\lambda_j}
				+\dfrac{C\eps_1^2}{\lambda_j^2}
			}\cdot \dfrac{C\eps_1}{\lambda_j}\\
			\le& \dfrac{C\eps_0\eps_1}{\lambda_j}+\dfrac{C\eps_1^3}{\lambda_j^2},\numberthis
		\end{align*}
		where we use \eqref{eq:secopert} in the second step. Finally we use Lemma \ref{lem:defl-er}, plugging in $L_T=\dfrac{C\eps_1}{\lambda_j}$, to conclude that
		\begin{equation}\label{eq:num3}
			\norm*{\hat{\scrT}_{-j}\times_{k\ge 2}\hat{\bu}^{(k)}_j}
			\le \dfrac{C\eps_0\eps_1}{\lambda_j}+\dfrac{C\eps_1^3}{\lambda_j^2}.
		\end{equation}
		By the bounds from \eqref{eq:denom}, \eqref{eq:num2} and \eqref{eq:num3}, we then have, 
		\begin{align*}
			\norm*{[R_1\, R_2\,\dots R_i]}
			\le &\norm*{
				\left[\dfrac{C\scrE\times_{2}(\hat{\bu}_1^{(2)}-\gamma_1^{(2)}\bu_1^{(2)})\times_{k\ge 3}\bu_1^{(k)}}{\lambda_1}
				\,\dots\, 
				\dfrac{C\scrE\times_{2}(\hat{\bu}_i^{(2)}-\gamma_i^{(2)}\bu_i^{(2)})\times_{k\ge 3}\bu_i^{(k)}}{\lambda_i}
				\right]}\\
			&+\dfrac{1}{\lambda_i}\cdot
			\dfrac{C\eps_0\eps_1\sqrt{r}}{\lambda_i}\\
			\le&
			\dfrac{C\eps_1}{\lambda_i}\cdot\max_{j\le i}\dfrac{\norm*{\scrE\times_{k\neq 2}\bu^{(k)}_j}}{\lambda_j}
			+\dfrac{C\eps_1}{\lambda_i}\cdot\dfrac{C\eps_1^2}{\lambda_j^2}\times\sqrt{r}
			+\dfrac{C\eps_1}{\lambda_i}\\
			\le& \dfrac{C\eps_1}{\lambda_i} 
		\end{align*}
		where we use Assumption A3 and \eqref{eq:secopert} on the second last line, along with the fact that $\lambda_j\ge C\max\{\eps_0\eps_1,\eps_1r^{1/4}\}$. Finally it follows again by assumption A3 that
		\begin{align*}
			\norm*{
				\left[\dfrac{\scrE\times_{k\neq 1}\bu_1^{(k)}}{\lambda_1}
				\,\dots\,
				\dfrac{\scrE\times_{k\neq 1}\bu_r^{(i)}}{\lambda_i}
				\right]	
			}\le \dfrac{C\eps_1}{\lambda_i}.
		\end{align*}
		This finishes the proof.	
	\end{proof}
	
\medskip

	\begin{proof}[Proof of Lemma \ref{lem:defl-er}]
		We now bound the errors arising from the  previous steps $i<j$.
		\begin{align*}
			&~\hat{\scrT}_{-j}\times_1\bv\times_{q>1}\hat{\bu}^{(q)}_{[t]}\\
			=&\sum_{i<j}(\lambda_i-\hat{\lambda}_i)
			\langle\bu^{(1)}_i,\bv\rangle
			\prod_{q=2}^p
			\langle\bu_i^{(q)},\hat{\bu}^{(q)}_{[t]}\rangle\\
			&+\sum_{i<j}
			\hat{\lambda}_i
			\left(
			\langle\bu_i^{(1)}\circ \bu_i^{(2)}\circ\dots \circ \bu_i^{(p)}
			-\hat{\bu}_i^{(1)}\circ \hat{\bu}_i^{(2)}\circ\dots \circ \hat{\bu}_i^{(p)},
			\,
			\bv\circ \hat{\bu}_{j,[t]}^{(2)}\circ\dots\circ \hat{\bu}_{j,[t]}^{(p)}
			\rangle
			\right)\\
			=:&T_0+\sum_{q=1}^pT_q
		\end{align*}
		for terms $T_0,\dots\,,T_p$ defined below. First,
		\begin{align*}
			|T_0|=&\abs*{\sum_{i<j}(\lambda_i-\hat{\lambda}_i)
				\langle\bu^{(1)}_i,\bv\rangle
				\prod_{q=2}^p
				\langle\bu_i^{(q)},\hat{\bu}^{(q)}_{[t]}\rangle}
			=\abs*{\sum_{i<j}(\lambda_i-\hat{\lambda}_i)
				\langle\bu^{(1)}_i,\bv\rangle
				\prod_{q=2}^p
				\langle\bu_i^{(q)},\hat{\bu}^{(q)}_{[t]}-\bu^{(q)}_j\rangle}\\
			\le& \max_{1\le i< j}|\lambda_i-\hat{\lambda}_i|\prod_{q=2}^p\|\hat{\bu}^{(q)}_{[t]}-\bu^{(q)}_j\|
			\le C\eps_1L_t^{p-1}.
		\end{align*}
		Let us define $\gamma_{i}^{(q)}={\rm sign}(\langle\hat{\bu}^{(q)}_i,\bu_i^{(q)}\rangle)$. Note that we can ensure  $\prod_{q=1}^p\gamma_i^{(q)}=1$. Next we have
		\begin{align*}
			|T_1|
			=&\abs*{\sum_{i<j}\hat{\lambda}_i
				\langle \bu_i^{(1)}-\gamma_i^{(1)}\hat{\bu}^{(1)}_i,\bv\rangle
				\prod_{q=2}^p
				\langle\bu^{(q)},\hat{\bu}^{(q)}_{j,[t]}\rangle}
			=\abs*{\sum_{i<j}\hat{\lambda}_i
				\langle \bu_i^{(1)}-\gamma_i^{(1)}\hat{\bu}^{(1)}_i,\bv\rangle
				\prod_{q=2}^p
				\langle\bu^{(q)},\hat{\bu}^{(q)}_{j,[t]}-\bu_{j,[t]}^{(q)}\rangle}\\
			\le& C\max_{1\le i<j }\hat{\lambda}_i\sin\angle\left(\bu^{(1)}_i,\hat{\bu}^{(1)}_i\right)L_t^{p-1}
			\le C\max_{i<j}(\lambda_i+C\eps_1)\cdot\dfrac{\eps_1}{\lambda_i}\cdot L_t^{p-1}
			\le C\eps_1L_t^{p-1}.
		\end{align*}
		Proceeding this way, we have for $2\le q\le p-1$
		\begin{align*}
			&|T_q|\\
			=&\bigg|\sum_{i<j}
			\hat{\lambda}_i
			\gamma_i^{(1)}\langle\hat{\bu}_i^{(1)},\,\bv\rangle
			\left(\prod_{1<k<q}
			\gamma_i^{(k)}\langle\hat{\bu}_i^{(k)},\,\hat{\bu}_{j,[t]}^{(k)}\rangle
			\right)
			\langle\bu_i^{(q)}-\gamma_i^{(q)}\hat{\bu}_i^{(q)},
			\hat{\bu}_{j,[t]}^{(q)}
			\rangle
			\left(
			\prod_{k'>q}
			\langle\bu_i^{(k')},
			\hat{\bu}_{j,[t]}^{(k')}-\bu^{(k')}_{j}
			\rangle
			\right)
			\bigg|\\
			\le& \max_{i}\left\{\hat{\lambda}_i\sin\angle\left(\hat{\bu}_i^{(q)},\bu_i^{(q)}\right)\right\}
			\left(\sum_{i<j}\langle \gamma_i^{(1)}\hat{\bu}_i^{(1)},\,\bv\rangle^2\right)^{\tfrac{1}{2}}
			\max_{1<k<q}
			\left(\sum_{i<j}\langle \gamma_i^{(k)}\hat{\bu}_i^{(k)},\,\hat{\bu}_{j,[t]}^{(k)}\rangle^2\right)^{\tfrac{1}{2}}
			L_t^{p-q}\\
			\le& C\eps_1L_t^{p-q}\cdot \norm*{\left(\check{\bU}_{(j-1)}^{(1)}\right)^\top\bv}
			\cdot\max_{1<k<q}
			\norm*{\left(\check{\bU}_{(j-1)}^{(k)}\right)^\top\hat{\bu}_{j,[t]}^{(k)}}\\
			\le&C\eps_1L_t^{p-q}
			\cdot \left(\norm*{\check{\bU}_{(j-1)}^{(1)}-\bU^{(1)}_{(j-1)}}
			+\norm*{\bU_{(j-1)}^{(1)}}\right)
			\cdot 
			\max_{1<k<q}
			\left(
			\norm*{\check{\bU}_{(j-1)}^{(k)}-\bU^{(k)}_{(j-1)}}
			+\norm*{\left(\bU^{(k)}_{(j-1)}\right)^\top
				\left(\hat{\bu}^{(k)}_{j,[t]}-\bu^{(k)}_j)\right)}
			\right)\\
			\le&C\eps_1L_t^{p-q}\cdot (1+C\eps_1/\lambda_{j-1})\cdot(C\eps_1/\lambda_{j-1}+L_t)\\
			\le &C\eps_1^2L_t^{p-q}/\lambda_{j-1}+C\eps_1L_t^{p-q+1}
			\le C\eps_1^2L_t/\lambda_{j-1}+C\eps_1L_t^2. 
		\end{align*}
		In the above, we have used the matrix operator norm bounds from Lemma \ref{lem:esti-op-norm}. Finally,
		\begin{align*}
			&|T_p|\\
			=&\bigg|\sum_{i<j}
			\hat{\lambda}_i
			\gamma_i^{(1)}\langle\hat{\bu}^{(1)}_i,\,\bv\rangle
			\left(\prod_{1<k<p}
			\gamma_i^{(k)}\langle\hat{\bu}_i^{(k)},\,\hat{\bu}_{j,[t]}^{(k)}\rangle
			\right)
			\langle\bu_i^{(p)}-\gamma_i^{(p)}\hat{\bu}_i^{(p)},
			\hat{\bu}_{j,[t]}^{(p)}
			\rangle
			\bigg|\\
			\le &\max_{i}\left\{
			\dfrac{\hat{\lambda}_i\abs*{\scrE\times_{k\neq p}\bu_i^{(k)}\times_p\hat{\bu}_{j,[t]}^{(p)}}
			}{\lambda_i}
			+
			\hat{\lambda}_i\sin^2\angle\left(\hat{\bu}_i^{(q)},\bu_i^{(q)}\right)\right\}
			\left(\sum_{i<j}\langle \gamma_i^{(1)}\hat{\bu}_i^{(1)},\,\bv\rangle^2\right)^{\tfrac{1}{2}}
			\left(\sum_{i<j}\langle \gamma_i^{(2)}\hat{\bu}_i^{(k)},\,\hat{\bu}_{j,[t]}^{(2)}\rangle^2\right)^{\tfrac{1}{2}}
			\\
			\le&\,C
			\max_i\left\{\eps_0+\eps_1\sin\angle(\hat{\bu}_{j,[t]}^{(p)},\bu_j^{(p)})
			+\dfrac{\eps_1^2}{\lambda_i}\right\}
			\cdot (1+C\eps_1/\lambda_{j-1})\cdot(C\eps_1/\lambda_{j-1}+L_t)\\
			\le & C(\eps_0+\eps_1^2/\lambda_{j-1})L_t
			+C\eps_1L_t^2+C(\eps_0+\eps_1^2/\lambda_{j-1})\eps_1/\lambda_{j-1}. 
		\end{align*}

		Adding all the terms, we have
		\begin{equation}\label{eq:piter-def-norm}
			\sup_{\bv}\hat{\scrT}_{-j}\times_1\bv\times_{q>1}\hat{\bu}^{(q)}_{[t]}
			\le C(\eps_0+\eps_1^2/\lambda_{j-1})\eps_1/\lambda_{j-1}
			+C(\eps_0+\eps_1^2/\lambda_{j-1})L_t
			+C\eps_1L_t^2
		\end{equation}
		which finishes the proof.
	\end{proof}

    \medskip

    \begin{proof}[Proof of Lemma~\ref{lem:hosvd}] 
		We fix $q=1$ and drop the subscripts and write $\bU$ to mean $\bU^{(1)}$ only for this proof. Moreover let
        \[
            \bX:={\sf Mat}_1(\scrX)
            \quad
            \text{and}
            \quad
            \bE:={\sf Mat}_1(\scrE)
        \]
		
		As done in \cite{cai2018rate}, we define the normalization matrix
		$$
		\bM={\rm diag}((\lambda_1^2+d^{p-1})^{-1/2},\dots,(\lambda_r^2+d^{p-1})^{-1/2}).
		$$
		Note that the $r^{\rm th}$ singular value of $\bU^\top\bX$ satisfies
		\begin{align*}
			\sigma_{r}^2(\bU^\top\bX)\ge 
			&
			\sigma_r^2(\bM^{-1})\sigma_r^2(\bM\bU^\top\bX)
			\\
			=&
			(\lambda_r^2+d^{p-1})\sigma_r(\bM\bU^{\top}\bX\bX^{\top}\bU\bM-\II_r+\II_r)\\
			\ge&
			(\lambda_r^2+d^{p-1})(1-\|\bM\bU^{\top}\bX\bX^{\top}\bU\bM-\II_r\|)
			\numberthis\label{eq:cai0}
		\end{align*} 
		using Weyl's eigenvalue perturbation theorem on the last line.
		Notice that
		\begin{align*}
			&\bM\bU^{\top}\bX\bX^{\top}\bU\bM-\II_r\\
			=&\bM\boldsymbol{\Lambda}^2\bM-\II_r
			+\bM\bU^\top \bE\bE^\top\bU\bM-\II_r
			+\bM\boldsymbol{\Lambda}\bV^\top\bE^\top\bU\bM
			+\bM\bU^\top\bE\bV\boldsymbol{\Lambda}\bM\\
			=&\bM(\boldsymbol{\Lambda}^2+d^{p-1}\II_d)\bM-\II_r
			+\bM\bU^\top (\bE\bE^\top-d^{p-1}\II_d)\bU\bM
			+\bM\boldsymbol{\Lambda}\bV^\top\bE^\top\bU\bM
			+\bM\bU^\top\bE\bV\boldsymbol{\Lambda}\bM\\
			=&\bM\bU^\top (\bE\bE^\top-d^{p-1}\II_d)\bU\bM
			+\bM\boldsymbol{\Lambda}\bV^\top\bE^\top\bU\bM
			+\bM\bU^\top\bE\bV\boldsymbol{\Lambda}\bM
		\end{align*}
		Now using Lemma \ref{lem:errnormbd1},
		\begin{align*}\label{eq:sigr_bd1}
			\|\bM\bU^\top\bX\bX^\top\bU\bM-\II_r\|
			\le& \|\bE\bE^\top-d^{p-1}\II_d\|/(\lambda_r^2+d^{p-1})+2\|\bE\bV\|(\lambda_r/(\lambda_r^2+d^{p-1}))\\
			\le & \dfrac{Cd^{p/2}(\log d)}{\lambda_r^2+d^{p-1}}+\dfrac{C\sqrt{d}(\log d)\lambda_r}{\lambda_r^2+d^{p-1}}\numberthis 
		\end{align*}
		with probability at least $1-d^{-1}$. By \eqref{eq:cai0}, one then obtains
		\begin{equation}\label{eq:sigma_r}
			\sigma_r^2(\bU^\top\bX)\ge(\lambda_r^2+d^{p-1})\left(1-\dfrac{Cd^{p/2}(\log d)}{\lambda_r^2+d^{p-1}}-\dfrac{C\sqrt{d}(\log d)\lambda_r}{\lambda_r^2+d^{p-1}}\right)=\lambda_r^2+d^{p-1}-C(d^{p/2}+\lambda_r\sqrt{d})(\log d).
		\end{equation}
		We next consider $\sigma_{r+1}(\bX)$. By Weyl's min-max theorem,
		$$
		\sigma_{r+1}(\bX)=\min_{\bW:{\rm rank}(\bW)\le r}\|\bX-\bW\|
		\le \|\bX-\bU\bU^{\top}\bX\|
		=\|\bU_{\perp}\bU_{\perp}^\top\bX\|
		=\|\bU_{\perp}\bU_{\perp}^{\top}\bE|
		\le \|\bU_{\perp}^{\top}\bE\|
		$$
		Then,
		\begin{align*}\label{eq:sigma_r1}
			\sigma_{r+1}^2(\bX)\le&  \|\bU_{\perp}^{\top}\bE\|^2
			= \|\bU_{\perp}^\top\bE\bE^\top\bU_{\perp}-d^{p-1}\bU_{\perp}\bU_{\perp}^\top+d^{p-1}\bU_{\perp}\bU_{\perp}^\top\|\\
			\le & \|\bE\bE^\top-d^{p-1}\II_d\|+d^{p-1}\\
			\le & Cd^{p/2}(\log d)+d^{p-1}\numberthis
		\end{align*}
		with probability at least $1-d^{-2}$. Moreover,
		\begin{align*}\label{eq:projr_bd0}
			\|\bP_{\bX^\top\bU}\bX^\top\bU_\perp\|
			=&\|\bP_{\bX^\top\bU\bM}\bX^\top\bU_\perp\|
			\le \|(\bX^\top\bU\bM)((\bX^\top\bU\bM)^\top\bX^\top\bU\bM)^{-1}
			\bM\bU^\top\bX\bX^\top\bU_{\perp}\|\\
			\le & \dfrac{1}{\sigma_r(\bX^\top\bU\bM)}\|\bM\bU^\top\bX\bX^\top\bU_{\perp}\|\numberthis
		\end{align*}
		where the last step uses the SVD of $\bX^\top\bU\bM$. As before, we write
		\begin{align*}
			\sigma_r^2(\bX^\top\bU\bM)
			=\sigma_r(\bM\bU^\top\bX\bX^\top \bU\bM)
			\ge 1-\|\bM\bU^\top\bX\bX^\top \bU\bM-\II_r\|
			\ge 1- \dfrac{C(d^{p/2}+\lambda_r\sqrt{d})(\log d)}{\lambda^2_r+d^{p-1}}
			\ge \dfrac{1}{2}
		\end{align*}
		with probability at least $1-d^{-1}$. Here we used \eqref{eq:sigr_bd1} and assumption on $\lambda_r$ in the last two inequalities.	Finally we have
		\begin{align*}
			\|\bM\bU^\top\bX\bX^\top\bU_\perp\|
			=\|\bM\boldsymbol{\Lambda}\bV^\top\bE^\top\bU_\perp +\bM\bU^\top(\bE\bE^\top-d^{p-1}\II_d)\bU_\perp\|
			\le \dfrac{C\sqrt{d}\lambda_r(\log d)}{\lambda_r^2+d^{p-1}}+\dfrac{Cd^{p/2}(\log d)}{
				\lambda_r^2+d^{p-1}}
		\end{align*}
		using Lemma \ref{lem:errnormbd1} and the fact that $\|\bM\|= (\lambda_r^2+d^{p-1})^{-1}$. Plugging in the last two inequalities into 
		\eqref{eq:projr_bd0}, one gets
		\begin{equation}\label{eq:projr_bd}
			\|\bP_{\bX^\top\bU}\bX^\top\bU_{\perp}\|
			\le \dfrac{C\sqrt{d}(\log d)\lambda_r}{\lambda_r^2+d^{p-1}}+\dfrac{Cd^{p/2}(\log d)}{
				\lambda_r^2+d^{p-1}}.
		\end{equation}
		Let $\hat{\bU}\in \RR^{d\times r}$ be the matrix whose columns are the top $r$ left singular vectors of $\bX$. We also define $\hat{\bP}=\hat{\bU}\hat{\bU}^\top$. Now using Proposition 1 of \cite{cai2018rate}, we have
		\begin{align*}\label{eq:tuckinit}
			\|\hat{\bP}-\bP\|^2=\|\sin\Theta(\hat{\bU},\,\bU)\|^2
			\le& \dfrac{\sigma_r(\bU^\top\bX)\|\bP_{\bX^\top\bU}\bX^\top\bU_{\perp}\|}
			{(\sigma_r^2(\bU^\top\bX)-\sigma_{r+1}^2(\bX))}\\
			\le &
			\dfrac{(\lambda_r^2+d^{p-1}-C(d^{p/2}+\lambda_r\sqrt{d}))
				\cdot \dfrac{C(\sqrt{d}\lambda_r+d^{p/2})(\log d)}{\lambda_r^2+d^{p-1}}}
			{\lambda_r^2+d^{p-1}-d^{p-1}-Cd^{p/2}}\\
			\le& \dfrac{C(\lambda_r\sqrt{d}+d^{p/2})(\log d)}{\lambda_r^2}\numberthis
		\end{align*}
		with probability at least $1-d^{-1}$. Here we have applied the bounds from equations \eqref{eq:sigma_r} \eqref{eq:sigma_r1} and \eqref{eq:projr_bd} in the second inequality, along with the fact that $x^2/(x^2-y^2)^2$ is a decreasing function of $x$. In the last inequality, we used the assumption that $\lambda_r\ge Cd^{p/4}(\log d)$. 
		
		The result for $\bP^{(1,2)}$ follows by retracing identical steps, noting that $\bP^{(1,2)}$ is a $d^2\times d^2$ dimensional matrix.
	\end{proof}
	
	\medskip

\begin{proof}[Proof of Lemma \ref{lem:sing-gap}]
		The proof follows from Lemma B.1 of \cite{anand2014sample}.
	\end{proof}	

\begin{lemma}\label{lem:eps012-bds}
	Under the assumptions of Theorem~\ref{th:init-incoh}, with probability at least $1-1/Cd$, we have $\eps_0\le Cd^{1/4}$, $\eps_1\le C\sqrt{d}$ and $\eps_2\le C\sqrt{d}$.
\end{lemma}

\begin{proof}[Proof of Lemma \ref{lem:eps012-bds}]
	Since $\scrE$ has entries which are independent copies of a random variable $E$ with $\EE(E)=0$, $\EE(E^2)=1$ and $\EE(E^8)\le C$, we have that,
	$
	E_j=\scrE\times_{q=1}^p\bu^{(q)}_j
	$
	satisfies $\EE(E_j)=0$, $\EE(E_j^2)=1$ and $\EE(E_j^8)\le C$. Consequently,
	$$
	\eps_0=\max_{1\le j\le r}E_j\le Cd^{1/4}
	$$
	with probability at least $1-\dfrac{r}{Cd^2}$. Similarly, we have $\bE_j^{(q)}=\scrE\times_{k\neq q}\bu^{(k)}_j\in \RR^d$ with entries $\{E_{ji}:1\le i\le d\}$ which are independent copies of $E'$ with $\EE E'=0$, $\EE (E')^2=1$ and $\EE (E')^8\le C$. Moreover, $\EE\|\bE_j^{(q)}\|^2=d$, and $\EE(\sum _i E_{ij}^2-1)^4\le Cd^2$
	by symmetrization and Rosenthal inequality (see, e.g., Corollary 2 and Remark 2 of \cite{latala1997estimation}). Then 
	we have
	\begin{align*}
		\PP(\eps_1\ge C\sqrt{d})
	\le dp\PP\left(\sum_{i=1}^d (E_{ji}^2-1)\ge Cd\right)
	\le \dfrac{d^3}{Cd^4}\le \dfrac{1}{Cd}.
	\end{align*}
	The proof for $\eps_2$ follows similarly.
\end{proof}

\medskip
		
		\begin{lemma}\label{lem:EVthetanorm}
			$$
				\|\bE_{\theta}\tilde{\bV}_j\|\le
				\begin{cases}
					C\sqrt{dr}(\log d)&\text{ when }p=3\\
					C\sqrt{dr}(\log d)
					+C\dfrac{d^{(p+1)/2}r(\log d)}{\lambda_r^2}
					&\text{ when }p\ge 4.
				\end{cases}
			$$
			with probability at least $1-d^{-1}$.
		\end{lemma}
		
		\begin{proof}[Proof of Lemma \ref{lem:EVthetanorm}]
			Note that, for any $i\ge j$,
			\begin{align*}\label{eq:defE-theta-v}
				\bE_{\theta}{\bv}_i
				=\sum_{l_2\dots l_{p-1}}
				(\bv_{i})_{l_2\dots l_{p-1}}\bE_{l_2\dots l_{p-1}}\bZ
				\stackrel{d}{=}\bE_{(i)}\bZ\numberthis
			\end{align*}
			Here $\bE_{(i)}$ and $\bE_{l_2\dots l_{p-1}}$ are independent $d\times d$ dimensional random matrices with independent copies of $E$ as entries, for $1\le l_q\le d$, $1\le q\le p-1$. We write  
			\begin{equation}\label{eq:defz}
				\bZ=\bP^{(p)}\theta+(\hat{\bP}^{(p)}-\bP^{(p)})\theta
				=:\bP^{(p)}\theta+\bz
			\end{equation}
							
			Since $\bP^{(p)}\theta\sim N(0,\bP^{(p)})$ and is independent of $\bE_{(i)}$, 
			$$
			\left(\bE_{(i)}\bP^{(p)}\theta\right)_{st}
			\quad\text{are independent, }
			\EE\bE_{(i)}\bP^{(p)}\theta=\mathbf{0},
			\quad
			\Var\left(\bE_{(i)}\bP^{(p)}\theta\right)=\|\bP\theta\|^2\II_d.
			$$
			Moreover, for any $s,t\in [r]/[j]$,  $\EE\left(\bE_{(s)}\bP^{(p)}\theta\right)^\top \left(\bE_{(t)}\bP^{(p)}\theta\right)=\delta_{i=j}$. Thus, the matrix 
			\begin{equation}\label{eq:tildE}
			\tilde{\bE}=[\bE_{(j)}\bP^{(p)}\theta
			\,\,
			\bE_{(j+1)}\bP^{(p)}\theta
			\,\,
			\dots\,
			\bE_{(r)}\bP^{(p)}\theta]
			\end{equation}
			has independent isotropic rows. By the moment assumption that $\EE(E^8)<C$, it can be checked that
			$$
			\max_{1\le i\le d}\|\be_i^\top\tilde{\bE}\|\le C\sqrt{d}\|\bP^{(p)}\theta\|\le C\sqrt{dr}
			$$
			with probability at least $1-d^{-3}$. Thus using Theorem 5.41 from \cite{vershynin2010introduction}, we have
			\begin{equation}\label{eq:tildEnorm}
				\|\tilde{\bE}\|\le C\sqrt{dr}(\log d)
			\end{equation}
			with probability at least $1-d^{-3}$. 
			
			Note that by \eqref{eq:defz} and \eqref{eq:tildE}, we have 
			\begin{equation}\label{eq:tildE-dec}
				\|\bE_{\theta}\tilde{\bV}_j\|
				=\tilde{\bE}+[\bE_{(j)}\bz,\,\,\dots,\,\,\bE_{(r)}\bz]
				=\tilde{\bE}+\hat{\bE}
			\end{equation}
			It remains to bound the part $\hat{\bE}$, that involves $\bz$. Since $\bz$ is no longer independent of $\bE$, we will instead use the uniform bound 
			$
			\sup_{\bw\in \SS^{d-1}}\|\sum_{i=1}^dw_i\check{\bE}_{(i)}\|
			$
			where $\check{\bE}_i$ are independent random matrices.

			We divide the proof into two cases.
			
			\paragraph{Case 1: $p=3$}
			We need to bound the operator norm of 
			$\hat{\bE}=\left[
			\bE_{(j)}\bz\,\,\dots\,\,\bE_{(r)}\bz
			\right].
			$
			 
			By the definition of $\bE_{(i)}$ in \eqref{eq:defE-theta-v}, we write 
			\begin{align*}
				\|\hat{\bE}\|
				=\sup_{\by\in \SS^{r-j+1}}
				\norm*{\sum_{i=j}^{r}
				y_i\bE_{(i)}\bz}
			=\sup_{\bx,\by}\sum_{l_2}(\sum_i y_i\bv_i)_{l_2}
			\bx^{\top}\bE_{l_2}\bz
			\le \sup_{\bx}
			\left(
			\sum_{l_2}(\bx^{\top}\bE_{l_2}\bz)^2
			\right)^{1/2}
			\end{align*}
			by Cauchy-Schwarz inequality, noting that $\sum_i y_i\bv_i\in \SS^{d-1}$. Note that $\bE_{l_2}$ are independent $d\times d$ random matrices with independent copies of $E$ as their entries.
			
			We can now apply Lemma \ref{lem:errnormbd2} with $n=d$ to get 
			\begin{align*}
				\sup_{\bx,\bz\in \SS^{d-1}}
				\left(
				\sum_{l_2}(\bx^{\top}\bE_{l_2}\bz)^2
				\right)
				\le&
				\left(d\|\scrG\times_1\bx\times_2\bz\times_3\bx\times_4\bz\|
				+\sqrt{Cd^2(\log d)^2}\right)\\
				\le & d\sum_i\sum_jx_i^2y_j^2+Cd(\log d)\\
				\le & Cd(\log d)
			\end{align*}
		with probability at least $1-d^{-2}$. Finally we have
		$$\|\bz\|\le C\|\hat{\bP}^{(p)}-\bP^{(p)}\|_{\rm F}
		\sqrt{\log d}\le 
		\dfrac{C\sqrt{r}(d^{3/2}+\lambda_r\sqrt{d})(\log d)^{3/2}
		}{\lambda_j^2}$$
		by Lemma~\ref{lem:hosvd}. Consequently, we get 
		$$\|\hat{\bE}\|\le C\sqrt{d(\log d)}\|\bz\|
		\le
		\dfrac{C\sqrt{r}(d^2+d\lambda_r)(\log d)^{2}
		}{\lambda_j^2}
		 ,
		$$
		with probability at least $1-d^{-1}$, 
		and thus using \eqref{eq:tildEnorm} and \eqref{eq:tildE-dec}, that
		$$
		\|\bE_{\btheta}\tilde{\bV}_j\|\le C\sqrt{dr}(\log d),
		$$
		since $\lambda_j\ge Cd^{3/4}(\log d)$.
			
		\paragraph{Case 2: $p\ge 4$.}
			The rows of $\check{\bE}_{(i)}$ are independent and each row $\check{\bE}_{(i)j}$ is isotropic. Since $\bw$ is a unit vector, it can be checked that, for fixed $\bw$, $\EE(\sum w_i\check{\bE}_{(i)})^\top(\sum w_i\check{\bE}_{(i)})=d\II_r$. We will bound
			\begin{align*}
				\norm*{(\sum w_i\check{\bE}_{(i)})^\top(\sum w_i\check{\bE}_{(i)})-d\II_r}
				=&
				\norm*{
				\sum_{i_1}w_{i_1}
				\sum_{j=1}^d
				\left\{\check{\bE}_{(i_1)j}\left(\sum_{i_2}w_{i_2}
				\check{\bE}_{(i_2)j}\right)^{\top}
				-w_{i_1}\II_r
				\right\} 	
			}\\
		=&
		\sup_{\by\in \SS^{r-j+1}}
		\norm*{
		\sum_{k=1}^d\left(\hat{\bE}_{(k)}\by\by^\top\hat{\bE}_{(k)}^\top
		-\II_d
		\right)},	
		\end{align*} for independent random matrices $\hat{\bE}_{(k)}\stackrel{d}{=}\check{\bE}\in\RR^{d\times r}$.		By the moment condition, it holds that $\max_{1\le k\le d}\|\check{\bE}_{(k)}\|\le \sqrt{d}$ with high probability.	Then using matrix Bernstein inequality for each fixed $\by$, one has
		$$
		\PP\left(
		\norm*{
			\sum_{k=1}^d\left(\hat{\bE}_{(k)}\by\by^\top\hat{\bE}_{(k)}^\top
		-\II_d
		\right)}
		\ge t
		\right)
		\le d\exp(-\dfrac{t^2/C}{d^2+dt})
		$$
		Taking $t=dr$ and then taking an $\eps$-net of $\SS^{r-j+1}$, which has size at most $\exp(Cr)$, applying a union bound yields
		$$
		\sup_{\bw\in \SS^{d-1}}\|\sum_{i=1}^dw_i\check{\bE}_{(i)}\|
		=\left(
		d+\sup_{\by\in \SS^{r-j+1}}
		\norm*{
			\sum_{k=1}^d\left(\hat{\bE}_{(k)}\by\by^\top\hat{\bE}_{(k)}^\top
			-\II_d
			\right)}
		\right)^{1/2}
		\le C\sqrt{dr\log d}
		$$
		with probability at least $1-d^{-3}$ (this bound arises from bounding the maximum of the matrix norms $\check{\bE}_{(k)}$, $1\le k\le d$). We then have, for $\bz=(\hat{\bP}^{(p)}-\bP^{(p)})\theta$, that
		\begin{align*}
			\norm*{\bE_{\theta}\tilde{\bV}_j}
			=&\norm*{[\bE_{(j)}\bZ\,\,\bE_{(j+1)}\bZ\,\,\dots\,\,\bE_{(r)}\bZ]}\\
			&\le 
			\norm*{\tilde{\bE}}+
			 \norm*{[\bE_{(j)}\bz\,\,\bE_{(j+1)}\bz\,\,\dots\,\,\bE_{(r)}\bz]}
			 \\
			 &\le C\sqrt{dr}(\log d)
			 +\sup_{\bw\in \SS^{d-1}}\|\sum_{i=1}^dw_i\check{\bE}_{(i)}\|\|\bz\|\\
			 &\le C\sqrt{dr}(\log d)+C\sqrt{dr(\log d)}
			 \cdot\left(\dfrac{C\sqrt{r}d^{p/2}(\log d)}{\lambda_r^2}
			 +\dfrac{C\sqrt{d(\log d)^2}}{\lambda_r}
			 \right)\\
			 &\le C\sqrt{dr}(\log d)+\dfrac{Cd^{(p+1)/2}r\log d}{\lambda_r^2}
		\end{align*}
	with probability at least $1-d^{-3}$, since $\lambda_r\ge d^{p/4}(\log d)\ge d$. In the above, we have used \eqref{eq:tildE}, \eqref{eq:tildEnorm}, the definition of $\check{\bE}_{(i)}$ and the upper bound
	$$\|\bz\|\le C\|\hat{\bP}^{(p)}-\bP^{(p)}\|_{\rm F}
	\sqrt{\log d}\le 
	\dfrac{C\sqrt{r}(d^{p/2}+\lambda_r\sqrt{d})(\log d)^{3/2}
	}{\lambda_j^2}$$
	 using Lemma~\ref{lem:hosvd}.
		\end{proof}
		
\medskip

		\begin{lemma}\label{lem:EVthetanorm-2}
        If $\bE_\theta$ is defined using sample splitting as in Algorithm~\ref{alg:alg-init-incoherent} we have
			$$
				\|\bE_{\theta}\tilde{\bV}_j\|\le					
                C\sqrt{dr}(\log d)
			$$
			with probability at least $1-d^{-1}$.
		\end{lemma}
		
		\begin{proof}[Proof of Lemma \ref{lem:EVthetanorm-2}]
			Note that, for any $i\ge j$,
			\begin{align*}\label{eq:defE-theta-v}
				\bE_{\theta}{\bv}_i
				=\sum_{l_2\dots l_{p-1}}
				(\bv_{i})_{l_2\dots l_{p-1}}\bE_{l_2\dots l_{p-1}}\bZ
				\stackrel{d}{=}\bE_{(i)}\bZ\numberthis
			\end{align*}
			Here $\bE_{(i)}$ and $\bE_{l_2\dots l_{p-1}}$ are independent $d\times d$ dimensional random matrices with independent copies of $E$ as entries, for $1\le l_q\le d$, $1\le q\le p-1$. We write  
			\begin{equation}\label{eq:defz-2}
				\bZ=\hat{\bP}^{(1,2)}\btheta
			\end{equation}							
			Since $\hat{\bP}^{(1,2)}\btheta\sim N(0,\hat{\bP}^{(1,2)})$ and is constructed from $\scrX_{[1]}$, hence independent of $\bE_{(i)}$, we have 
			$$
			\left(\bE_{(i)}\bP^{(p)}\theta\right)_{st}
			\quad\text{are independent, }
			\EE\bE_{(i)}\hat{\bP}^{(1,2)}\btheta=\mathbf{0},
			\quad
			\Var\left(\bE_{(i)}\hat{\bP}^{(1,2)}\btheta\right)=\|\hat{\bP}^{(1,2)}\btheta\|^2\II_d.
			$$
			Moreover, for any $s,t\in [r]/[j]$,  $\EE\left(\bE_{(s)}\hat{\bP}^{(1,2)}\btheta\right)^\top \left(\bE_{(t)}\hat{\bP}^{(1,2)}\btheta\right)=\delta_{i=j}$. Thus, the matrix 
			\begin{equation}\label{eq:tildE}
			\tilde{\bE}=[\bE_{(j)}\hat{\bP}^{(1,2)}\btheta
			\,\,
			\bE_{(j+1)}\hat{\bP}^{(1,2)}\btheta
			\,\,
			\dots\,
			\bE_{(r)}\hat{\bP}^{(1,2)}\btheta]
			\end{equation}
			has independent isotropic rows. By the moment assumption that $\EE(E^8)<C$, it can be checked that
			$$
			\max_{1\le i\le d}\|\be_i^\top\tilde{\bE}\|\le C\sqrt{d}\|\hat{\bP}^{(1,2)}\btheta\|\le C\sqrt{dr}
			$$
			with probability at least $1-d^{-3}$. Thus using Theorem 5.41 from \cite{vershynin2010introduction}, we have
			\begin{equation}\label{eq:tildEnorm}
				\|\tilde{\bE}\|\le C\sqrt{dr}(\log d)
			\end{equation}
			with probability at least $1-d^{-3}$. 
		\end{proof}
		
\medskip

\begin{lemma}\label{lem:errnormbd1}
	$\|\bE\bE^\top-d^{p-1}\II_d\|\le Cd^{p/2}(\log d)$, $\|\bE\bV\|\le C\sqrt{d}(\log d)$ with probability at least $1-d^{-1}$.
\end{lemma}

\begin{proof} 
	For the first inequality, note that the columns $\bE_i$ of the matrix $\bE\in \RR^{d\times d^{p-1}}$ are independent random vectors with $\EE(\bE_i\bE_i^\top)=\II_d$. Define the truncated version of the columns as $\bY_i:=\bE_i\mathbbm{1}(\|\bE_i\|\le Cd^{p/4})$. Note that $\|\bY_i\bY_i^\top\|\le Cd^{p/2}$ almost surely, and 
	$$
	\norm*{\EE(\bY_i\bY_i^\top\bY_i\bY_i^\top)}\le 
	\sup_{\bx\in \SS^{d-1}}
	\left(\{\EE \|\bY_i\|^4\}
	\{\EE\langle \bY_i,\bx\rangle^4\}
	\right)^{1/2}
	\le C(\EE\|\bE_i\|^4)^{1/2}\le Cd
	$$
	using the moment condition that $\EE(E^8)<C$, and hence, for any $\bx\in \SS^{d-1}$, $\EE\langle \bE_i,\bx\rangle^4\le C$ by Rosenthal inequalities. Similarly it is not hard to check that $\EE\|\bX_i\|^4\le Cd^2$.
	 
	Thus by matrix Bernstein inequality,
	\begin{equation}\label{eq:mat-bern}
		\norm*{\sum_{i=1}^{d^{p-1}}\bY_i\bY_i^{\top}-d^{p-1}\EE(\bY_i\bY_i^{\top})}
		\le  
		Cd^{p/2}(\log d)
	\end{equation}
	with probability at least $1-d^{-2}$. It remains to bound the un-truncated part
	\begin{align*}
		\norm*{
	\sum_{i=1}^{d^{p-1}}
	\bE_i\bE_i^{\top}\mathbbm{1}(\|\bE_i\|>Cd^{p/4})
	-d^{p-1}\EE(\bE_i\bE_i^{\top}\mathbbm{1}(\|\bE_i\|>Cd^{p/4}))	
	}.
	\end{align*}
	To this end, we will use the tail bound techniques from \cite{adamczak2010}, \cite{vershynin2011approximating} and \cite{vershynin2012}. Note that
	$$
	\PP(\|\bE\|>Cd^{p/4})
	=\PP(\sum_{i=1}^{d}(E_i^2-1)>Cd^{p/2})
	\le \dfrac{Cd\EE(E^4)}{Cd^{p}}
	\le \dfrac{C}{d^{p-1}}. 
	$$
	Let $N:=\displaystyle\sum_{i=1}^{d^{p-1}}\mathbbm{1}(\|\bE_i\|>Cd^{p/4})$. Since $\bE_i$ are i.i.d., we have that $\EE(N)=d^{p-1}\PP(\|\bE\|>Cd^{p/4})\le C$. Then by Chernoff bounds for Binomial counts, we have
	$$
	\PP(N>C(\log d))
	=\PP(N-\EE(N)> C(\log d))\le d^{-2}.
	$$
	Then 
	\begin{equation}\label{eq:tail-bd1}
	\norm*{
		\sum_{i=1}^{d^{p-1}}
		\bE_i\bE_i^{\top}\mathbbm{1}(\|\bE_i\|>Cd^{p/4})}
	\le N\max_{1\le i\le d^{p-1}}\|\bE_i\|^2
	\le C(\log d)d^{p/2}
	\end{equation}
	with probability at least $1-d^{-2}$, since
	$$
	\PP(\max_{1\le i\le d^{p-1}}\|\bE_i\|^2>d^{p/2})
	\le d^{p-1}\cdot \dfrac{d^2\EE(|E_{ij}^2-1|^4)}{d^{2p}}
	\le d^{1-p}
	$$
	by Rosenthal and Khintchine inequalities, since $\EE(E)=0,\EE(E^2)=1$ and $\EE(E^8)\le C$. Finally, it can be checked that
	\begin{equation}\label{eq:tail-bd2}
			\|\EE(\bE_i\bE_i^{\top}\mathbbm{1}(\|\bE_i\|>Cd^{p/4}))\|
		\le \sup_{\bx\in \SS^{d-1}}(\EE\langle\bE_i,\bx\rangle^4\PP(\|\bE_i\|>Cd^{p/4}))^{1/2}
		\le C\cdot d^{(p-1)/2}.
	\end{equation}
	Adding the bounds from \eqref{eq:mat-bern}, \eqref{eq:tail-bd1} and \eqref{eq:tail-bd2} finishes the proof of the first part of the lemma. 
	
	For the second part, note that $\bV\in \RR^{d^{p-1}\times r}$ and  $\bV^{\top}\bV=\II_d$. Since the elements of $\bE$ are i.i.d., we have that $\bE\bV$ has independent rows $\bW_i=(\be_i^{\top}\bE\bV)^{\top}$, which satisfy
	$
	\EE(\bW_i\bW_i^{\top})=\II_r 
	$. 
	Since $\be_i^{\top}\bE$ has i.i.d. entries with mean zero and variance one, it can be checked that, for $\bA=\bV\bV^{\top}$,
	$$
	\EE(\be_i^{\top}\bE\bA\bE^{\top}\be_i)={\rm tr}(\bA)={\rm tr}(\bV^{\top}\bV)=r
	$$
	and
	\begin{align*}
		\EE(\be_i^{\top}\bE\bA\bE^{\top}\be_i)^4
		=\EE(\sum_{j,k}a_{jk}E_{ij}E_{ik} )^4
		=&\sum_{j,k}a_{jk}^4\EE(E_{ij}E_{ik})^4
		+\sum_{j_1,k_1}\sum_{j_2k_2}a_{j_1k_1}^2a_{j_2k_2}^2\\
		\le& C\|\bA\|_F^4
		\le C({\rm tr}(\bA^{\top}\bA))^2=Cr^2.
	\end{align*}
	which means, by Markov inequality, that
	\begin{align*}
		\PP(\|\bW_i\|\ge C\sqrt{d})
		=& \PP(\be_i^{\top}\bE\bV\bV^{\top}\bE^{\top}\be_i^{\top}
		\ge Cd)
		\le \dfrac{r^2}{d^4}.
	\end{align*}
	As before, if we define $S=\{i:\|\bW_i\|\ge C\sqrt{d}\}$, we have using Chernoff bounds that $\PP(|S|>C\log d)\le d^{-2}$.
	
	We now follow the truncation strategy in part 1 of the proof. By a possible permutation of the rows, let us write $\bE\bV=[\bW_{(1)}^{\top}\,\bW_{(2)}^{\top}]^{\top}$ where $\bW_{(1)}$ contain the rows corresponding to $i\notin S$. This means, by Theorem 5.41 of \cite{vershynin2010introduction}, that
	$$
	\|\bW_{(1)}\|\le C\sqrt{d}(\log d)
	$$ 
	with probability at least $1-d^{-2}$. On the other hand, 
	$$
	\|\bW_{(2)}\|\le N\cdot \max\|\bW_i\|\le C\sqrt{d}(\log d)
	$$
	since $\PP(\max\|\bW_i\|\ge C\sqrt{d})\le r^2/d^3$. This finishes the proof of part 2.
\end{proof}	

\bigskip

\begin{lemma}\label{lem:errnormbd2}
	Let $\bE_{(i)}$ be $d\times d$ iid random matrices whose elements are independent copies of a random variable $E$, where $\EE(E)=0$ and $\EE(E^2)=1$ and $\EE(|E|^{\alpha})<C$ for some $\alpha\ge 8$. Then
	$$
	\norm*{\dfrac{1}{n}\sum_{s=1}^{n}\bE_{(s)}\circ \bE_{(s)}-\scrG}
	\le \sqrt{\dfrac{Cd(\log d)^2}{n}},
	$$
	provided $n>Cd$, with probability at least $1-d^{-2}$. Here $\scrG\in \RR^{d\otimes 4}$ has entries $\scrG_{ijkl}=\delta_{(i,j)(k,l)}$ where $\delta$ is Kronecker's delta function.
\end{lemma}	


\medskip

	\begin{lemma}\label{lem:piter-so-err}
		After $T=C\log (\eps_1)$ iterations, the power iteration estimator $\hat{\bu}^{(q)}_j$ satisfies
		$$
		\abs*{\langle\bu_j^{(q)},\hat{\bu}^{(q)}_j\rangle^2-1
		+\dfrac{1}{\lambda^2_j}
		\left(\bE^{\top}\bU_{\perp}\bU_{\perp}^{\top}\bE
		+\bE^{\top}\bU_{\perp}(\tilde{\bLambda}+\bU_{\perp}^{\top}\Delta_j)
		\right)}
		\le C\left(\dfrac{\eps_0\eps_1}{\lambda_j^3}+\dfrac{\eps_1^3\eps_2}{\lambda_j^4}\right)$$
	\end{lemma}

	\begin{proof}[Proof of Lemma \ref{lem:piter-so-err}]
		We fix $q=1$ and assume that $\pi=Id$. Note that we use the power iteration estimator $\hat{\bu}^{(1)}:=\hat{\bu}_{j,[T+1]}^{(1)}$ for some $T>\log (\eps_1)$. Thus, $\check{\bu}^{(q)}:=\hat{\bu}_{j,[T]}^{(q)}$ satisfy
		$$
		\sin\angle\left(\check{\bu}^{(q)},\bu^{(q)}_j\right)\le \dfrac{\eps_1}{\lambda_j}
		$$ 
		by Theorem~\ref{th:piter} and Lemma~\ref{lem:eps012-bds}. By the definition of the power iteration estimator, $\hat{\bu}^{(1)}$ is the top eigenvector of 
		\begin{align*}\label{eq:defGrand}
			(\scrX\times_{q>1}\check{\bu}^{(q)})(\scrX\times_{q>1}\check{\bu}^{(q)})^{\top}
			=&\hat{\lambda}_j^2\bu_j^{(1)}\left(\bu^{(1)}_j\right)^\top
			+\bY\bY^{\top}
			+\hat{\lambda}_j\bu^{(1)}_j\bY^{\top}
			+\hat{\lambda}_j\bY\left(\bu^{(1)}_j\right)^{\top}\\
			=:&\hat{\lambda}_j^2\bu_j^{(1)}\left(\bu^{(1)}_j\right)^\top+\bG
			\numberthis
		\end{align*}
		where $\hat{\lambda}_j:=\lambda_j\prod_{q>1}(\langle\bu^{(q)}_j,\check{\bu}^{(q)}\rangle)$
		and 
		\begin{equation}\label{eq:defYrand}
			\bY=\sum_{i<j}\lambda_j\prod_{q>1}
			\langle\bu^{(q)}_j,\,\check{\bu}^{(q)}\rangle\bu^{(1)}_j
			+\hat{\scrT}_{-j}\times_{q>1}\check{\bu}^{(q)}
			+\scrE\times_{q>1}\check{\bu}^{(q)}
			=:\bU_{\perp}\tilde{\bLambda}+\bDelta_j+\bE,
		\end{equation}
	for $\tilde{\bLambda}\in \RR^{d-1}$, with $\tilde{\lambda}_i:=\lambda_i\displaystyle\prod_{q>1}
	\langle\bu^{(q)}_i,\,\check{\bu}^{(q)}\rangle$ for $j+1\le i\le r$ and $\tilde{\lambda}_j=0$ otherwise. We also define  $\bU_{\perp}=[\bu^{(1)}_1\dots\,\bu^{(1)}_{j-1}\,\bu^{(1)}_{j+1}\,\dots\,\bu^{(1)}_d]$. We will first bound $\|\bG\|$.
	
	Note that
	\begin{equation}\label{eq:lamb-hat}
		\hat{\lambda}_j=\lambda_j\prod_{q>1}
		\langle\bu^{(q)}_j,\,\check{\bu}^{(q)}\rangle
		\ge \lambda_j\left(1-\dfrac{C\eps_1^2}{\lambda_j^2}\right)
	\end{equation}
	Next, we have, $\tilde{\bLambda}=
	{\rm diag}\left(
	\mathbf{0}_{j-1},\,
	\displaystyle
	\lambda_{j+1}\prod_{q>2}
	\langle\bu^{(q)}_{j+1},\,\check{\bu}^{(q)}\rangle,
	\,\dots \,,
	\lambda_r
	\prod_{q>2}
	\langle\bu^{(q)}_{r},\,\check{\bu}^{(q)}\rangle,
	\,\mathbf{0}_{d-r}\right)
	\left(\bU^{(2)}_{\perp}\right)^{\top}\check{\bu}^{(2)}_j$. By \eqref{eq:secopert}, we have for $i\neq j$ that 
	\begin{equation}\label{eq:delocal}
		|\langle\bu_i^{(q)},\check{\bu}^{(q)}\rangle|
		\le \dfrac{1}{\lambda_j}|\scrE\times_{k\neq q}\bu_j^{(k)}\times_q\bu_i^{(q)}|
		+\dfrac{C\eps_1^2}{\lambda_j^2}
		\le \dfrac{C\eps_0}{\lambda_j}+\dfrac{C\eps_1^2}{\lambda_j^2}
	\end{equation}
	since $\sin\angle\left(\bu_j^{(q)},\check{\bu}^{(q)}\right)\le L_T\le \dfrac{C\eps_1}{\lambda_j}$. Thus
	\begin{align*}\label{eq:til-lambd-norm}
		\|\tilde{\bLambda}\|
		\le& \sin\angle\left(\bu^{(2)}_j,\check{\bu}^{(2)}\right)
		\max_{j+1\le i\le r}\lambda_i \prod_{q>2}
		\langle\bu^{(q)}_i,\,\check{\bu}^{(q)}\rangle \\
		\le&  \dfrac{C\eps_1}{\lambda_j}\cdot \lambda_j\cdot \left(\dfrac{C\eps_0+C\eps_1^2/\lambda_j}{\lambda_j}\right)^{p-2}
		\le \dfrac{C\eps_0\eps_1}{\lambda_j}+\dfrac{C\eps_1^3}{\lambda_j^2}.
		\numberthis
	\end{align*}
	Next, we have using \eqref{eq:piter-def-norm}, (and similarly checking the bound for $\bDelta_j^{\top}\bu^{(1)}_j$), that with $L_t=L_T=C\eps_1/\lambda_j$,
	\begin{equation}\label{eq:asy-del-bd}
		\|\bDelta_j\|\le \dfrac{C\eps_0\eps_1}{\lambda_j}+\dfrac{C\eps_1^3}{\lambda_j^2}.
	\end{equation}
	Finally, following the steps of \eqref{eq:piter-errbd}, we have
	\begin{equation}\label{eq:asy-err-bd}
		\|\bE\|=\|\scrE\times_{q>1}\check{\bu}^{(q)}\|
		\le \eps_1+\dfrac{\eps_1\eps_2}{\lambda_j}+\left({C\eps_1\over\lambda_j}\right)^{2}\|\scrE\|
		\le C\eps_1+\dfrac{C\eps_1^2}{\lambda_j}
	\end{equation}
	since $\lambda_1\ge C\|\scrE\|$, and using the bounds on $\eps_1$ and $\eps_2$. Plugging these bounds into \eqref{eq:defYrand}, one has
	\begin{equation}\label{eq:asy-Y-bd}
		\|\bY\|\le \|\tilde{\bLambda}\|+
		\|\bDelta_j\|
		+\|\bE\|\le C\eps_1+\dfrac{C\eps_1^2}{\lambda_j}\le C\eps_1
	\end{equation}
	which means, using \eqref{eq:defGrand} and \eqref{eq:lamb-hat}, that
	\begin{align*}\label{eq:asy-G-bd}
		\dfrac{\|\bG\|}{\hat{\lambda}_j^2}\le \dfrac{\|\bY\|^2+\hat{\lambda}_j\|\bY\|}{\hat{\lambda}_j^2}
		\le \dfrac{C\eps_1^2+\dfrac{C\eps_1^4}{\lambda_j^2}+C\lambda_j\eps_1}{\hat{\lambda}_j^2}
		\le \dfrac{C\eps_1}{\lambda_j}.\numberthis
	\end{align*}

	Consider $\ba=\bu^{(1)}_j$ first. Using resolvent based series expansion of projection matrices, we have the following expression. See Theorem 1 from \cite{xia2021normal} and Lemma 1 of \cite{koltchinskii2016asymptotics}.
	\begin{align*}\label{eq:uusq-expnd}
		\langle\bu^{(1)}_1,\hat{\bu}^{(1)}\rangle^2-1
		=&-\dfrac{1}{\hat{\lambda}^4_j}
		(\bu^{(1)}_j)^{\top}\bG\bU_{\perp}\bU_{\perp}^\top\bG\bu^{(1)}_j
		-\dfrac{2}{\hat{\lambda}^6_j}
		(\bu^{(1)}_j)^{\top}\bG\bU_{\perp}\bU_{\perp}^\top\bG
		\bU_{\perp}\bU_{\perp}^\top\bG\bu^{(1)}_j\\
		&+\dfrac{2}{\hat{\lambda}_j^6}
		(\bu^{(1)}_j)^{\top}\bG\bU_{\perp}\bU_{\perp}^\top\bG
		\bu^{(1)}_j(\bu^{(1)}_j)^\top\bG\bu^{(1)}_j
		+\calR_4\numberthis
	\end{align*}
	where $\calR_4$ is the fourth order remainder term from the series expansion in Theorem 1 of \cite{xia2021normal}. By the definition of $\bG$ from \eqref{eq:defGrand},
	\begin{align*}\label{eq:asy-fo-term}
		\bU_{\perp}^{\top}\bG\bu^{(1)}
		=&(\hat{\lambda}_j
		+\Delta_j^{\top}\bu^{(1)}_j
		+\scrE\times_1\bu^{(1)}_1\times_{q=2}^p\check{\bu}^{(q)})
		(\tilde{\bLambda}+\bU^{\top}_{\perp}\bDelta_j
		+\bU_{\perp}^{\top}\bE)\\
		=&\left(\lambda_j+O\left(\dfrac{\eps_1^2}{\lambda_j}\right)\right)
		(\tilde{\bLambda}+\bU^{\top}_{\perp}\bDelta_j+\bU_{\perp}^{\top}\bE)\\
		=&\lambda_j\bU_{\perp}^{\top}\bE
		+\lambda_jO(\|\tilde{\bLambda}\|+\|\bDelta_j\|)
		+O\left(\dfrac{\eps_1^2}{\lambda_j}\right)
		\cdot O(\|\tilde{\bLambda}\|+\|\bDelta_j\|+\|\bE\|)\\
		=&\lambda_j\bU_{\perp}^{\top}\bE
		+O(\eps_1^2)\numberthis
	\end{align*}
where we repeatedly use the bounds on $\hat{\lambda}_1$, $\tilde{\bLambda}$ and $\bE$ from \eqref{eq:lamb-hat}, \eqref{eq:asy-del-bd} and \eqref{eq:asy-err-bd}. 
Similarly, from \eqref{eq:defGrand}
\begin{align*}\label{eq:asy-fo-term-2}
	(\bu^{(1)}_j)^{\top}\bG\bu^{(1)}_j
	=& \left((\bu^{(1)}_j)^{\top}\bY\right)^2+2\hat{\lambda}_j\left((\bu^{(1)}_j)^{\top}\bY\right)\\
	=&\left((\bu^{(1)}_j)^{\top}(\bDelta_j+\bE)\right)^2+2\hat{\lambda}_j\left((\bu^{(1)}_j)^{\top}(\bDelta_j+\bE)\right)\\
	\le & (\|\bDelta_j\|+C\eps_0+C\eps_1\eps_2/\lambda_j)^2+2\hat{\lambda}_j(\|\bDelta_j\|+\eps_0+C\eps_1\eps_2/\lambda_j)\\
	\le & C\eps_1\eps_2+C\lambda_j\eps_0.\numberthis
\end{align*}
using \eqref{eq:lamb-hat} and \eqref{eq:asy-del-bd}. In the first inequality we have also used the fact that, following \eqref{eq:piter-errbd},
\begin{align*}\label{eq:ujEhat}
	(\bu^{(1)}_j)^{\top}\bE=\scrE\times_1\bu^{(1)}\times_{k>1}\check{\bu}^{(k)}\le& \eps_0+\dfrac{C\eps_1\eps_2}{\lambda_j}+\dfrac{C\eps_1^2\|\scrE\|}{\lambda_j^2}
	\le \eps_0+\dfrac{C\eps_1\eps_2}{\lambda_j}.\numberthis
\end{align*}

On the other hand, by \eqref{eq:defGrand} and \eqref{eq:asy-Y-bd}, since $\lambda_j>C\eps_1$, one has
\begin{align*}
	\norm*{\bU_{\perp}^{\top}\bG\bU_{\perp}}
	=\norm*{\bU^{\top}_{\perp}\bY}^2\le C\eps_1^2.
\end{align*}
Referring back to \eqref{eq:uusq-expnd}, the second order terms can therefore be bounded as
\begin{align*}\label{eq:asy-so-term}
	&\dfrac{2}{\hat{\lambda}^6_j}
	(\bu^{(1)}_j)^{\top}\bG\bU_{\perp}\bU_{\perp}^\top\bG
	\bU_{\perp}\bU_{\perp}^\top\bG\bu^{(1)}_j
	+\dfrac{2}{\hat{\lambda}_j^6}
	(\bu^{(1)}_j)^{\top}\bG\bU_{\perp}\bU_{\perp}^\top\bG
	\bu^{(1)}_j(\bu^{(1)}_j)^\top\bG\bu^{(1)}_j\\
	\le& \dfrac{2}{\hat{\lambda}^6_j}
	\cdot \left((\bu^{(1)}_j)^{\top}\bG\bU_{\perp}\bU_{\perp}^\top\bG
	\bu^{(1)}_j\right)
	\cdot 
	(\|\bU_{\perp}\bG\bU_{\perp}\|+(\bu^{(1)}_j)^{\top}\bG\bu^{(1)}_j)\\
	\le& \dfrac{2}{\hat{\lambda}^6_j}
	\cdot (\lambda^2_j\eps_1^2+O(\eps_1^3\lambda_j+\eps_1^4))(C\eps_1\eps_2+C\lambda_j\eps_0)\\
	\le& \dfrac{C\eps_0\eps_1^2}{\lambda_j^3}+\dfrac{C\eps_1^3\eps_2}{\lambda_j^4}
	+\dfrac{C\eps_1^4\eps_2}{\lambda_j^5}+\dfrac{C\eps_1^5\eps_2}{\lambda_j^6}
	\le \dfrac{C\eps_0\eps_1}{\lambda_j^3}+\dfrac{2C\eps_1^3\eps_2}{\lambda_j^4}.
	\numberthis
\end{align*}
By the definition of $\bG$ in \eqref{eq:defGrand}, and plugging in the inequalities \eqref{eq:lamb-hat}, \eqref{eq:asy-fo-term}, \eqref{eq:asy-so-term} into \eqref{eq:uusq-expnd}, we have
\begin{align*}
	&\langle\bu_j^{(1)},\hat{\bu}^{(1)}\rangle^2-1\\
	=& -\dfrac{1}{\hat{\lambda}^4_j}
	(\bu^{(1)}_j)^{\top}\bG\bU_{\perp}\bU_{\perp}^\top\bG\bu^{(1)}_j
	+O\left(\dfrac{\eps_0\eps_1}{\lambda_j^3}+\dfrac{\eps_1^3\eps_2}{\lambda_j^4}\right)+\calR_4\\
	=& -\dfrac{1}{\lambda^2_j}
	\left(1+O\left(\dfrac{\eps_1^2}{\lambda_j^2}\right)\right)
	\cdot 
	\left(\bE^{\top}\bU_{\perp}\bU_{\perp}^{\top}\bE
	+\bE^{\top}\bU_{\perp}\bU_{\perp}^{\top}(\bU_{\perp}\tilde{\bLambda}+\Delta_j)
	\right)
	+O\left(\dfrac{\eps_0\eps_1}{\lambda_j^3}+\dfrac{\eps_1^3\eps_2}{\lambda_j^4}\right)+O(\calR_4)\\
	=& -\dfrac{1}{\lambda^2_j}
	\left(\bE^{\top}\bU_{\perp}\bU_{\perp}^{\top}\bE
	+\bE^{\top}\bU_{\perp}\bU_{\perp}^{\top}(\bU_{\perp}\tilde{\bLambda}+\Delta_j)
	\right)
	+O\left(\dfrac{\eps_0\eps_1}{\lambda_j^3}+\dfrac{\eps_1^3\eps_2}{\lambda_j^4}\right)
\end{align*}
since $|\calR_4|\le \dfrac{4\|\bG\|^4}{\hat{\lambda}_j^8}\le \dfrac{C\eps_1^4}{\lambda_j^4}$ by \eqref{eq:asy-G-bd}. 
\end{proof}

 \begin{proof}[Proof of Lemma \ref{lem:errnormbd2}]
 	Let $\beps_{(s)}\in \RR^{d\times d}$ be iid random matrices with independent Rademacher entries $\beps_{(s),ij}$. Define
 	$\scrG_{(s)}\in 
 	\RR^{d\times d\times d\times d}$ as
 	$$
 	\scrG_{(s)ijkl}=\bE_{(s),ij}^2\delta_{(ij)(kl)}.
 	$$
 	Also define $\bar{\scrG}:=\dfrac{1}{n}\displaystyle\sum_{i=1}^n\scrG_{(s)}$. We use symmetrization to bound the expectation first. Note that for independent copies $\tilde{\bE}_{(s)}$ of $\bE_{(s)}$,
 	\begin{align*}
 		\EE\norm*{\dfrac{1}{n}\sum_{s=1}^{n}\bE_{(s)}\circ \bE_{(s)}-\scrG}
 		=&\EE\norm*{\dfrac{1}{n}\sum_{s=1}^{n}\left(\bE_{(s)}\circ \bE_{(s)}-\tilde{\bE}_{(s)}\circ\tilde{\bE}_{(s)}\right)}\\
 		=&\EE\norm*{
 		\dfrac{1}{n}\sum_{s=1}^n
 		\left((\beps_{(s)}\ast\bE_{(s)})\circ (\beps_{(s)}\ast\bE_{(s)})-
 		(\beps_{(s)}\ast\tilde{\bE}_{(s)})\circ(\beps_{(s)}\ast\tilde{\bE}_{(s)})\right)}\\
 		\le & 2\EE\norm*{\dfrac{1}{n}\sum_{s=1}^n
 			\big[(\beps_{(s)}*\bE_{(s)})\circ(\beps_{(s)}*\bE_{(s)})\big]_{\rm offdiag}}
 		+\EE\norm*{\dfrac{1}{n}\sum_{s=1}^n(\scrG_{(s)}-\tilde{\scrG}_{(s)})}\\
 		\le & 2\EE\norm*{\dfrac{1}{n}\sum_{s=1}^n\left\{(\beps_{(s)}*\bE_{(s)})\circ(\beps_{(s)}*\bE_{(s)})-\scrG_{(s)}\right\}}
 		+3\EE\norm*{\dfrac{1}{n}\sum_{s=1}^n(\scrG_{(s)}-\scrG)}.
 	\end{align*}
 	The first inequality separates the tensor with entries $(i,j)\neq (k,l)$ and those with $(i,j)=(k,l)$. The first kind is denoted by ``${\rm offdiag}$", since it is similar to  a matrix with the diagonal entries replaced by zeroes. For the ``diagonal" part, we define $\tilde{\scrG}_{(s)}$ just as $\scrG_{(s)}$, with the $\bE_{(s),ij}$ replaced by $\tilde{\bE}_{(s),ij}$. The second inequality uses triangle inequality on the ``diagonal" tensor.
	
 	Note that $\EE\left((\beps_{(s)}*\bE_{(s)})\circ(\beps_{(s)}*\bE_{(s)})\big|\bE_{(s)}\right)=\scrG_{(s)}$. We now bound the first term conditional on $\bE_{(s)}$. To this end, we shall use an $\eps$-net argument as done in \cite{auddy2021estimating}, \cite{nguyen2015tensor} and \cite{latala2005some}.
	
 	For any integer $L,$ write $S_L=\{0,1,\dots,2^{-L}\}.$ It follows from Lemma 10 of \cite{nguyen2015tensor} that the set $N_L=\{\bx\in \RR^d:\norm{\bx}\le 1,x_i^2\in S_L\}$ forms a $(1/2)$-net for $\SS^{d-1}$ by taking $L=\log d+c_0$ for some constant $c_0$. Now define the projections
 	$$\Pi_{l}(\bx)_i=x_i\mathbbm{1}(x_i^2=2^{-l})\qquad {\rm and}\qquad \Pi_{<l}(\bx)_i=x_i\mathbbm{1}(x_i^2\ge 2^{-l}).$$
 	Let $N_{l}=\Pi_{l}(N_L)$ and $N_{<l}=\Pi_{<l}(N_L)$ for $1\le l\le L$. Note that for any $\bx\in N_L,$
 	$$\bx=\sum_{l=1}^L\Pi_{l}(\bx)\qquad {\rm and} \qquad \sum_{m<l}\Pi_{m}(\bx)=\Pi_{<l}(\bx).$$ 
 	Let us define
 	$$
 	\mathscr{H}=\dfrac{1}{n}\sum_{s=1}^n\left\{(\beps_{(s)}*\bE_{(s)})\circ(\beps_{(s)}*\bE_{(s)})-\scrG_{(s)}\right\}
 	$$
 	Expanding the sum for vectors $\bx_1,\bx_2$  we get
 	$$\begin{aligned}
 		&\mathscr{H}\times_1\bx_1\times_2\bx_2\times_3\bx_1\times_4\bx_2\\
 		=&\sum_{l_1=1}^L\sum_{l_2=1}^L\mathscr{H}\times_1\Pi_{l_1}(\bx_1)\times_2
 		\Pi_{l_2}(\bx_2)\times_3\Pi_{l_1}(\bx_1)\times_4\Pi_{l_2}(\bx_2)
 		\\=&\sum_{k=1}^2\underset{\mathrm{argmax} \,l_i=k}{\sum_{l_k=1}^L}\sum_{\underset{i\neq k}{l_i\le l_k}}
 		\mathscr{H}\times_1\Pi_{l_1}(\bx_1)\times_2
 		\Pi_{l_2}(\bx_2)\times_3\Pi_{l_1}(\bx_1)\times_4\Pi_{l_2}(\bx_2)
 		\\=&
 		\sum_{l_1=1}^L
 		\mathscr{H}\times_1\Pi_{l_1}(\bx_1)
 		\times_2\left(\sum_{l_2\le l_1}\Pi_{l_2}(\bx_2)\right)
 		\times_3\Pi_{l_1}(\bx_1)
 		\times_4\left(\sum_{l_2\le l_1}\Pi_{l_2}(\bx_2)\right)\\
 		&+
 		\sum_{l_2=1}^L
 		\mathscr{H}\times_1\left(\sum_{l_1\le l_2}\Pi_{l_1}(\bx_1)\right)
 		\times_2\Pi_{l_2}(\bx_2)
 		\times_3\left(\sum_{l_1\le l_2}\Pi_{l_1}(\bx_1)\right)
 		\times_4\Pi_{l_2}(\bx_2)\\
 		=&\sum_{l_1=1}^L
 		\mathscr{H}\times_1\Pi_{l_1}(\bx_1)
 		\times_2\Pi_{<l_1}(\bx_2)
 		\times_3\Pi_{l_1}(\bx_1)
 		\times_4\Pi_{<l_1}(\bx_2)\\
 		&+\sum_{l_2=1}^L
 		\mathscr{H}\times_1\Pi_{<l_2}(\bx_1)
 		\times_2\Pi_{l_2}(\bx_2)
 		\times_3\Pi_{<l_2}(\bx_1)
 		\times_4\Pi_{l_2}(\bx_2)
 	\end{aligned}$$
 Note that
 $$\begin{aligned}
 \|\mathscr{H}\|
 =&\sup_{\bx_1,\bx_2\in\SS^{d-1}}\mathscr{H}\times_1\bx_1\times_2\bx_2\times_3\bx_1\times_4\bx_2\\
 \le& \sup_{\bx_1,\bx_2\in N_L}\sum_{l_1=1}^L
 \mathscr{H}\times_1\Pi_{l_1}(\bx_1)
 \times_2\Pi_{<l_1}(\bx_2)
 \times_3\Pi_{l_1}(\bx_1)
 \times_4\Pi_{<l_1}(\bx_2)\\
 &+\sup_{\bx_1,\bx_2\in N_L}\sum_{l_2=1}^L
 \mathscr{H}\times_1\Pi_{<l_2}(\bx_1)
 \times_2\Pi_{l_2}(\bx_2)
 \times_3\Pi_{<l_2}(\bx_1)
 \times_4\Pi_{l_2}(\bx_2)
 \end{aligned}
 $$
 We will bound the first term as the other one follows by symmetry.

 Fix $\bx_1,\bx_2\in N_L$. Note that, conditional on $\bE_{(s)}$,
 \begin{align*}
 Y_s=&\left\{(\beps_{(s)}*\bE_{(s)})\circ(\beps_{(s)}*\bE_{(s)})-\scrG_{(s)}\right\}\times_1\bx_1\times_2\bx_2\times_3\bx_1\times_4\bx_4\\
 =&\left(\bx_1^\top(\beps_{(s)}*\bE_{(s)})\bx_2 \right)^2-\left(\sum_{i,j}\bE_{(s),ij}^2\bx_{1i}^2\bx_{2j}^2\right)
 \end{align*}
 satisfy $\EE Y_s=0$, and have sub-exponential norms
 \begin{align*}
 	\|Y_s\|_{\psi_1}(\bx_1,\bx_2)\le2\sum_{i,j}\bE_{(s),ij}^2\bx_{1i}^2\bx_{2j}^2.
 \end{align*}
 Thus by Bernstein inequality, for $\bx_1,\bx_2\in N_L$
 \begin{align*}
 &|\mathscr{H}\times_1\Pi_{l_1}(\bx_1)
 \times_2\Pi_{<l_1}(\bx_2)
 \times_3\Pi_{l_1}(\bx_1)
 \times_4\Pi_{<l_1}(\bx_2)|\\
 >&\max\left\{\left(\dfrac{t}{n}\sum_{s}\|Y_s(\Pi_l(\bx_1),\Pi_{<l}(\bx_2))\|_{\psi_1}^2\right)^{1/2},
 \dfrac{t}{n}\cdot \max_s\|Y_s(\Pi_l(\bx_1),\Pi_{<l}(\bx_2))\|
 \right\}
 \end{align*}
 with probability at most $\exp(-t)$. By Lemma 4 of \cite{latala2005some}, 
 $$
 |N_l|<|N_{<l}|<C2^l(1+L-l).
 $$
 Thus by union bound,
 \begin{align*}
 &\sup_{\bx_1,\bx_2\in N_L}|\mathscr{H}\times_1\Pi_{l_1}(\bx_1)
 \times_2\Pi_{<l_1}(\bx_2)
 \times_3\Pi_{l_1}(\bx_1)
 \times_4\Pi_{<l_1}(\bx_2)|\\
 \le& \max\left\{\left(\dfrac{t+2^l(1+L-l)}{n^2}\sum_{s}\|Y_s(\Pi_l(\bx_1),\Pi_{<l}(\bx_2))\|_{\psi_1}^2\right)^{1/2},
 \dfrac{t+2^l(1+L-l)}{n}\cdot \max_s\|Y_s(\Pi_l(\bx_1),\Pi_{<l}(\bx_2))\|
 \right\}	
 \end{align*}
 with probability at most $\exp(-t)$. Summing over $l$, the first term becomes 
 \begin{align*}
 	&\sum_{l=1}^L\left(\dfrac{t+2^l(1+L-l)}{n^2}\sum_{s=1}^n\|Y_s(\Pi_l(\bx_1),\Pi_{<l}(\bx_2))\|_{\psi_1}^2\right)^{1/2}\\
 	\le & 2\sum_{l=1}^L\sqrt{\dfrac{t+2^l(1+L-l)}{n}}
 	\left(\dfrac{1}{n}
 	\sum_{s=1}^n
 	\sum_{i,j}\left(E^2_{(s),ij}(\Pi_{l}(\bx_1))_i^2(\Pi_{<l}(\bx_2))_j^2\right)^2
 	\right)^{1/2}\\
 	\le&
 	2\left(\sum_{l=1}^L\dfrac{t\cdot 2^{-l}+(1+L-l)}{n}\right)^{1/2}
 	\left\{
 	\sum_{l=1}^L
 	\dfrac{1}{n}\cdot 2^l\sum_{s=1}^n\max_j\left(\sum_iE_{(s),ij}^42^{-l}\mathbbm{1}(\Pi_l(\bx_1)_i\neq 0)\right)
 	\right\}^{1/2}
 	\\
 	\le&
 		2\left(\dfrac{t+(\log d+c_0)^2}{n}\right)^{1/2}
 		\left\{\dfrac{1}{n}\sum_{s=1}^n
 		\max_j
 		\sum_{i}E^4_{(s),ij}
 		\right\}^{1/2}.
 \end{align*}
 Similarly the second term becomes
 \begin{align*}
 	&\sum_{l=1}^L\dfrac{t+2^l(1+L-l)}{n}\cdot \max_s\|Y_s(\Pi_l(\bx_1),\Pi_{<l}(\bx_2))\|\\
 	\le &2\sum_{l=1}^L\dfrac{t+2^l(1+L-l)}{n}\cdot \max_s
 	\sum_{i,j}E^2_{(s),ij}(\Pi_l(\bx_1))_i^2(\Pi_{<l}(\bx_2))_j^2\\
 	\le & \dfrac{2}{n}\max_{s,j}\sum_{i=1}^d
 	\sum_{l=1}^L(t\cdot 2^{-l}+1+L_l)
 	E^2_{(s),ij}\mathbbm{1}(\Pi_l(\bx_1)_i\neq 0)\\
 	\le & \dfrac{C\log d}{n}\max_{s,j}\sum_{i=1}^dE^2_{(s),ij}.
 \end{align*}
 Taking expectation conditional over $\bE_{(s)}$, one has
 \begin{align*}
 	\EE(\|\mathscr{H}\|\big|\bE_{(s)})
 	\le \dfrac{C\log d}{\sqrt{n}}\left\{\dfrac{1}{n}\sum_{s=1}^n
 	\max_j
 	\sum_{i}E^4_{(s),ij}
 	\right\}^{1/2}
 	+\dfrac{C\log d}{n}\max_{s,j}\sum_{i=1}^dE^2_{(s),ij}
 \end{align*}
 and thus unconditionally,
 \begin{align*}
 	\EE\|\mathscr{H}\|
 	\le &\dfrac{C\log d}{\sqrt{n}}\left\{\dfrac{1}{n}\sum_{s=1}^n
 	\sum_{i,j}\EE(E^8_{(s),ij})
 	\right\}^{1/4}
 	+\dfrac{C\log d}{n}\EE\left(\max_{s,j}\sum_{i=1}^dE^2_{(s),ij}\right)\\
 	\le& C(\log d)\sqrt{\dfrac{d}{n}}+\dfrac{C(\log d)\cdot (d+(nd)^{1/4}d^{1/2})}{n}
 \end{align*}
 by Rosenthal and Khintchine inequalities as used in the proof of Theorem 2.1 of \cite{auddy2021estimating}. Meanwhile,
 \begin{align*}
 	\norm*{\dfrac{1}{n}\sum_{s=1}^n(\scrG_{(s)}-\scrG)}
 	=\sup_{\bx_1,\bx_2\in \SS^{d-1}}\dfrac{1}{n}
 	\sum_{s=1}^n(E^2_{(s),ij}-1)(\bx_1)_{i}^2(\bx_2)_{j}^2\le \max_{i,j}\abs*{\dfrac{1}{n}\sum_{s=1}^nE^2_{(s),ij}-1}
 \end{align*}
 By Chebyshev inequality,
 $$
 \PP\left(\max_{i,j}\abs*{\dfrac{1}{n}\sum_{s=1}^nE^2_{(s),ij}-1}>t\right)
 \le \dfrac{Cd^2n^2\EE|E^8_{(s),ij}|}{n^4t^4}
 $$
 and thus, integrating over $t$
 $$
 \EE\norm*{\dfrac{1}{n}\sum_{s=1}^n(\scrG_{(s)}-\scrG)}
 \le C\sqrt{\dfrac{d(\log d)^2}{n}},
 $$
 provided $n\ge C d$. Combining all the results above, we have
 $$
 \EE\norm*{\dfrac{1}{n}\sum_{s=1}^n\bE_{(s)}\circ\bE_{(s)}-\scrG}\le
 C\sqrt{\dfrac{d(\log d)^2}{n}}
 $$
 Finally using Talagrand's concentration inequality for bounded functions (see Theorem 1.1 of \cite{klein2005concentration}), we get
 $$
 \norm*{\dfrac{1}{n}\sum_{s=1}^n\bE_{(s)}\circ\bE_{(s)}-\scrG}\le C\sqrt{\dfrac{d(\log d)^2}{n}}
 $$
 with probability at least $1-d^{-2}$.
 \end{proof}

\end{document}